\documentclass{article}

\usepackage[final]{neurips_2025}

\usepackage[utf8]{inputenc} % allow utf-8 input
\usepackage[T1]{fontenc}    % use 8-bit T1 fonts
\usepackage{hyperref}       % hyperlinks
\usepackage{url}            % simple URL typesetting
\usepackage{booktabs}       % professional-quality tables
\usepackage{amsfonts}       % blackboard math symbols
\usepackage{nicefrac}       % compact symbols for 1/2, etc.
\usepackage{microtype}      % microtypography
\usepackage{xcolor}         % colors
\usepackage{enumitem}
\usepackage{tikz}
\usepackage{subcaption}
\usepackage{microtype}
\usepackage{graphicx}
\usepackage[percent]{overpic}
\usepackage{amsmath}
\usepackage{amssymb}
\usepackage{mathtools}
\usepackage{amsthm}
\usepackage{algorithm}
\usepackage{algorithmic}
\usepackage[capitalize,noabbrev]{cleveref}
\usepackage{wrapfig}
\usepackage[most]{tcolorbox}

\usepackage{mymacros}

\definecolor{softred}{rgb}{0.8, 0.2, 0.2}
\definecolor{softgreen}{rgb}{0.2, 0.6, 0.2}
\definecolor{softblue}{rgb}{0.3, 0.5, 0.8}
\definecolor{softorange}{rgb}{0.9, 0.5, 0.2}

\pgfplotsset{compat=1.18}
\theoremstyle{plain}
\newtheorem{theorem}{Theorem}[section]

\newtheorem{lemma}[theorem]{Lemma}

\theoremstyle{definition}

\newtheorem{assumption}[theorem]{Assumption}
\theoremstyle{remark}

\definecolor{softbluegray}{HTML}{F2F6F9}
\definecolor{textbluegray}{HTML}{629999}

\title{Towards Principled Unsupervised Multi-Agent Reinforcement Learning}

\author{%
  Riccardo Zamboni \\
  Politecnico di Milano\\
  \texttt{riccardo.zamboni@polimi.it} \\
  \And
  Mirco Mutti \\
  Technion\\
  \\
  \And
  Marcello Restelli \\
  Politecnico di Milano\\
}

\begin{document}

\maketitle

\begin{abstract}
    In reinforcement learning, we typically refer to \emph{unsupervised} pre-training when we aim to pre-train a policy without a priori access to the task specification, \ie rewards, to be later employed for efficient learning of downstream tasks. In single-agent settings, the problem has been extensively studied and mostly understood. A popular approach casts the unsupervised objective as maximizing the \emph{entropy} of the state distribution induced by the agent's policy, from which principles and methods follow.
    In contrast, little is known about state entropy maximization in multi-agent settings, which are ubiquitous in the real world. What are the pros and cons of alternative problem formulations in this setting? How hard is the problem in theory, how can we solve it in practice? In this paper, we address these questions by first characterizing those alternative formulations and highlighting how the problem, even when tractable in theory, is non-trivial in practice. Then, we present a scalable, decentralized, trust-region policy search algorithm to address the problem in practical settings. Finally, we provide numerical validations to both corroborate the theoretical findings and pave the way for unsupervised multi-agent reinforcement learning via state entropy maximization in challenging domains, showing that optimizing for a specific objective, namely \emph{mixture entropy}, provides an excellent trade-off between tractability and performances.
\end{abstract}

\section{Introduction}
\label{sec:intro}

Multi-Agent Reinforcement Learning~\citep[MARL,][]{marlbook} recently showed promising results in learning complex behaviors, such as coordination and teamwork~\citep{samvelyan2019starcraft}, strategic planning in the presence of imperfect knowledge~\citep{perolat2022mastering}, and trading~\citep{johanson2022emergentbarteringbehaviourmultiagent}. Just like in single-agent RL, however, most of the efforts are focused on tabula rasa learning, that is, without exploiting any prior knowledge gathered from offline data and/or policy pre-training. Despite its generality, learning tabula rasa hinders MARL from addressing real-world situations, where training from scratch is slow, expensive, and arguably unnecessary~\citep{agarwal2022reincarnating}. In this regard, some progress has been made on techniques specific to the multi-agent setting, ranging from ad hoc teamwork~\citep{mirsky2022survey} to zero-shot coordination~\citep{hu2020other}, but our understanding of what can be done \emph{instead of} learning tabula rasa is still limited.

In single-agent RL, unsupervised pre-training frameworks~\citep{laskin2021urlb} have emerged as a viable solution: a policy is pre-trained without a priori access to the task specification, \ie rewards, to be later employed for efficient learning of downstream tasks. Among others, state-entropy maximization~\citep{hazan2019provably, lee2020efficientexplorationstatemarginal} was shown to be a useful tool for policy pre-training~\citep{hazan2019provably,mutti2021taskagnosticexplorationpolicygradient} and data collection for offline learning~\citep{yarats2022don}. In this setting, the unsupervised objective is cast as maximizing the entropy of the state distribution induced by the agent's policy. Recently, the potential of entropy objectives in MARL was empirically corroborated by a plethora of works~\citep{liu2021cooperative, zhang2021made, yang2021ciexplore, xu2024population} investigating entropic reward-shaping techniques to boost exploration in downstream tasks. Yet, to the best of our knowledge, the literature still lacks a principled understanding of how state entropy maximization works in multi-agent settings, and how it can be used for unsupervised pre-training. Let us think of an illustrative example that highlights the central question of this work: multiple autonomous robots deployed in a factory for a production task. The robots' main goal is to perform many operations over a large set of products, with objectives ranging from optimizing for costs and energy to throughput, which may change over time depending on the market's condition. Arguably, trying to learn each possible task from scratch is inefficient and unnecessary. On the other hand, one could think of first learning to cover the possible states of the system and then fine-tune this general policy over a specific task. Yet, if everyone is focused on their own exploration, any incentive to collaborate with each other may disappear, especially when coordinating comes at a cost for individuals. Similarly, covering the entire space might be unreasonable in most real-world cases. Here we are looking for a third alternative.%\vspace{-4pt}

\begin{tcolorbox}[colback=softbluegray, colframe=softbluegray,  boxrule=0.5pt, arc=4pt, width=\linewidth]
\textbf{Research Questions:}
\begin{itemize}[leftmargin=10pt]
    \item[] (\textbf{Q1}) Can we formulate a multi-agent counterpart of the unsupervised pre-training via state entropy maximization in a principled way?
    \item[] (\textbf{Q2}) How are different formulations related? Do crucial theoretical differences emerge?
    \item[] (\textbf{Q3}) Can we explicitly pre-train a policy for state entropy maximization in practical multi-agent  scenarios?
    \item[] (\textbf{Q4}) Do crucial differences emerge in practice? Does this have an impact on downstream tasks learning?
\end{itemize}
\end{tcolorbox}

\textbf{Content Outline and Contributions.}~~First, in Section~\ref{sec:problem_formulation}, we address (\textbf{Q1}) by showing that the problem can be addressed through the lenses of a specific class of decision making problems, called \emph{convex Markov Games}~\citep{gemp2025convexmarkovgamesframework, kalogiannis2025solving}, yet it can take different, alternative, formulations. Specifically, they differ on whether the agents are trying to \emph{jointly} cover the space through conditionally dependent actions, or they neglect the presence of others and deploy fully \emph{disjoint} strategies, or they coordinate to cover the state space beforehand, but taking actions independently as components of a \emph{mixture}. We formalize these cases into three distinct objectives. Then, in Section~\ref{sec:properties}, we address (\textbf{Q2}), highlighting that these objectives are related through performance bounds that scale with the number of agents. We also show that only the joint or mixture objectives enjoy remarkable convergence properties under policy gradient updates in the ideal case of evaluating the agents' performance over infinite realizations (trials). However, as one shifts the attention to the more practical scenario of reaching good performance over a handful, or even just one, trial, we show that different objectives lead to different behaviors and mixture objectives do enjoy more favorable properties. Then, in Section~\ref{sec:algorithm}, we address (\textbf{Q3}) by introducing a decentralized multi-agent policy optimization algorithm, called \emph{Trust Region Pure Exploration} (TRPE), explicitly addressing state entropy maximization pre-training over finite trials. Finally, we address (\textbf{Q4}) by testing the algorithm on some simple yet challenging settings, showing that optimizing for a specific objective, namely mixture entropy, provides an excellent trade-off between tractability and performances. We show that this objective yields superior sample complexity and remarkable zero-shot performance when the pre-trained policy is deployed in sparse reward downstream tasks.
%\vspace{-8pt}
\section{Preliminaries}
\label{sec:setting}
%\vspace{-8pt}
In this section, we introduce the most relevant background and the basic notation.%\vspace{-4pt}

\textbf{Notation.}~~We denote $[N] := \{1, 2, \ldots, N\}$ for a constant $N < \infty$. We denote a set with a calligraphic letter $\Acal$ and its size as $|\Acal|$. For a (finite) set $\Acal = \{1,2,\dots, i, \dots\}$, we denote $-i = \Acal / \{i\}$ the set of all its elements but $i$. $\Acal^T := \times_{t = 1}^T \Acal$ is the $T$-fold Cartesian product of $\Acal$. The simplex on $\Acal$ is $\Delta_{\Acal} := \{ p \in [0, 1]^{|\Acal|} | \sum_{a \in \Acal} p(a) = 1 \}$ and $\Delta_{\Acal}^{\Bcal}$ denotes the set of conditional distributions $p: \Acal \to \Delta_\Bcal$. Let $X, X'$ random variables on the set of outcomes $\Xs$ and corresponding probability measures $p_X, p_{X'}$, we denote the Shannon entropy of $X$ as $H (X) = - \sum_{x \in \Xs} p_X (x) \log (p_X (x))$ and the Kullback-Leibler (KL) divergence as $\kl(p_X\| p_{X'}) = \sum_{x \in \Xs} p_X (x) \log (p_{X} (x)/p_{X'} (x))$. We denote $\xs = (X_1, \ldots, X_T)$ a random vector of size $T$ and $\xs[t]$ its entry at position $t \in [T]$.%\vspace{-4pt}

\textbf{Interaction Protocol.}~~As a model for interaction, we consider finite-horizon Markov Games~\citep[MGs,][]{Littman1994} \emph{without rewards}. A MG $\MDP := (\Ns, \Ss, \Acal, \Pb, \mu, T)$ is composed of a set of agents $\Ns$, a set $\Ss = \times_{i \in [\Ns]} \Ss_i$ of states, and a set of (joint) actions $\Acal = \times_{i \in [\Ns]} \Acal_i$, which we assume to be discrete and finite in size $|\Ss|, |\Acal|$ respectively. At the start of an episode, the initial state $s_1$ of $\MDP$ is drawn from an initial state distribution $\mu \in \Delta_{\Ss}$. Upon observing $s_1$, each agent takes action $a_1^i \in \Acal_i$, the system transitions to $s_2 \sim \Pb(\cdot|s_1, a_1)$ according to the \emph{transition model} $\Pb \in \Delta^{\Ss}_{\Ss \times \Acal}$. The process is repeated until $s_T$ is reached and $s_T$ is generated, being $T < \infty$ the horizon of an episode. Each agent acts according to a \emph{ policy}, that can be either Markovian when the action is only conditioned on the current state, i.e., $\pi^i \in \Delta_{\Ss}^{\Acal^i}$, or non-Markovian when the action is conditioned on the history, i.e., $\pi^i \in \Delta_{\Ss^t\times \Acal^t}^{\Acal^i}$.\footnote{In general, we will denote the set of valid per-agent policies with $\Pi^i$ and the set of joint policies with $\Pi$.} Also, we will denote as \emph{decentralized-information} policies the ones conditioned on either $\Ss_i$ or $\Ss^t_i\times \Acal^t_i$ for agent $i$, and \emph{centralized-information} ones the ones conditioned over the full state or state-actions sequences. It follows that the joint action is taken according to the \emph{joint} policy $\Delta_{\Ss}^{\Acal}  \ni\pi = (\pi^{i})_{i \in [\Ns]}$. 

\textbf{Induced Distributions.}~~
Now, let us denote as $S$ and $S_i$ the random variables corresponding to the joint state and $i$-th agent state respectively. Then the former is distributed as $d^\pi \in \Delta_{\Ss}$, where $d^\pi (s) = \frac{1}{T} \sum_{t \in [T]} Pr (s_t = s|\pi,\mu)$, the latter is distributed as $d^\pi_i \in \Delta_{\Ss_i}$, where $d^\pi_i (s_i) = \frac{1}{T} \sum_{t \in [T]} Pr (s_{t,i} = s_i|\pi,\mu)$. 
Furthermore, let us denote with $\sbf,\as$ the random vectors corresponding to sequences of (joint) states, and actions of length $T$, which are supported in $\Ss^T, \Acal^T$ respectively. We define $p^\pi \in \Delta_{\Ss^T \times \Acal^T}$, where $p^\pi(\sbf,\va) = \prod_{t \in [T]}Pr(s_t = \sbf[t], a_t = \as[t])$. Finally, we denote the empirical state distribution induced by $K \in \mathbb N^+$ trajectories $\{\sbf_k\}_{k \in [K]}$ as $d_K (s) = \frac{1}{KT}  \sum_{k \in [K]} \sum_{t \in [T]}  \mathds{1} (\sbf_k[t] = s)$.

\textbf{Convex MDPs and State Entropy Maximization.}~~
In the MDP setting ($|\Ns|=1$), the problem of state entropy maximization can be viewed as a special case of \emph{convex RL} \citep{hazan2019provably, zhang2020variationalpolicygradientmethod,zahavy2023rewardconvexmdps}. In such framework, the general task is defined via an F-bounded concave\footnote{In practice, the function can be either convex, concave, or even non-convex. The term is used to distinguish the objective from the standard (linear) RL objective. We will assume $\mathcal F$ is concave if not mentioned otherwise.} utility function $\mathcal F : \Delta_{\Ss} \rightarrow (-\infty,F]$, with $F < \infty$, that is a function of the state distribution $d^\pi$. This allows for a generalization of the standard RL objective, which is a linear product between a reward vector and the state(-action) distribution~\citep{puterman2014markov}. Usually, some regularity assumptions are enforced on the function $\mathcal F$. In the following, we align with the literature through the following smoothness assumption:
\begin{restatable}[Lipschitz]{ass}{lip}
    \label{thr:lip} 
    A function $\mathcal F : \Acal \rightarrow \mathbb R$ is Lipschitz-continuous for some constant $L < \infty$, or L-Lipschitz for short, if it holds \(
        |\mathcal F(x) - \mathcal F(y)| \leq L \|x - y\|_1, \; \forall (x,y) \in \Acal^2.
    \)
\end{restatable}
%\vspace{-6pt}
More recently,~\citet{mutti2023challengingcommonassumptionsconvex} noticed that in many practical scenarios only a finite number of $K\in \mathbb N^+$ episodes/trials can be drawn while interacting with the environment, and in such cases one should focus on $d_K$ rather than $d^\pi$. As a result, they contrast the \emph{infinite-trials} objective defined as $\zeta_\infty(\pi) :=\mathcal F(d^\pi)$ with a \emph{finite-trials} one, namely $\zeta_K(\pi) := \E_{d_K\sim p^\pi_K}\mathcal F (d_K)$, noticing that convex MDPs (cMDPs) are characterized by the fact that $\zeta_K(\pi) \leq \zeta_\infty(\pi)$, differently from standard (linear) MDPs for which equality holds. In single-agent convex RL, state entropy maximization is defined as solving a cMDP equipped with an entropy functional~\citep{hazan2019provably}, namely $\mathcal F(d^\pi) := H(d^\pi)$. 

Interestingly, even in single-agent settings, the infinite-trials state entropy objective can be formulated as a non-Markovian reward, as the \emph{value} of being in a state depends on the states visited before and after that state.\footnote{By conditioning with respect to the policy, such a reward would result to be Markovian. However, the contraction argument does not appear to hold for a Bellman operator over this kind of policy-based rewards.} As a consequence, it is not possible to derive Bellman operators of any kind~\citep{takacs1966non, whitehead1995reinforcement, zhang2020variationalpolicygradientmethod}. Conversely, for finite-trials formulations, it is possible to define a Bellman operator by extending the state representation to include the whole trajectories of interaction. Unfortunately though, even this option is intractable as the size of such an extended MDP will grow exponentially.\footnote{ Indeed, the optimization of the finite-trial formulation is NP-hard~\citep{2023mutticonvexrlfinite}.}%\vspace{-6pt}
\section{Problem Formulation}
\label{sec:problem_formulation}
This section addresses the first of the research questions outlined in the introduction.

\begin{tcolorbox}[colback=softbluegray, colframe=softbluegray,  boxrule=0.5pt, arc=4pt, width=\linewidth]
\begin{center}
  (\textbf{Q1}) \emph{Can we formulate a multi-agent counterpart of the unsupervised pre-training \\ via state entropy maximization in a principled way?} 
\end{center}
\end{tcolorbox}

In fact, when a reward function is not available, the core of the problem resides in finding a well-behaved problem formulation coherent with the task.~\citet{gemp2025convexmarkovgamesframework} recently introduced a convex generalization of MGs called \textbf{convex Markov Games} (cMGs), namely a tuple $\MDP_{\mathcal F} := (\Ns, \Ss, \Acal, \Pb, \mathcal F, \mu, T)$, that consists in a MG equipped with (non-linear) functions of the \emph{stationary joint state} distribution $\mathcal F(d^\pi)$. We expand over this definition, by noticing that state entropy maximization can be casted as solving a cMG equipped with an entropy functional, namely $\mathcal F(\cdot) := H(\cdot)$. Yet, important new questions arise: Over which distributions should agents compute the entropy? How much information should they have access to? Can we define objectives accounting for a finite number of trials? Different answers depict different objectives.

\textbf{Joint Objectives.}~~The first and most straightforward way to formulate the problem is to define it as in the MDP setting, with the joint state distribution simply taking the place of the single-agent state distribution. In this case, we define \emph{infinite-trials} and \emph{finite-trials} \emph{Joint} objectives, respectively
\begin{align}
\vspace{-5pt}
&\max_{\pi = (\pi^i\in \Pi^i)_{i \in [\Ns]}} \ \Big\{ \zeta_\infty(\pi) := \mathcal F(d^\pi)\Big\} &
\max_{\pi = (\pi^i\in \Pi^i)_{i \in [\Ns]}} \ \Big\{ \zeta_K(\pi) :=\E_{d_K\sim p^\pi_K}\mathcal F (d_K)\Big\} \label{eq:mse} 
\end{align} %\vspace{-5pt}
In state entropy maximization tasks, an optimal (joint) policy will try to cover the joint state space uniformly, either in expectation or over a finite number of trials respectively. In this, the joint formulation is rather intuitive as it describes the most general case of multi-agent exploration. Moreover, as each agent sees a difference in performance explicitly linked to others, this objective should be able to foster coordinated exploration. As we shall see, this comes at a price.%\vspace{-4pt}

\textbf{Disjoint Objectives.}~~One might look for formulations that fully embrace the multi-agent setting, such as defining a set of functions supported on per-agent state distributions rather than joint distributions. This intuition leads to \emph{infinite-trials} and \emph{finite-trials} \emph{Disjoint} objectives: 
%\vspace{-2pt}
\begin{align}
         &\Big\{ \max_{\pi^i\in \Pi^i} \zeta^i_\infty(\pi^i,\cdot) := \mathcal F(d_i^{\pi^i,\cdot})\Big\}_{ i \in [\Ns] } &
        \Big\{ \max_{\pi^i\in \Pi^i} \zeta^i_K(\pi^i, \cdot) := \E_{d_K\sim p^{\pi^i,\cdot}_K}\mathcal F (d_{K,i}) \Big\}_{ i \in [\Ns] } \label{eq:mse_decentralized}
\end{align}
According to these objectives, each agent will try to maximize her own marginal state entropy separately, neglecting the effect of her actions over others performances. In other words, we expect this objective to hinder the potential coordinated exploration, where one has to take as step down as so allow a better performance overall.%\vspace{-4pt}

\begin{wrapfigure}{r}{0.45\columnwidth}  % 'r' for right, 'l' for left; width of the figure box
    \centering
    \vspace{-13pt}  % Adjust vertical space if needed
    \includegraphics[trim=50 65 50 40, clip, width=0.42\columnwidth]{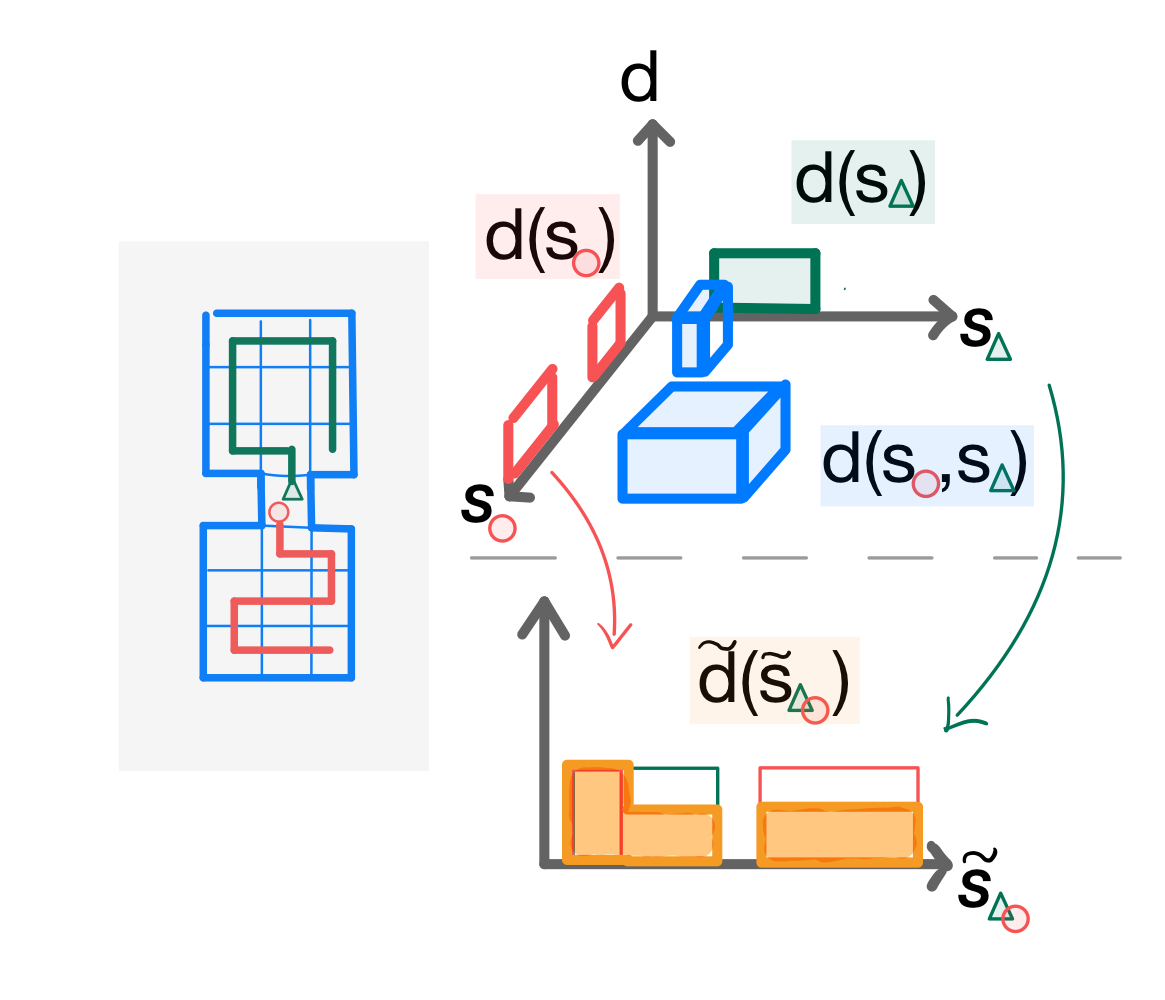}
    \vspace{-6pt} 
    \caption{The interaction on the \emph{left} induces different (empirical) distributions: Marginal distributions for \textcolor{softred}{\textbf{agent 1}} and \textcolor{softgreen}{\textbf{agent 2}} over their respective states; a \textcolor{softblue}{\textbf{joint distribution}} over the product space; a \textcolor{softorange}{\textbf{mixture distribution}} over a common space, defined as the average. The mixture distribution is usually \emph{less sparse}.}
    \label{fig:distributions}
    \vspace{-45pt}  % Optional: reduce space below the wrapfigure
\end{wrapfigure}

\textbf{Mixture Objectives.}~~At last, we introduce a problem formulation that will later prove capable of uniquely taking advantage of the structure of the problem. First, we introduce the following:

\begin{restatable}[Uniformity]{ass}{mixutre}
    \label{ass:mixture} The agents have the same state space $\Ss_i = \Ss_j = \tilde \Ss, \forall (i,j) \in \Ns \times \Ns$.\footnote{One should notice that even in cMGs where this is not (even partially) the case, the assumption can be enforced by padding together the per-agent states.}
\end{restatable}

Under this assumption, we will drop the agent subscript when referring to the per-agent states and use $\tilde \Ss$ instead. Interestingly, this assumption allows us to define a particular distribution:
%\vspace{-1pt}
\[
    \tilde d^\pi(\tilde s) := \frac{1}{|\Ns|}\sum_{i \in [\Ns]} d^\pi_i(\tilde s) \in \Delta_{\tilde \Ss}.
\]
%\vspace{-6pt}

We refer to this distribution as \emph{mixture} distribution, given that it is defined as a uniform mixture of the per-agent marginal distributions. Intuitively, it describes the average probability over all the agents to be in a common state $\tilde s \in \tilde \Ss$, in contrast with the joint distribution that describes the probability for them to be in a joint state $s$, or the marginals that describes the probability of each one of them separately. In Figure~\ref{fig:distributions} we provide a visual representation of these concepts. Similarly to what happens for the joint distribution, one can define the empirical distribution induced by $K$ episodes as $\tilde d_{K} (\tilde s) = \frac{1}{|\Ns|} \sum_{i \in [\Ns]}  d_{K,i}(\tilde s)$ and $\tilde d^\pi = \E_{\tilde d_K\sim p^\pi_K}[\tilde d_K]$. The mixture distribution allows for the definition of the \emph{Mixture} objectives, in their infinite and finite trials formulations respectively:
%\vspace{-0.2cm}
\begin{align}
        &\max_{\pi = (\pi^i\in \Pi^i)_{i \in [\Ns]}} \ \Big\{ \tilde \zeta_\infty(\pi) := \mathcal F (\tilde d^{\pi}) \Big\}
       & \max_{\pi = (\pi^i\in \Pi^i)_{i \in [\Ns]}} \ \Big\{ \tilde \zeta_K(\pi) := \E_{\tilde d_{K}\sim p^\pi_K}\mathcal F (\tilde d_{K}) \Big\} \label{eq:mse_mixture}
\end{align}
\vspace{-10pt}

When this kind of objectives is employed in state entropy maximization, the entropy of the mixture distribution decomposes as 
\(H(\tilde d^{\pi}) = \frac{1}{|\Ns|}\sum_{i\in [\Ns]}H(d^{\pi}_i) + \frac{1}{|\Ns|}\sum_{i\in [\Ns]}\kl(d^{\pi}_i||\tilde d^{\pi})\) and one remarkable scenario arises: Agents follow policies possibly inducing lower disjoint entropies, but their induced marginal distributions are maximally different. Thus, the average entropy remains low, but the overall mixture entropy is high due to diversity (i.e., high values of the KL divergences). This scenario has been referred to in~\citet{Kolchinsky_2017} as the \emph{clustering} scenario and, in the following, we will provide additional evidences why this scenario is particularly relevant.

 %At least, we know that the mixture entropy is also a lower bound to the joint entropy with a $\log (|\mathcal{N}|)$ approximation (see \textbf{Lem. 4.1}) and thus a valid proxy also in the latter case.

%A \textbf{practical example} in which mixture entropy is desirable: We want to train a team of "search and rescue" agents. In a specific building (environment) the target may be found in different place (different rewards) and we want to prepare for all. Mixture entropy is a good surrogate objective in this case, as the agents will split up into different portions of the buildings to traverse in order to find the target quickly.

\section{A Formal Characterization of Multi-Agent State Entropy Maximization}
\label{sec:properties}

In the previous section, we provided a principled problem formulation of multi-agent state entropy maximization through an array of different objectives. Here, we address the second research question:
\begin{tcolorbox}[colback=softbluegray, colframe=softbluegray,  boxrule=0.5pt, arc=4pt, width=\linewidth]
\begin{center}
    (\textbf{Q2}) \emph{How are different formulations related? Do crucial theoretical differences emerge?}
\end{center}
\end{tcolorbox}

First of all, we show that if we look at state entropy maximization tasks specifically, i.e. cMGs $\mathcal M_{H}$ equipped with entropy functionals $\mathcal F(\cdot) := H(\cdot)$, all the objectives in infinite-trials formulation can be elegantly linked one to the other though the following result:

\begin{restatable}[Entropy Mismatch]{lem}{entropymismatch}
    \label{lem:entropymismatch} 
    For every cMG $\mathcal M_{H}$, for a fixed (joint) policy $\pi = (\pi^{i})_{i \in \Ns}$ the infinite-trials objectives are ordered according to:
    \[
        \frac{H(d^\pi)}{|\Ns|} \leq \frac{1}{|\Ns|}\sum_{i \in [\Ns]}H(d_i^\pi)  \leq H(\tilde d^\pi) \leq \sup_{i \in [\Ns]}H(d_i^\pi) + \log(|\Ns|) \leq  H(d^\pi) + \log(|\Ns|)
    \]
\end{restatable}
\vspace{-0.2cm}

The full derivation of these bounds is reported in Appendix~\ref{apx:proof}. This set of bounds demonstrates that the difference in performances over infinite-trials objective for the same policy can be generally bounded as a function of the number of agents. In particular, disjoint objectives generally provides poor approximations of the joint objective from the point of view of the single-agent, while the mixture objective is guaranteed to be a rather good lower bound to the joint entropy as well, since its over-estimation scales logarithmically with the number of agents.

It is still an open question how hard it is to actually optimize for these objectives. Now, while general cMGs $\mathcal M_{\mathcal F}$ are an interaction framework whose general properties are far from being well-understood, they surely enjoy some nice properties. In particular, it is possible to exploit the fact that performing Policy Gradient~\citep[PG,][]{sutton1999policy, peters2008reinforcement} independently among the agents is equivalent to running PG jointly, since this is done over the same common objective as for Potential Markov Games~\citep{leonardos2021globalconvergencemultiagentpolicy} (see Lemma~\ref{claim:projection} in Appendix~\ref{apx:theory}). This allows us to provide a rather positive answer, here stated informally and extensively discussed in Appendix~\ref{apx:theory}:

\begin{restatable}[(Informal) Sufficiency of Independent Policy Gradient]{fact}{sufficiencypga}
    \label{fact:sufficiencypga} 
    Under proper assumptions, for every cMG $\mathcal M_{\mathcal F}$, independent Policy Gradient over infinite trials non-disjoint objectives via centralized-information policies of the form $\pi = (\pi^i\in \Delta_\Ss^{\Acal^i})_{i \in [\Ns]}$ converges \emph{fast}.
\end{restatable}

This result suggests that PG should be generally enough for the infinite-trials optimization, and thus, in some sense, these problems might not be of so much interest. However, cMDP theory has outlined that optimizing for infinite-trials objectives might actually lead to extremely poor performance as soon as the policies are deployed over just a handful of trials, i.e. in almost any practical scenario~\citep{2023mutticonvexrlfinite}. We show that this property transfers almost seamlessly to cMGs as well, with interesting additional take-outs:

\begin{restatable}[Finite-Trials Mismatch in cMGs]{thr}{objectivemismatch}
    \label{thr:objectivemismatch} 
    For every cMG $\mathcal M_{\mathcal F}$ equipped with a $L$-Lipschitz function $\mathcal F$, let $K \in \mathbb N^+$ be a number of evaluation episodes/trials, and let $\delta \in (0, 1]$ be a confidence level, then for any (joint) policy $\pi = (\pi^i\in \Pi^i)_{i \in [\Ns]}$, it holds that
    \begin{equation*}
        |\zeta_K(\pi) - \zeta_\infty(\pi)| \leq  LT \sqrt{\frac{2 |\Ss| \log(2T/\delta)}{K}} ,\quad 
        |\zeta^i_K(\pi) - \zeta^i_\infty(\pi)| \leq  LT \sqrt{\frac{2|\tilde \Ss| \log(2T/\delta)}{K}},
    \end{equation*}
    \begin{equation*}
        |\tilde \zeta_K(\pi) - \tilde \zeta_\infty(\pi)| \leq  LT \sqrt{\frac{2|\tilde \Ss| \log(2T/\delta)}{|\Ns|K}}. 
    \end{equation*}
\end{restatable}
%\vspace{-0.2cm}

In general, this set of bounds confirms that for any given policy, infinite and finite trials performances might be extremely different, and thus optimizing the infinite-trials objective might lead to unpredictable performance at deployment, whenever this is done over a handful of trials. %\footnote{Yet, we notice that while it is possible to link globally optimal policies, this is not the case for general policies in the set of Nash Equilibria (NE)~\cite{nash51equilibria}.}. 
This property is inherently linked to the \emph{convex} nature of  cMGs, and \citet{2023mutticonvexrlfinite} introduced it for cMDPs to highlight that the concentration properties of empirical state-distributions~\cite{Weissman2003InequalitiesFT} allow for a nice dependency on the number of trials in controlling the mismatch. In multi-agent settings, the result portraits a more nuanced scene:

\emph{(i)}~The mismatch still scales with the cardinality of the support of the state distribution, yet, for joint objectives, this quantity scales very poorly in the number of agents.\footnote{Indeed, in the case of product state-spaces $\Ss = \times_{i \in [\Ns]} \Ss_i$ the cardinality scales exponentially with the number of agents $|\Ns|$.} Thus, even though optimizing infinite-trials joint objectives might be rather easy \emph{in theory} as Fact~\ref{fact:sufficiencypga} suggests, it might result in poor performances \emph{in practice}. On the other hand,  the quantity is independent of the number of agents for disjoint and mixture objectives.

\emph{(ii)}~Looking at mixture objectives, the mismatch scales sub-linearly with the number of agents $\Ns$. In some sense, the number of agents has the same role as the number of trials: The more the agents the less the deployment mismatch, and at the limit, with $\Ns \rightarrow \infty$, the mismatch vanishes completely.\footnote{In this scenario, all the bounds of Lemma~\ref{lem:entropymismatch} linking different objectives become vacuous.} In other words, this result portraits a striking difference with respect to joint objectives: When facing state entropy maximization over mixtures, a reasonably high number of agents compared to the size of the state-space actually helps, and simple policy gradient over mixture objectives might be enough.%\vspace{-4pt}

\textbf{Remark 1.}~~Although we do not claim that the mixture objective is the one-fits-all solution, it is nonetheless \emph{well-founded}. In particular whenever the rewards the agents will face in downstream tasks are equivalent for every agent, as it happens in relevant practical settings. When, on the other hand, the agents will aim to visit every joint state while solving for a specific task,\footnote{For instance, when for two agents the reward $r(s, s')$ is different from $r(s',s)$, i.e. the order matters.} the joint entropy objective is preferable, although it may be impractical: We report in Appendix~\ref{apx:comparison} an overall comparison of the two options, providing a motivating example as well.

\textbf{Remark 2.}~~Fact~\ref{fact:sufficiencypga} is valid for \emph{centralized-information} policies only. Up to our knowledge, no guarantees are known for \emph{decentralized-information} policies even in linear MGs. Interestingly though, the finite-trials formulation does offer additional insights on the behavior of optimal decentralized-information policies: The interested reader can learn more about this in Appendix~\ref{apx:policies}.%\vspace{-6pt}

%\vspace{-10pt}
\section{An Algorithm for Multi-Agent State Entropy Maximization in Practice}
\label{sec:algorithm}
%In the previous section, we showed that CMGs, and state entropy maximization in particular, might be easily optimized in theory by addressing infinite-trials objectives directly. On the other hand, we also showed that this might often lead to poor performances once such policies are deployed over a handful of trials. 
As stated before, a core drive of this work is addressing practical scenarios, where only a handful of trials can be drawn while interacting with the environment. Yet, Th.~\ref{thr:objectivemismatch} implies that optimizing for infinite-trials objectives, as with PG updates in Fact~\ref{fact:sufficiencypga}, might result in poor performance at deployment. As a result, here we address the third research question, that is:
\begin{tcolorbox}[colback=softbluegray, colframe=softbluegray,  boxrule=0.5pt, arc=4pt, width=\linewidth]
\begin{center}
    (\textbf{Q3}) \emph{Can we explicitly pre-train a policy for state entropy maximization \\ in practical multi-agent scenarios?} 
\end{center}
\end{tcolorbox}

To do so, we will shift our attention from infinite trials objectives to finite trials ones explicitly, more specifically on the single-trial case with $K=1$. Remarkably, it is possible to directly optimize the single-trial objective in multi-agent cases with decentralized algorithms: We introduce \emph{Trust Region Pure Exploration} (TRPE), the first decentralized algorithm that explicitly addresses single-trial objectives in cMGs, with state entropy maximization as a special case. TRPE takes inspiration from trust-region based methods as TRPO~\citep{schulman2017trustregionpolicyoptimization} due to their ability to address brittle optimization landscapes in which a small change into the policy parameters of each agent may drastically change the value of the objective function and the use of the trust region, like in TRPE, allows for accounting for this effect.\footnote{Previous works have connected the trust region with the natural gradient~\citep{pajarinen2019compatible}.}While the TRPE algorithm is new, the benefits of trust-region methods in multi-agent settings recently enjoyed an ubiquitous success and interest for their surprising effectiveness~\citep{yu2022surprisingeffectivenessppocooperative}.%\vspace{-8pt}
\begin{wrapfigure}{r}{0.58\textwidth}
\vspace{-10pt}
\begin{minipage}{\linewidth}
\begin{tcolorbox}
[colback=softbluegray!30, colframe=softbluegray, boxrule=0.5pt, arc=4pt, width=\linewidth, coltext=black, coltitle=black, title=\textbf{Algorithm}: Trust Region Pure Exploration (\textbf{TRPE})]
\refstepcounter{algorithm}
\label{alg:trpe}
\begin{algorithmic}[1]
    \STATE \textbf{Input}: exploration horizon $T$, trajectories $N$, trust-region threshold $\delta$, learning rate $\eta$
    \STATE Initialize $\vtheta = (\theta^i)_{i \in [\Ns]}$
    \FOR{epoch = $1, 2, \ldots$ until convergence}
        \STATE Collect $N$ trajectories with $\pi_{\vtheta}= (\pi^i_{\theta^i})_{i \in [\Ns]}$
        \FOR{agent $i=1, 2, \ldots$ \textbf{concurrently}}
            \STATE Set datasets $\mathcal{D}^i = \{ (\sbf^i_n,\va^i_n), \zeta_1^n\}_{n \in [N]}$
            \STATE $h = 0$, $\theta^i_h = \theta^i$
            \WHILE{$\kl(\pi^i_{ \theta^i_h } \| \pi^i_{\theta^i_0}) \leq \delta$}
                \STATE Compute $\hat{\mathcal L}^i(\theta^i_h/ \theta^i_0)$ via IS as in Eq.~\eqref{eq:empiricalis}
                \STATE $\theta^i_{h + 1} = \theta^i_{h} + \eta \nabla_{\theta^i_{h}} \hat{\mathcal L}^i(\theta^i_h/ \theta^i_0)$
                \STATE $h \gets h + 1$
            \ENDWHILE
            \STATE $\theta^i \gets \theta^i_h$
        \ENDFOR
    \ENDFOR
    \STATE \textbf{Output}: joint policy $\pi_{\vtheta}= (\pi^i_{\theta^i})_{i \in [\Ns]}$
\end{algorithmic}
\tcbset{label={alg:trpe}}
\end{tcolorbox}
\end{minipage}
\vspace{-30pt}
\end{wrapfigure}
In fact, trust-region analysis nicely align with the properties of finite-trials formulations and allow for an elegant extension to cMGs through the following.
\begin{restatable}[Surrogate Function over a Single Trial]{defi}{surrogate} \label{def:surrogate} For every cMG $\mathcal M_{\mathcal F}$ equipped with a $L$-Lipschitz function $\mathcal F$, let $d_1$ be a general single-trial distribution $d_1 = \{ d_1, d_{1,i}, \tilde d_1\}$, then for any per-agent deviation over policies $\pi = (\pi^i, \pi^{-i})$, $ \tilde \pi = (\tilde \pi^i, \pi^{-i})$, it is possible to define a per-agent \emph{Surrogate Function} $\mathcal L^i(\tilde \pi/\pi)$ of the form \(
        \mathcal L^i(\tilde \pi/\pi) = \E_{d_1\sim p^\pi_1} \rho^i_{\tilde \pi/ \pi} \mathcal F (d_1),\)
where $\rho^i$ is the per-agent importance-weight coefficient $\rho^i_{\tilde \pi/ \pi} = p^{\tilde \pi}_1 / p^\pi_1 = \prod_{t \in [T]} \frac{\tilde \pi^i(\va^i[t]|\sbf^i[t])}{ \pi^i(\va^i[t]|\sbf^i[t])}$.
\end{restatable}

From this definition, it follows that the trust-region algorithmic blueprint of~\citet{schulman2017trustregionpolicyoptimization} can be directly applied to single-trial formulations, with a parametric space of stochastic differentiable policies for each agent $\Theta= \{\pi^i_{\theta^i}: \theta^i \in \Theta^i \subseteq \mathbb R^q\}$. 

In practice, KL-divergence is employed for greater scalability provided a trust-region threshold $\delta$, we address the following optimization problem for each agent:\vspace{-4pt}
\begin{align*}
&\max_{\tilde \theta^i \in \Theta^i} \mathcal L^i(\tilde \theta^i/\theta^i) \qquad \text{s.t. } \kl(\pi^i_{\tilde \theta^i } \| \pi^i_{\theta^i}) \leq \delta \vspace*{-8pt}
\end{align*}
where we simplified the notation by letting $\mathcal L^i(\tilde \theta^i/\theta^i)  := \mathcal L^i(\pi^i_{\tilde \theta^i},\pi^{-i}_{ \theta^{-i}}/\pi_{\theta} )$.\footnote{More precisely, $ \mathcal {L}^i(\pi^i_{\tilde \theta^i}, \pi^{-i}_{ \theta^{-i}}/\pi_\theta) = \mathbb{E}_{d_1 \sim p_1^{\pi_\theta} } p_1^{\pi^i_{\tilde \theta^i}, \pi^{-i}_{ \theta^{-i}}}/ p_1^{\pi^{i}_{ \theta^{i}}, \pi^{-i}_{ \theta^{-i}}} \mathcal F (d_1)$.}

The main idea then follows from noticing that the surrogate function in Def.~\ref{def:surrogate} consists of an Importance Sampling (IS) estimator~\citep{mcbook}, and it is then possible to optimize it in a fully decentralized and off-policy manner~\citep{metelli2020pois, muttirestelli2020}. More specifically, given a pre-specified objective of interest $\zeta_1 \in \{\zeta_1,\zeta_1^i, \tilde \zeta_1\}$, agents sample $N$ trajectories $\{(\sbf_n, \va_n)\}_{n \in [N]}$ following a (joint) policy with parameters $\vtheta_0 = (\theta^i_0,\theta^{-i}_0)$. They then compute the values of the objective for each trajectory, building separate datasets $\mathcal{D}^i = \{ (\sbf^i_n,\va^i_n), \zeta_1^n\}_{n \in [N]}$ and using it to compute the Monte-Carlo approximation of the surrogate function, namely 
\begin{align}
&\hat{\mathcal L}^i(\theta^i_h/ \theta^i_0) = \frac{1}{N}\sum_{n\in[N]} \rho^{i,n}_{\theta^i_h/ \theta^i_0} \zeta^n_1, \quad \rho^{i,n}_{\theta^i_h/ \theta^i_0} = \prod_{t \in [T]} \pi^i_{\theta^i_h}(\va^i_n[t]|\sbf^i_n[t])/ \pi^i_{\theta^i_0}(\va^i_n[t]|\sbf^i_n[t]), \label{eq:empiricalis}
%\vspace{-0.3cm}
\end{align}
and $\zeta^n_1$ is the plug-in estimator of the entropy based on the empirical measure $d_1$~\citep{paninski2003}. 
Finally, at each off-policy iteration $h$, each agent updates its parameter via gradient ascent
$
    \theta^i_{h+1} \leftarrow \theta^i_{h} + \eta \nabla_{\theta^i_h} \hat{\mathcal L}^i(\theta^i_h/ \theta^i_0) 
$
until the trust-region boundary is reached, i.e., when it holds $
    \kl(\pi^i_{\tilde \theta^i } \| \pi^i_{\theta^i}) > \delta.
$
The psudo-code of TRPE is reported in Algorithm~\ref{alg:trpe}. We remark that even though TRPE is applied to state entropy maximization, the algorithmic blueprint does not explicitly require the function $\mathcal F$ to be the entropy function and thus it is of independent interest.

\textbf{Limitations.}~~The main limitations of the proposed methods are two. First, the Monte-Carlo estimation of single-trial objectives might be sample-inefficient in high-dimensional tasks. However, more efficient estimators of single-trial objectives remain an open question in single-agent convex RL as well, as the convex nature of the problem hinders the applicability of Bellman operators. Secondly, the plug-in estimator of the entropy is applicable to discrete spaces only, but designing scalable estimators of the entropy in continuous domains is usually a contribution \emph{per se}~\citep{mutti2021taskagnosticexplorationpolicygradient}.
%\vspace{-6pt}
%\vspace{-10pt}
\section{Empirical Corroboration}
\label{sec:experiments}%\vspace{-10pt}
In this section, we address the last research question, that is:

\begin{tcolorbox}[colback=softbluegray, colframe=softbluegray,  boxrule=0.5pt, arc=4pt, width=\linewidth]
\begin{center}
    (\textbf{Q4}) \emph{Do crucial differences emerge in practice? Does this have \\ an impact on downstream tasks learning?}
\end{center}
\end{tcolorbox}

by providing empirical corroboration of the findings discussed so far. Especially, we aim to answer the following questions: (\textbf{Q4.1}) Is Algorithm~\ref{alg:trpe} actually capable of optimizing finite-trials objectives? (\textbf{Q4.2}) Do different objectives enforce different behaviors, as expected from Section~\ref{sec:problem_formulation}? (\textbf{Q4.3}) Does the \emph{clustering} behavior of mixture objectives play a crucial role? If yes, when and why?

Throughout the experiments, we will compare the result of optimizing finite-trial objectives, either joint, disjoint, mixture ones, through Algorithm~\ref{alg:trpe} via fully decentralized policies. The experiments will be performed with different values of the exploration horizon $T$, so as to test their capabilities in different exploration efficiency regimes.\footnote{The exploration horizon $T$, rather than being a given trajectory length, has to be seen as a parameter of the exploration phase which allows to tradeoff exploration quality with exploration efficiency.} The full implementation details are reported in Appendix~\ref{apx:exp}.%\vspace{-4pt}

\textbf{Experimental Domains.}~~The experiments were performed with the aim to illustrate essential features of state entropy maximization suggested by the theoretical analysis, and for this reason the domains were selected for being challenging while keeping high interpretability. The first is a notoriously difficult multi-agent exploration task called \emph{secret room}~\citep[MPE,][]{liu2021cooperative},\footnote{We highlight that all previous efforts in this task employed centralized-information policies. On the other hand, we are interested on the role of the entropic feedback in fostering coordination rather than full-state conditioning, thus we employed decentralized-information policies.} referred to as  Env.~(\textbf{i}). In such task, two agents are required to reach a target while navigating over two rooms divided by a door. In order to keep the door open, at least one agent have to remain on a switch. Two switches are located at the corners of the two rooms. The hardness of the task then comes from the need of coordinated exploration, where one agent allows for the exploration of the other. The second is a simpler exploration task yet over a high dimensional state-space, namely a 2-agent instantiation of \emph{Reacher}~\citep[MaMuJoCo,][]{peng2021facmac}, referred to as Env.~(\textbf{ii}). Each agent corresponds to one joint and equipped with decentralized-information policies. In order to allow for the use of plug-in estimator of the entropy~\citep{paninski2003}, each state dimension was discretized over 10 bins.%\vspace{-4pt}
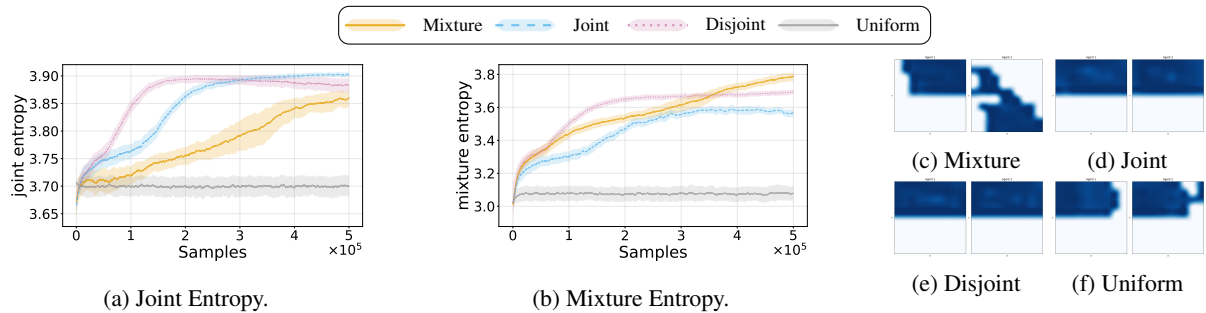
\begin{figure*}[!t]
\vspace{-6pt}

\begin{minipage}[t]{0.64\textwidth} % Left side: legend and 2 figures
    % Legend above first two plots
    \centering
    \hspace{15pt}
    \begin{tikzpicture}
    % Draw rounded box for the legend
    \node[draw=black, rounded corners, inner sep=2pt, fill=white] (legend) at (0,0) {
        \begin{tikzpicture}[scale=0.8]
            % Mixture
            \draw[thick, color={rgb,255:red,230; green,159; blue,0}, opacity=0.8] (0,0) -- (1,0);
            \fill[color={rgb,255:red,230; green,159; blue,0}, opacity=0.2] (0,-0.1) rectangle (1,0.1);
            \node[anchor=west, font=\scriptsize] at (1.2,0) {Mixture};
            
            % Joint
            \draw[thick, dashed, color={rgb,255:red,86; green,180; blue,233}, opacity=0.8] (2.5,0) -- (3.5,0);
            \fill[color={rgb,255:red,86; green,180; blue,233}, opacity=0.2] (2.5,-0.1) rectangle (3.5,0.1);
            \node[anchor=west, font=\scriptsize] at (3.7,0) {Joint};

            % Disjoint
            \draw[thick, dotted, color={rgb,255:red,204; green,121; blue,167}, opacity=0.8] (4.7,0) -- (5.7,0);
            \fill[color={rgb,255:red,204; green,121; blue,167}, opacity=0.2] (4.7,-0.1) rectangle (5.7,0.1);
            \node[anchor=west, font=\scriptsize] at (5.9,0) {Disjoint};
            
            % Uniform
            \draw[thick, color={rgb,255:red,153; green,153; blue,153}, opacity=0.8] (7.2,0) -- (8.2,0);
            \fill[color={rgb,255:red,153; green,153; blue,153}, opacity=0.2] (7.2,-0.1) rectangle (8.2,0.1);
            \node[anchor=west, font=\scriptsize] at (8.4,0) {Uniform};
        \end{tikzpicture}
    };
\end{tikzpicture}
    %\vspace{2pt}

    \begin{subfigure}[t]{0.49\textwidth}
        \includegraphics[width=\textwidth]{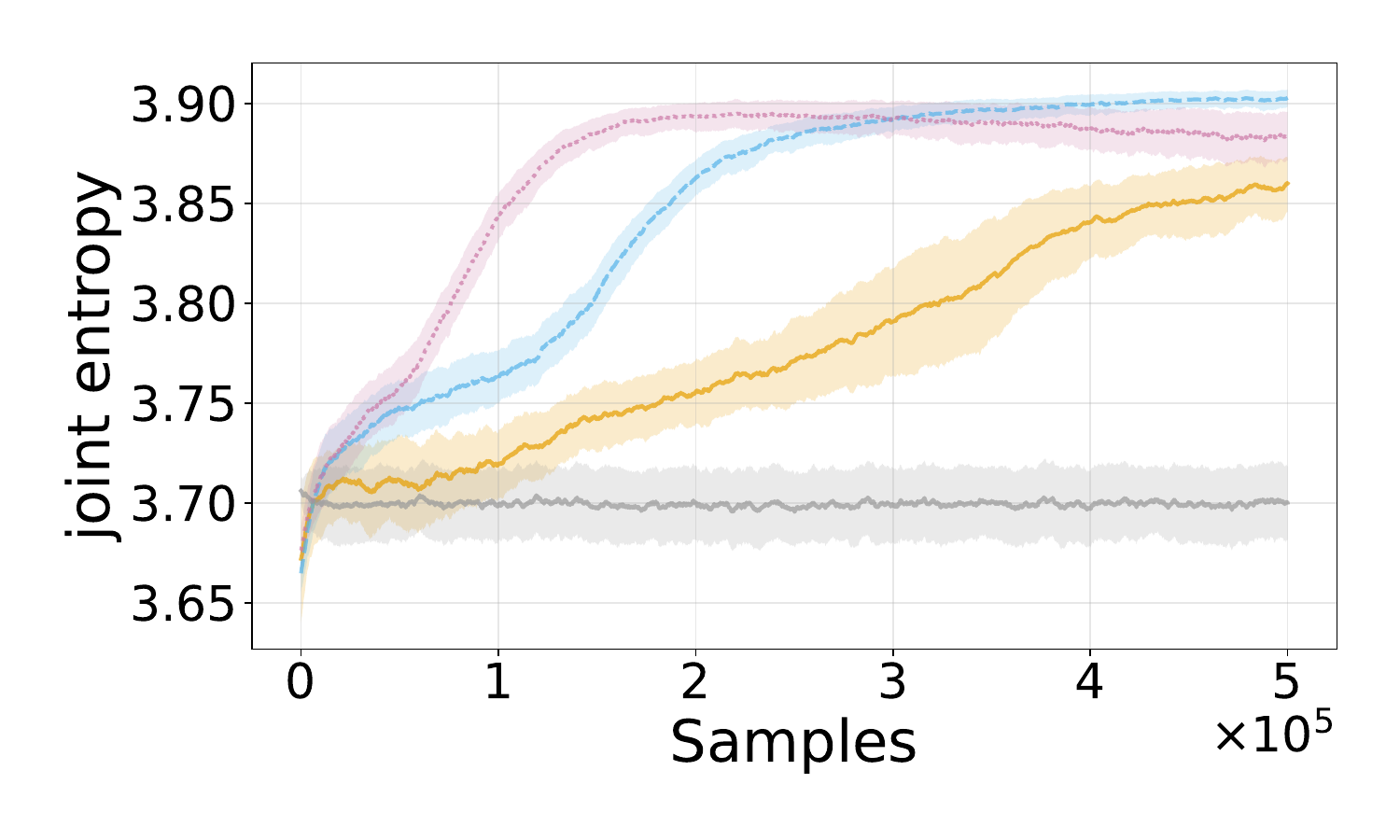}
        \vspace{-18pt}
        \caption{\centering Joint Entropy}
        \label{subfig:image2}
    \end{subfigure}
    \hfill
    \begin{subfigure}[t]{0.49\textwidth}
        \includegraphics[width=\textwidth]{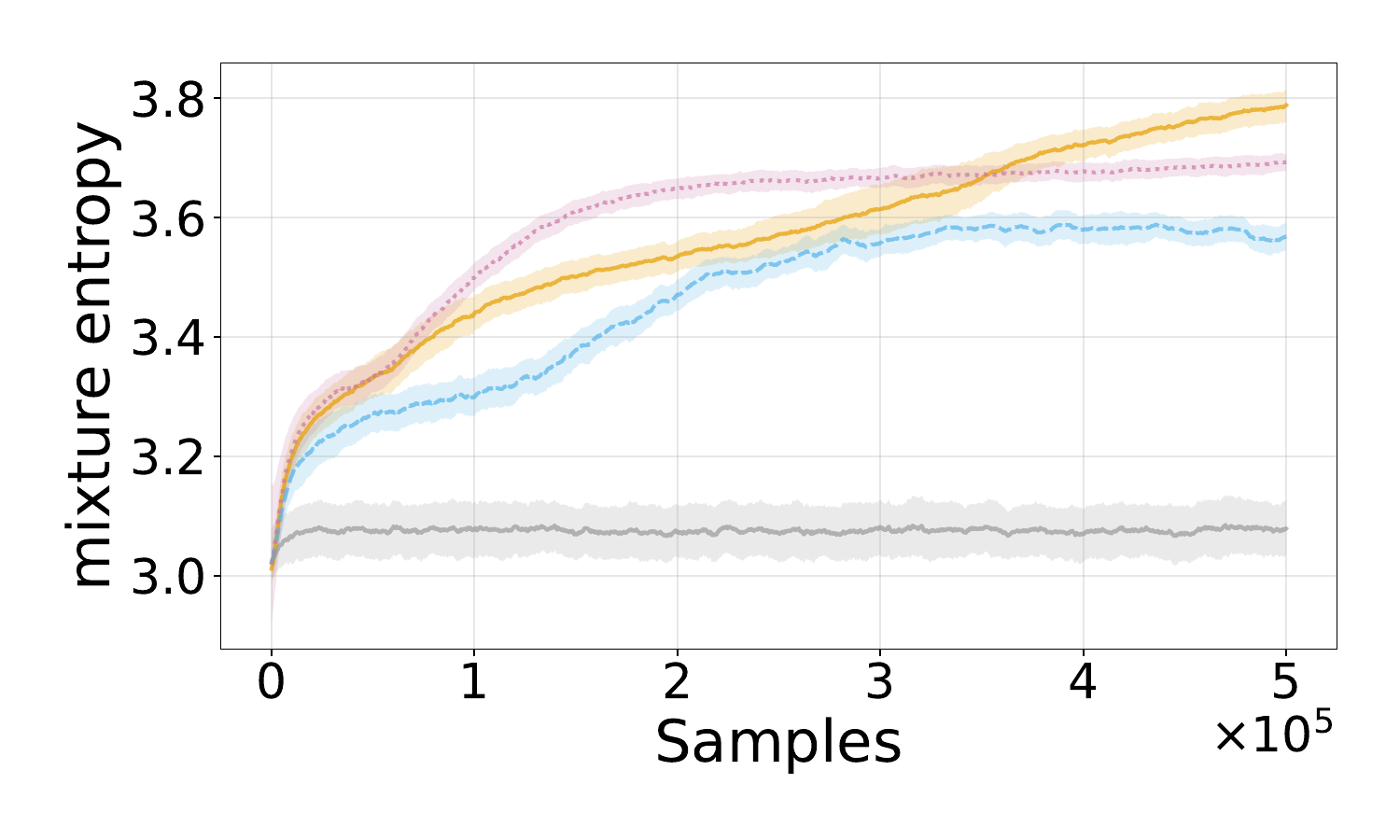}
        \vspace{-18pt}
        \caption{\centering Mixture Entropy}
        \label{subfig:image3}
    \end{subfigure}
\end{minipage}
\hfill
\raisebox{-17pt}{
\begin{minipage}[t]{0.34\textwidth} % Right side: raised 4 subfigures
    \begin{subfigure}[t]{0.48\textwidth}
        \includegraphics[width=\textwidth]{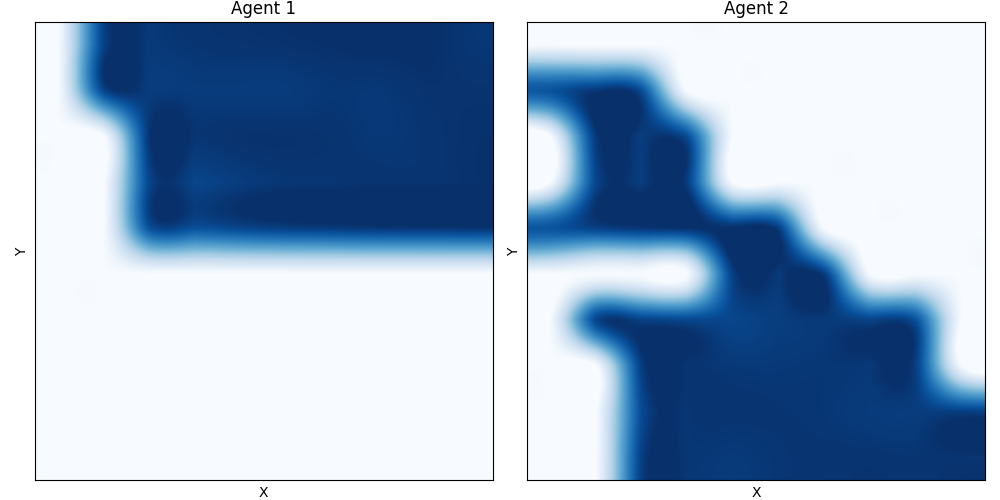}
        \caption{\centering Mixture}
        \label{subfig:image7}
    \end{subfigure}
    \hfill
    \begin{subfigure}[t]{0.48\textwidth}
        \includegraphics[width=\textwidth]{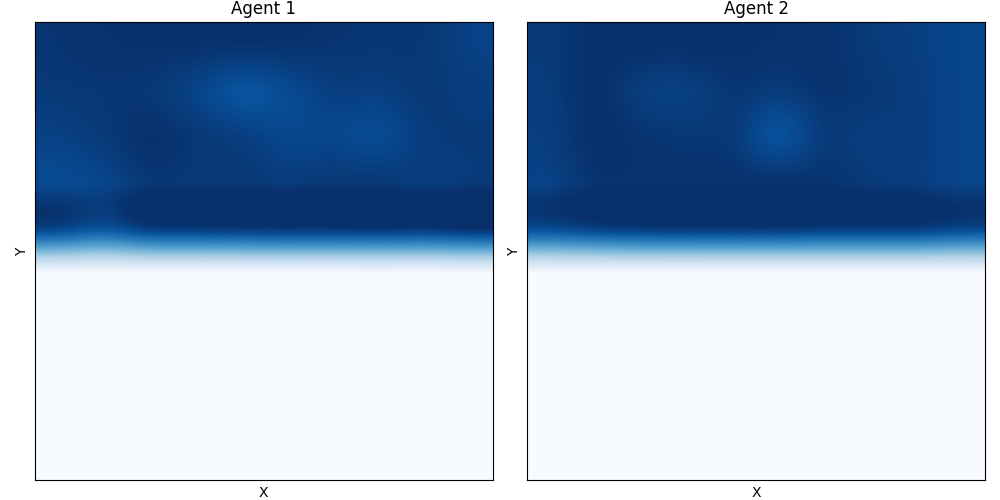}
        \caption{\centering Joint}
        \label{subfig:image5}
    \end{subfigure}
    \vspace{2pt}
    
    \begin{subfigure}[t]{0.48\textwidth}
        \includegraphics[width=\textwidth]{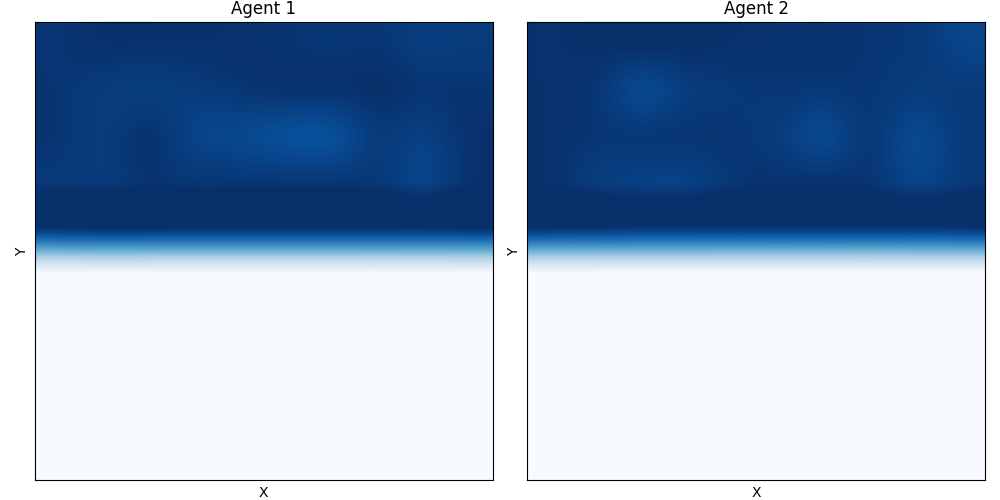}
        \caption{\centering Disjoint}
        \label{subfig:image6}
    \end{subfigure}
    \hfill
    \begin{subfigure}[t]{0.48\textwidth}
        \includegraphics[width=\textwidth]{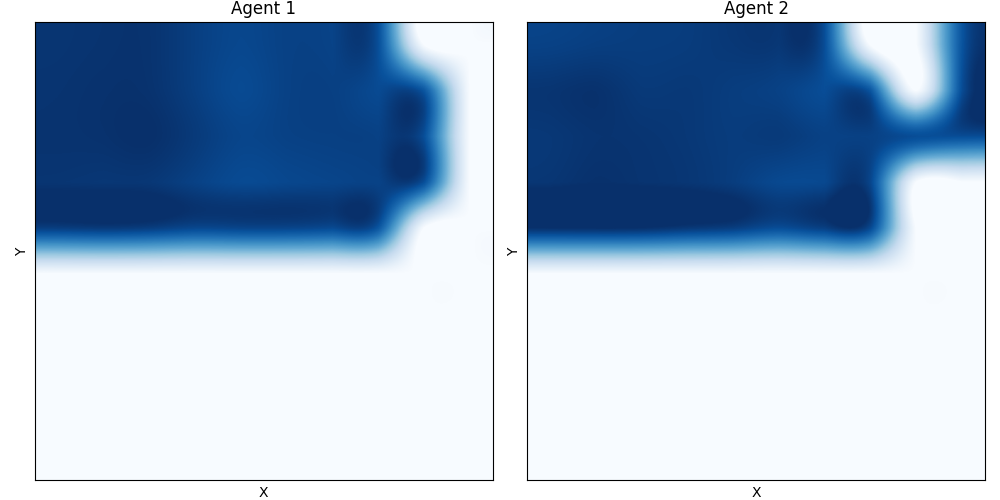}
        \caption{\centering Uniform}
        \label{subfig:image4}
    \end{subfigure}
\end{minipage}
}
\vspace{-5pt}
\caption{Single-trial Joint and Mixture Entropy induced by different objective optimization along a $T = 50$ horizon. (\emph{Right}) State Distributions of two agents induced by different learned policies. We report the average and 95\% confidence interval over 4 runs.}
\label{fig:room}
\vspace{-15pt}
\end{figure*}

\textbf{State Entropy Maximization.}~~As common for the unsupervised RL framework~\citep{hazan2019provably, laskin2021urlb, liu2021behavior, mutti2021taskagnosticexplorationpolicygradient}, Algorithm~\ref{alg:trpe} was first tested in her ability to optimize for state entropy maximization objectives, thus in environments \emph{without rewards}. In Figure~\ref{fig:room}, we report the results for a short, and thus more challenging, exploration horizon $(T=50)$ over Env.~(\textbf{i}), as it is far more interpretable. Other experiments with longer horizons or over Env.~(\textbf{ii}) can be found in Appendix~\ref{apx:exp}. Interestingly, at this challenging exploration regime, when looking at the joint entropy in Figure~\ref{subfig:image2}, joint and disjoint objectives perform rather well compared to mixture ones in terms of induced joint entropy, while they fail to address mixture entropy explicitly, as seen in Figure~\ref{subfig:image3}. On the other hand mixture-based objectives result in optimizing both mixture \emph{and} joint entropy effectively, as one would expect by the bounds in Th.~\ref{lem:entropymismatch}. By looking at the actual state visitation induced by the trained policies, the difference between the objectives is apparent. While optimizing joint objectives, agents exploit the high-dimensionality of the joint space to induce highly entropic distributions even without exploring the space uniformly via coordination (Fig.~\ref{subfig:image5}); the same outcome happens in disjoint objectives, with which agents focus on over-optimizing over a restricted space loosing any incentive for coordinated exploration (Fig. \ref{subfig:image6}). On the other hand, mixture objectives enforce a clustering behavior (Fig. \ref{subfig:image7}) and result in a better efficient exploration.\footnote{While it is true that mixture objectives optimization appears to lead to slower optimization, this is the result of such pathological behaviors.}%\vspace{-4pt}

\textbf{Policy Pre-Training via State Entropy Maximization.}~~Importantly, while metrics in Fig.~\ref{fig:room} are indeed interesting qualitative metrics, especially to understand how the unsupervised optimization process works, they do not fully capture the ultimate goal in a vacuum: the ultimate goal of unsupervised (MA)RL is to provide good pre-trained models for (MA)RL. As such, the most important experimental metric to look at is the return achieve in downstream tasks, where the policy optimizing the mixture entropy fares well in comparison to others. Thus, we tested the effect of pre-training policies via state entropy maximization as a way to alleviate the well-known hardness of \emph{sparse-reward} settings. In order to do so, we employed a multi-agent counterpart of the TRPO algorithm~\cite{schulman2017trustregionpolicyoptimization} with different pre-trained policies. First, we investigated the effect on the learning curve in the hard-exploration task of Env.~(\textbf{i}) under long horizons ($T=150$), with a \emph{worst-case goal} set on the opposite corner of the closed room. Pre-training via mixture objectives still lead to a faster learning compared to initializing the policy with a uniform distribution. On the other hand, joint objective pre-training did not lead to substantial improvements over standard initializations. More interestingly, when extremely short horizons were taken into account ($T=50$) the difference became appalling, as shown in Fig.~\ref{subfig:image9}: pre-training via mixture-based objectives lead to faster learning and higher performances, while pre-training via disjoint objectives turned out to be even \emph{harmful} (Fig.~\ref{subfig:image10}). This was motivated by the fact that the disjoint objective overfitted the task over the states reachable without coordinated exploration, resulting in almost deterministic policies, as shown in Fig.~\ref{fig:333} in Appendix~\ref{apx:exp}. Finally, we tested the zero-shot capabilities of policy pre-training on the simpler but high dimensional exploration task of Env.~(\textbf{ii}), where the goal was sampled randomly between \emph{worst-case positions} at the boundaries of the region reachable by the arm. As shown in Fig.~\ref{subfig:image11}, both joint and mixture were able to guarantee zero-shot performances via pre-training compatible with MA-TRPO after learning over $2$e$4$ samples, while disjoint objectives were not. On the other hand, pre-training with joint objectives showed an extremely high-variance, leading to worst-case performances not better than the ones of random initialization. Mixture objectives on the other hand showed higher stability in guaranteeing compelling zero-shot performance. These results are the first to extend findings from single-agent environments~\citep{zisselman2023explore} to multi-agent ones.%\vspace{-10pt}
\begin{figure*}[!]
    \centering
    \begin{tikzpicture}
    % Draw rounded box for the legend
    \node[draw=black, rounded corners, inner sep=2pt, fill=white] (legend) at (0,0) {
        \begin{tikzpicture}[scale=0.8]
            % Mixture
            \draw[thick, color={rgb,255:red,230; green,159; blue,0}, opacity=0.8] (0,0) -- (1,0);
            \fill[color={rgb,255:red,230; green,159; blue,0}, opacity=0.2] (0,-0.1) rectangle (1,0.1);
            \node[anchor=west, font=\scriptsize] at (1.2,0) {Mixture};
            
            % Joint
            \draw[thick, dashed, color={rgb,255:red,86; green,180; blue,233}, opacity=0.8] (2.5,0) -- (3.5,0);
            \fill[color={rgb,255:red,86; green,180; blue,233}, opacity=0.2] (2.5,-0.1) rectangle (3.5,0.1);
            \node[anchor=west, font=\scriptsize] at (3.7,0) {Joint};

            % Disjoint
            \draw[thick, dotted, color={rgb,255:red,204; green,121; blue,167}, opacity=0.8] (4.7,0) -- (5.7,0);
            \fill[color={rgb,255:red,204; green,121; blue,167}, opacity=0.2] (4.7,-0.1) rectangle (5.7,0.1);
            \node[anchor=west, font=\scriptsize] at (5.9,0) {Disjoint};
            
            % Uniform
            \draw[thick, color={rgb,255:red,153; green,153; blue,153}, opacity=0.8] (7.2,0) -- (8.2,0);
            \fill[color={rgb,255:red,153; green,153; blue,153}, opacity=0.2] (7.2,-0.1) rectangle (8.2,0.1);
            \node[anchor=west, font=\scriptsize] at (8.4,0) {Random Initialization};
        \end{tikzpicture}
    };
\end{tikzpicture}
    %\hfill
    \vfill
    %vspace{-0.2cm}
    \begin{subfigure}[b]{0.32\textwidth}
        \includegraphics[width=\textwidth]{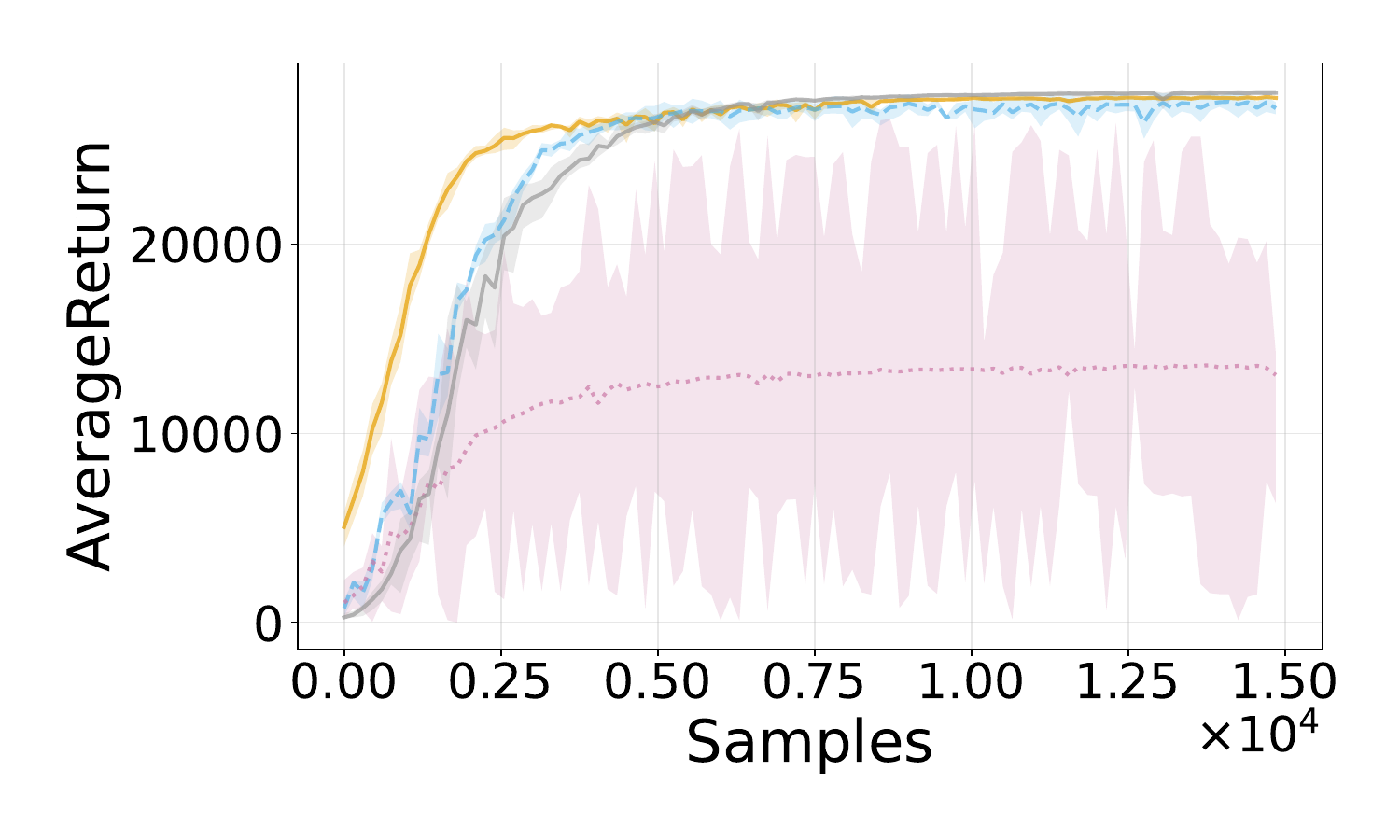}
        \vspace{-18pt}
        \caption{\centering MA-TRPO + TRPE Pre-Training (Env.~(\textbf{i}), $T=150$)}
        \label{subfig:image9}
    \end{subfigure}
    \hfill
    \begin{subfigure}[b]{0.32\textwidth}
        \includegraphics[width=\textwidth]{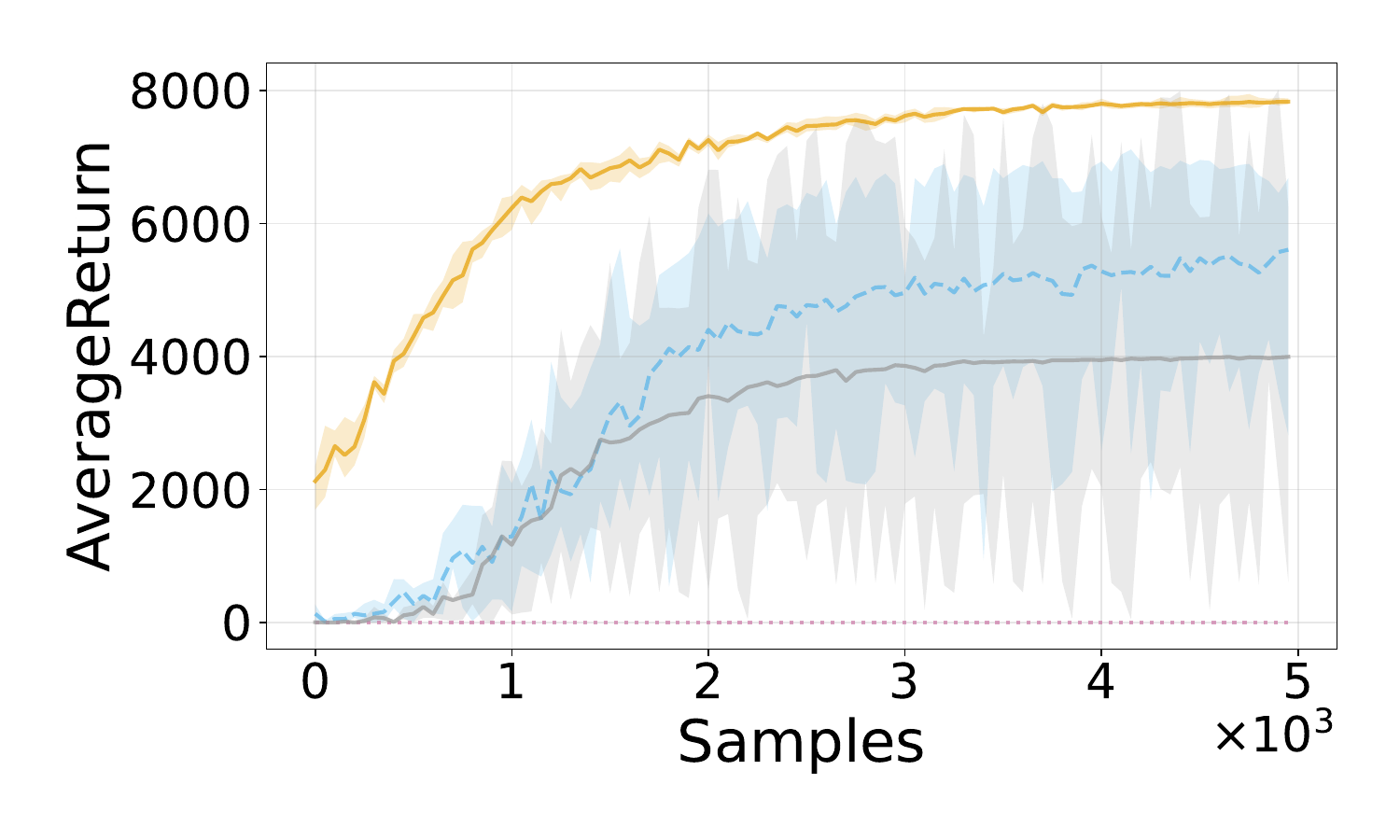}
        \vspace{-18pt}
        \caption{\centering MA-TRPO + TRPE Pre-Training (Env.~(\textbf{i}), $T=50$)}
        \label{subfig:image10}
    \end{subfigure}
    \hfill
    \begin{subfigure}[b]{0.34\textwidth}
        \centering
        \includegraphics[width=0.8\textwidth]{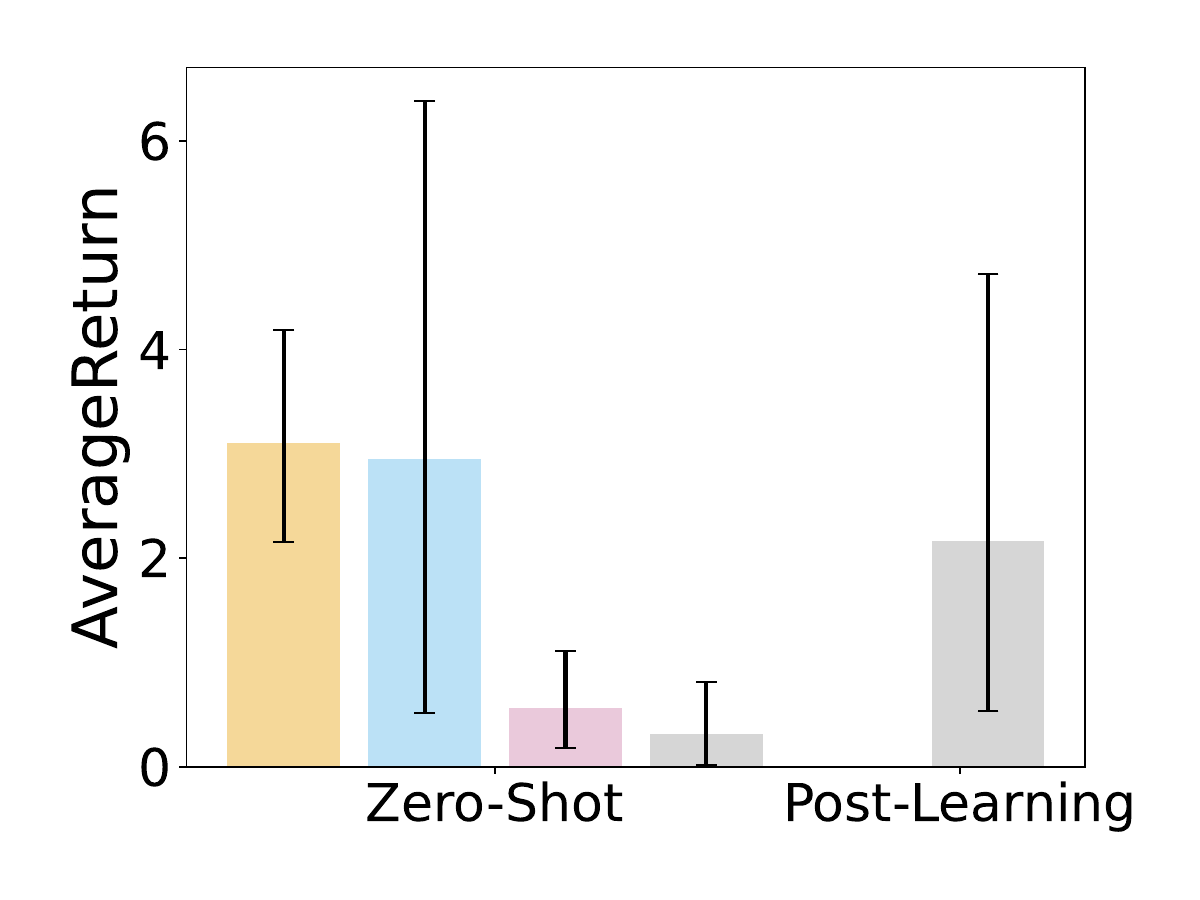}
        \vspace{-10pt}
        \caption{\centering MA-TRPO + TRPE Pre-Training (Env.~(\textbf{ii}), $T=100$)}
        \label{subfig:image11}
    \end{subfigure}
\caption{Effect of pre-training in sparse-reward settings. (\emph{Left}) Policies initialized with either Uniform or TRPE pre-trained policies. (\emph{Right}) Policies initialized with either Zero-Mean or TRPE pre-trained policies. We report the average and 95\% c.i. over 4 runs over worst-case goals.}
\label{fig:pretraining}
\vspace{-10pt}
\end{figure*}
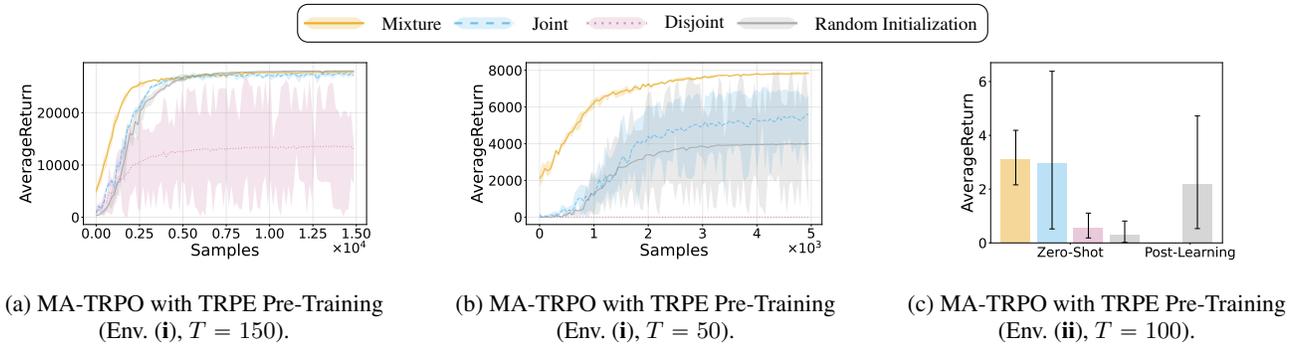

\textbf{Takeaways.}~~Overall, the proposed experiments managed to answer to all of the experimental questions:~(\textbf{Q4.1})~Algorithm~\ref{alg:trpe} is indeed able to optimize for finite-trial objectives;~(\textbf{Q4.2})~\textbf{Mixture objectives enforce coordination}, essential when high efficiency is required, while joint or disjoint objectives may fail to lead to relevant solutions because of under or over optimization;~(\textbf{Q4.3})~\textbf{The efficient coordination} through mixture objectives enforces the ability of \textbf{pre-training via state entropy maximization} to lead to \textbf{faster and better training} and even \textbf{zero-shot generalization}.%\vspace{-6pt}
%\newpage
\section{Related Works}
\label{sec:relatedworks}%\vspace{-6pt}
Below, we summarize the most relevant work investigating related concepts.%\vspace{-4pt}

\textbf{Entropic Functionals in MARL.}~~A large plethora of works on both swarm robotics~\citep{mclurkin2005dynamic, breitenmoser2010voronoi} and multi-agent intrinsic motivation, such as~\citep[][]{ iqbal2019coordinated,yang2021ciexplore, zhang2021made, zhang2023self, xu2024population, toquebiau2024joint}, investigated the effects of employing entropic-like functions to boost exploration and performances in down-stream tasks. Importantly, these works are of empirical nature, and they do not investigate the theoretical properties of cMGs or multi-agent state entropy maximization, nor they propose algorithms able to pre-train policies \emph{without access} to extrinsic rewards.\footnote{The interested reader can refer to~\citet{mutti2021taskagnosticexplorationpolicygradient, liu2021behavior} for an extensive investigation of the fundamental differences between intrinsic motivation and state entropy maximization.} Finally, while a similar notion of cMGs was proposed in~\citep{gemp2025convexmarkovgamesframework, kalogiannis2025solving}, their contributions are focused on the existence and computation of equilibria and performance of centralized algorithms over infinite-trials objectives.

\textbf{State Entropy Maximization.}~~Entropy maximization in MDPs was first introduced in~\citet{hazan2019provably} and then investigated extensively in various subsequent works~\citep[e.g.,][]{muttirestelli2020, mutti2021taskagnosticexplorationpolicygradient, mutti2022importance, mutti2022unsupervised, mutti2023unsupervised, liu2021behavior, liu2021aps, seo2021state, yarats2021reinforcement, zhang2020exploration,guo2021geometric, yuan2022renyi, nedergaard2022k, yang2023cem,  pmlr-v202-tiapkin23a, jain2023maximum, kim2023accelerating, zisselman2023explore, li2024element, bolland2024off, zamboni2024limits, zamboni2024explore, depaola2025enhancing}. Its infinite-trials formulation\footnote{Conversely, the finite-trial formulation targeted by Algorithm~\ref{alg:trpe} is not studied in the literature of regularized MDPs.} can also be seen as a particular reward-free instance of state-entropy regularized MDPs~\citep{brekelmans2022your, ashlag2025state}, although this reduction does not alleviate the aforementioned criticalities in solving such problems in multi-agent settings. To the best of our knowledge, our work is the first to study a multi-agent variation of the state entropy maximization problem.

\textbf{Policy Optimization.}~~
Finally, our algorithmic solution (Algorithm~\ref{alg:trpe}) draws heavily on the literature of policy optimization and trust-region methods~\citep{schulman2017trustregionpolicyoptimization}.
Specifically, we considered an IS policy gradient estimator, which is partially inspired by the work of~\citet{metelli2020pois}, but considers other forms of IS estimators, such as non-parametric k-NN estimators previously employed in~\citet{mutti2021taskagnosticexplorationpolicygradient}. %\vspace{-6pt}
\section{Conclusions and Perspectives}
\label{sec:conclusions}%\vspace{-6pt}

In this paper, we introduce a principled framework for unsupervised pre-training in MARL via state entropy maximization. First, we formalize the problem as a convex generalization of Markov Games, and show that it can be defined via several different objectives: one can look at the joint distribution among all the agents, the marginals which are agent-specific, or the mixture which is a tradeoff of the two. Thus, we link these three options via performance bounds and we theoretically characterize how the problem, even when tractable in theory, is non-trivial in practice. Then, we design a practical algorithm and we use it in a set of experiments to confirm the expected superiority of mixture objectives in practice, due to their ability to enforce efficient coordination over short horizons. Future works can build over our results in many directions, for instance by pushing forward the knowledge on convex Markov Games, developing scalable algorithms for continuous domains, or performing extensive empirical investigation over large scale problems. We believe that our work can be a crucial step in the direction of extending policy pre-training via state entropy maximization in a principled way to yet more practical settings.

%%%%%%%%%%%%%%%%%%%%%%%%%%%%%%%%%%%%%%%%%%%%%%%%%%%%%%%%%%%%

\bibliographystyle{plainnat}
\bibliography{biblio}

\begin{thebibliography}{70}
\providecommand{\natexlab}[1]{#1}
\providecommand{\url}[1]{\texttt{#1}}
\expandafter\ifx\csname urlstyle\endcsname\relax
  \providecommand{\doi}[1]{doi: #1}\else
  \providecommand{\doi}{doi: \begingroup \urlstyle{rm}\Url}\fi

\bibitem[Agarwal et~al.(2022)Agarwal, Schwarzer, Castro, Courville, and Bellemare]{agarwal2022reincarnating}
Rishabh Agarwal, Max Schwarzer, Pablo~Samuel Castro, Aaron~C Courville, and Marc Bellemare.
\newblock Reincarnating {R}einforcement {L}earning: Reusing prior computation to accelerate progress.
\newblock \emph{Advances in neural information processing systems}, 2022.

\bibitem[Albrecht et~al.(2024)Albrecht, Christianos, and Sch\"afer]{marlbook}
Stefano~V. Albrecht, Filippos Christianos, and Lukas Sch\"afer.
\newblock \emph{Multi-Agent Reinforcement Learning: Foundations and Modern Approaches}.
\newblock MIT Press, 2024.
\newblock URL \url{https://www.marl-book.com}.

\bibitem[Ashlag et~al.(2025)Ashlag, Koren, Mutti, Derman, Bacon, and Mannor]{ashlag2025state}
Yonatan Ashlag, Uri Koren, Mirco Mutti, Esther Derman, Pierre-Luc Bacon, and Shie Mannor.
\newblock State entropy regularization for robust reinforcement learning.
\newblock \emph{arXiv preprint arXiv:2506.07085}, 2025.

\bibitem[Beirlant et~al.(1997)Beirlant, Dudewicz, Gy{\"o}rfi, Van~der Meulen, et~al.]{beirlant1997nonparametric}
Jan Beirlant, Edward~J Dudewicz, L{\'a}szl{\'o} Gy{\"o}rfi, Edward~C Van~der Meulen, et~al.
\newblock Nonparametric entropy estimation: An overview.
\newblock \emph{International Journal of Mathematical and Statistical Sciences}, 6\penalty0 (1):\penalty0 17--39, 1997.

\bibitem[Bertsekas and Tsitsiklis(2002)]{bertsekas2002introduction}
Dimitri~P Bertsekas and John~N Tsitsiklis.
\newblock Introduction to probability (athena scientific, belmont, ma).
\newblock \emph{EKLER Ek A: S{\i}ral{\i} Istatistik Ek B: Integrallerin Say{\i}sal Hesab{\i} Ek B}, 1, 2002.

\bibitem[Bolland et~al.(2024)Bolland, Lambrechts, and Ernst]{bolland2024off}
Adrien Bolland, Gaspard Lambrechts, and Damien Ernst.
\newblock Off-policy maximum entropy rl with future state and action visitation measures.
\newblock \emph{arXiv preprint arXiv:2412.06655}, 2024.

\bibitem[Breitenmoser et~al.(2010)Breitenmoser, Schwager, Metzger, Siegwart, and Rus]{breitenmoser2010voronoi}
Andreas Breitenmoser, Mac Schwager, Jean-Claude Metzger, Roland Siegwart, and Daniela Rus.
\newblock Voronoi coverage of non-convex environments with a group of networked robots.
\newblock In \emph{IEEE international conference on robotics and automation}, 2010.

\bibitem[Brekelmans et~al.(2022)Brekelmans, Genewein, Grau-Moya, Del{\'e}tang, Kunesch, Legg, and Ortega]{brekelmans2022your}
Rob Brekelmans, Tim Genewein, Jordi Grau-Moya, Gr{\'e}goire Del{\'e}tang, Markus Kunesch, Shane Legg, and Pedro Ortega.
\newblock Your policy regularizer is secretly an adversary.
\newblock \emph{arXiv preprint arXiv:2203.12592}, 2022.

\bibitem[De~Paola et~al.(2025)De~Paola, Zamboni, Mutti, and Restelli]{depaola2025enhancing}
Vincenzo De~Paola, Riccardo Zamboni, Mirco Mutti, and Marcello Restelli.
\newblock Enhancing diversity in parallel agents: A maximum state entropy exploration story.
\newblock In \emph{Internation Conference on Machine Learning}, 2025.

\bibitem[Duan et~al.(2016)Duan, Chen, Houthooft, Schulman, and Abbeel]{duan2016benchmarking}
Yan Duan, Xi~Chen, Rein Houthooft, John Schulman, and Pieter Abbeel.
\newblock Benchmarking deep {R}einforcement {L}earning for continuous control.
\newblock In \emph{International Conference on Machine Learning}, 2016.

\bibitem[Gemp et~al.(2024)Gemp, Haupt, Marris, Liu, and Piliouras]{gemp2025convexmarkovgamesframework}
Ian Gemp, Andreas Haupt, Luke Marris, Siqi Liu, and Georgios Piliouras.
\newblock Convex markov games: A framework for creativity, imitation, fairness, and safety in multiagent learning.
\newblock \emph{arXiv preprint arXiv:2410.16600}, 2024.

\bibitem[Guo et~al.(2021)Guo, Azar, Saade, Thakoor, Piot, Pires, Valko, Mesnard, Lattimore, and Munos]{guo2021geometric}
Zhaohan~Daniel Guo, Mohammad~Gheshlagi Azar, Alaa Saade, Shantanu Thakoor, Bilal Piot, Bernardo~Avila Pires, Michal Valko, Thomas Mesnard, Tor Lattimore, and R{\'e}mi Munos.
\newblock Geometric entropic exploration.
\newblock \emph{arXiv preprint arXiv:2101.02055}, 2021.

\bibitem[Hazan et~al.(2019)Hazan, Kakade, Singh, and Van~Soest]{hazan2019provably}
Elad Hazan, Sham Kakade, Karan Singh, and Abby Van~Soest.
\newblock Provably efficient maximum entropy exploration.
\newblock In \emph{International Conference on Machine Learning}, 2019.

\bibitem[Hu et~al.(2020)Hu, Lerer, Peysakhovich, and Foerster]{hu2020other}
Hengyuan Hu, Adam Lerer, Alex Peysakhovich, and Jakob Foerster.
\newblock “{O}ther-{P}lay” for zero-shot coordination.
\newblock In \emph{International Conference on Machine Learning}, 2020.

\bibitem[Iqbal and Sha(2019)]{iqbal2019coordinated}
Shariq Iqbal and Fei Sha.
\newblock Coordinated exploration via intrinsic rewards for multi-agent {R}einforcement {L}earning.
\newblock \emph{arXiv preprint arXiv:1905.12127}, 2019.

\bibitem[Jain et~al.(2023)Jain, Lehnert, Rish, and Berseth]{jain2023maximum}
Arnav~Kumar Jain, Lucas Lehnert, Irina Rish, and Glen Berseth.
\newblock Maximum state entropy exploration using predecessor and successor representations.
\newblock In \emph{Advances in Neural Information Processing Systems}, 2023.

\bibitem[Johanson et~al.(2022)Johanson, Hughes, Timbers, and Leibo]{johanson2022emergentbarteringbehaviourmultiagent}
Michael~Bradley Johanson, Edward Hughes, Finbarr Timbers, and Joel~Z Leibo.
\newblock Emergent bartering behaviour in multi-agent reinforcement learning.
\newblock \emph{arXiv preprint arXiv:2205.06760}, 2022.

\bibitem[Kalogiannis et~al.(2025)Kalogiannis, Vlatakis-Gkaragkounis, Gemp, and Piliouras]{kalogiannis2025solving}
Fivos Kalogiannis, Emmanouil-Vasileios Vlatakis-Gkaragkounis, Ian Gemp, and Georgios Piliouras.
\newblock Solving zero-sum convex markov games.
\newblock \emph{arXiv preprint arXiv:2506.16120}, 2025.

\bibitem[Kim et~al.(2023)Kim, Shin, Abbeel, and Seo]{kim2023accelerating}
Dongyoung Kim, Jinwoo Shin, Pieter Abbeel, and Younggyo Seo.
\newblock Accelerating reinforcement learning with value-conditional state entropy exploration.
\newblock In \emph{Advances in Neural Information Processing Systems}, 2023.

\bibitem[Kolchinsky and Tracey(2017)]{Kolchinsky_2017}
Artemy Kolchinsky and Brendan Tracey.
\newblock Estimating mixture entropy with pairwise distances.
\newblock \emph{Entropy}, 19\penalty0 (7):\penalty0 361, 2017.

\bibitem[Laskin et~al.(2021)Laskin, Yarats, Liu, Lee, Zhan, Lu, Cang, Pinto, and Abbeel]{laskin2021urlb}
Michael Laskin, Denis Yarats, Hao Liu, Kimin Lee, Albert Zhan, Kevin Lu, Catherine Cang, Lerrel Pinto, and Pieter Abbeel.
\newblock {URLB}: Unsupervised {R}einforcement {L}earning benchmark.
\newblock \emph{Advances in Neural Information Processing Systems (Datasets \& Benchmarks)}, 2021.

\bibitem[Lee et~al.(2019)Lee, Eysenbach, Parisotto, Xing, Levine, and Salakhutdinov]{lee2020efficientexplorationstatemarginal}
Lisa Lee, Benjamin Eysenbach, Emilio Parisotto, Eric Xing, Sergey Levine, and Ruslan Salakhutdinov.
\newblock Efficient exploration via state marginal matching.
\newblock \emph{arXiv preprint arXiv:1906.05274}, 2019.

\bibitem[Leonardos et~al.(2022)Leonardos, Overman, Panageas, and Piliouras]{leonardos2021globalconvergencemultiagentpolicy}
Stefanos Leonardos, Will Overman, Ioannis Panageas, and Georgios Piliouras.
\newblock Global convergence of multi-agent policy gradient in markov potential games.
\newblock In \emph{International Conference on Learning Representations}, 2022.

\bibitem[Li et~al.(2024)Li, Yu, Liu, and Principe]{li2024element}
Hongming Li, Shujian Yu, Bin Liu, and Jose~C Principe.
\newblock Element: Episodic and lifelong exploration via maximum entropy.
\newblock \emph{arXiv preprint arXiv:2412.03800}, 2024.

\bibitem[Littman(1994)]{Littman1994}
Michael~L. Littman.
\newblock Markov games as a framework for multi-agent reinforcement learning.
\newblock In \emph{Machine Learning Proceedings}, pages 157--163. 1994.

\bibitem[Liu and Abbeel(2021{\natexlab{a}})]{liu2021aps}
Hao Liu and Pieter Abbeel.
\newblock {APS}: Active pretraining with successor features.
\newblock In \emph{International Conference on Machine Learning}, 2021{\natexlab{a}}.

\bibitem[Liu and Abbeel(2021{\natexlab{b}})]{liu2021behavior}
Hao Liu and Pieter Abbeel.
\newblock Behavior from the void: unsupervised active pre-training.
\newblock In \emph{Advances on Neural Information Processing Systems}, 2021{\natexlab{b}}.

\bibitem[Liu et~al.(2021)Liu, Jain, Yeh, and Schwing]{liu2021cooperative}
Iou-Jen Liu, Unnat Jain, Raymond~A Yeh, and Alexander Schwing.
\newblock Cooperative exploration for multi-agent deep reinforcement learning.
\newblock In \emph{International Conference on Machine Learning}, 2021.

\bibitem[McLurkin and Yamins(2005)]{mclurkin2005dynamic}
James McLurkin and Daniel Yamins.
\newblock Dynamic task assignment in robot swarms.
\newblock In \emph{Robotics: Science and Systems}, volume~8. Cambridge, USA, 2005.

\bibitem[Metelli et~al.(2020)Metelli, Papini, Montali, and Restelli]{metelli2020pois}
Alberto~Maria Metelli, Matteo Papini, Nico Montali, and Marcello Restelli.
\newblock Importance sampling techniques for policy optimization.
\newblock \emph{Journal of Machine Learning Research}, 21\penalty0 (141):\penalty0 1--75, 2020.

\bibitem[Mirsky et~al.(2022)Mirsky, Carlucho, Rahman, Fosong, Macke, Sridharan, Stone, and Albrecht]{mirsky2022survey}
Reuth Mirsky, Ignacio Carlucho, Arrasy Rahman, Elliot Fosong, William Macke, Mohan Sridharan, Peter Stone, and Stefano~V Albrecht.
\newblock A survey of ad hoc teamwork research.
\newblock In \emph{European conference on multi-agent systems}, 2022.

\bibitem[Mutti(2023)]{mutti2023unsupervised}
Mirco Mutti.
\newblock \emph{Unsupervised reinforcement learning via state entropy maximization}.
\newblock PhD Thesis, Universit{\`a} di Bologna, 2023.

\bibitem[Mutti and Restelli(2020)]{muttirestelli2020}
Mirco Mutti and Marcello Restelli.
\newblock An intrinsically-motivated approach for learning highly exploring and fast mixing policies.
\newblock \emph{AAAI Conference on Artificial Intelligence}, 2020.

\bibitem[Mutti et~al.(2021)Mutti, Pratissoli, and Restelli]{mutti2021taskagnosticexplorationpolicygradient}
Mirco Mutti, Lorenzo Pratissoli, and Marcello Restelli.
\newblock Task-agnostic exploration via policy gradient of a non-parametric state entropy estimate.
\newblock In \emph{AAAI Conference on Artificial Intelligence}, 2021.

\bibitem[Mutti et~al.(2022{\natexlab{a}})Mutti, De~Santi, De~Bartolomeis, and Restelli]{mutti2023challengingcommonassumptionsconvex}
Mirco Mutti, Riccardo De~Santi, Piersilvio De~Bartolomeis, and Marcello Restelli.
\newblock Challenging common assumptions in convex reinforcement learning.
\newblock \emph{Advances in Neural Information Processing Systems}, 2022{\natexlab{a}}.

\bibitem[Mutti et~al.(2022{\natexlab{b}})Mutti, De~Santi, and Restelli]{mutti2022importance}
Mirco Mutti, Riccardo De~Santi, and Marcello Restelli.
\newblock The importance of non-{M}arkovianity in maximum state entropy exploration.
\newblock In \emph{International Conference on Machine Learning}, 2022{\natexlab{b}}.

\bibitem[Mutti et~al.(2022{\natexlab{c}})Mutti, Mancassola, and Restelli]{mutti2022unsupervised}
Mirco Mutti, Mattia Mancassola, and Marcello Restelli.
\newblock Unsupervised reinforcement learning in multiple environments.
\newblock In \emph{AAAI Conference on Artificial Intelligence}, 2022{\natexlab{c}}.

\bibitem[Mutti et~al.(2023)Mutti, Santi, Bartolomeis, and Restelli]{2023mutticonvexrlfinite}
Mirco Mutti, Riccardo~De Santi, Piersilvio~De Bartolomeis, and Marcello Restelli.
\newblock Convex reinforcement learning in finite trials.
\newblock \emph{Journal of Machine Learning Research}, 24\penalty0 (250):\penalty0 1--42, 2023.

\bibitem[Nedergaard and Cook(2022)]{nedergaard2022k}
Alexander Nedergaard and Matthew Cook.
\newblock k-means maximum entropy exploration.
\newblock \emph{arXiv preprint arXiv:2205.15623}, 2022.

\bibitem[Owen(2013)]{mcbook}
Art~B. Owen.
\newblock \emph{Monte Carlo theory, methods and examples}.
\newblock 2013.

\bibitem[Pajarinen et~al.(2019)Pajarinen, Thai, Akrour, Peters, and Neumann]{pajarinen2019compatible}
Joni Pajarinen, Hong~Linh Thai, Riad Akrour, Jan Peters, and Gerhard Neumann.
\newblock Compatible natural gradient policy search.
\newblock \emph{Machine Learning}, 108\penalty0 (8):\penalty0 1443--1466, 2019.

\bibitem[Paninski(2003)]{paninski2003}
Liam Paninski.
\newblock Estimation of entropy and mutual information.
\newblock \emph{Neural Computation}, 15\penalty0 (6):\penalty0 1191–1253, 2003.

\bibitem[Peng et~al.(2021)Peng, Rashid, Schroeder~de Witt, Kamienny, Torr, B{\"o}hmer, and Whiteson]{peng2021facmac}
Bei Peng, Tabish Rashid, Christian Schroeder~de Witt, Pierre-Alexandre Kamienny, Philip Torr, Wendelin B{\"o}hmer, and Shimon Whiteson.
\newblock {FACMAC}: Factored multi-agent centralised policy gradients.
\newblock \emph{Advances in Neural Information Processing Systems}, 2021.

\bibitem[Perolat et~al.(2022)Perolat, De~Vylder, Hennes, Tarassov, Strub, de~Boer, Muller, Connor, Burch, Anthony, et~al.]{perolat2022mastering}
Julien Perolat, Bart De~Vylder, Daniel Hennes, Eugene Tarassov, Florian Strub, Vincent de~Boer, Paul Muller, Jerome~T Connor, Neil Burch, Thomas Anthony, et~al.
\newblock Mastering the game of stratego with model-free multiagent {R}einforcement {L}earning.
\newblock \emph{Science}, 378\penalty0 (6623):\penalty0 990--996, 2022.

\bibitem[Peters and Schaal(2008)]{peters2008reinforcement}
Jan Peters and Stefan Schaal.
\newblock Reinforcement learning of motor skills with policy gradients.
\newblock \emph{Neural Networks}, 2008.

\bibitem[Puterman(2014)]{puterman2014markov}
Martin~L Puterman.
\newblock \emph{Markov decision processes: discrete stochastic dynamic programming}.
\newblock John Wiley \& Sons, 2014.

\bibitem[Samvelyan et~al.(2019)Samvelyan, Rashid, De~Witt, Farquhar, Nardelli, Rudner, Hung, Torr, Foerster, and Whiteson]{samvelyan2019starcraft}
Mikayel Samvelyan, Tabish Rashid, Christian~Schroeder De~Witt, Gregory Farquhar, Nantas Nardelli, Tim~GJ Rudner, Chia-Man Hung, Philip~HS Torr, Jakob Foerster, and Shimon Whiteson.
\newblock The starcraft multi-agent challenge.
\newblock \emph{arXiv preprint arXiv:1902.04043}, 2019.

\bibitem[Schulman et~al.(2015)Schulman, Levine, Abbeel, Jordan, and Moritz]{schulman2017trustregionpolicyoptimization}
John Schulman, Sergey Levine, Pieter Abbeel, Michael Jordan, and Philipp Moritz.
\newblock Trust region policy optimization.
\newblock In \emph{International conference on machine learning}, 2015.

\bibitem[Seo et~al.(2021)Seo, Chen, Shin, Lee, Abbeel, and Lee]{seo2021state}
Younggyo Seo, Lili Chen, Jinwoo Shin, Honglak Lee, Pieter Abbeel, and Kimin Lee.
\newblock State entropy maximization with random encoders for efficient exploration.
\newblock In \emph{International Conference on Machine Learning}, 2021.

\bibitem[Sutton et~al.(1999)Sutton, McAllester, Singh, and Mansour]{sutton1999policy}
Richard~S Sutton, David McAllester, Satinder Singh, and Yishay Mansour.
\newblock Policy gradient methods for {R}einforcement {L}earning with function approximation.
\newblock In \emph{Advances in Neural Information Processing Systems}, 1999.

\bibitem[Tak{\'a}cs(1966)]{takacs1966non}
L~Tak{\'a}cs.
\newblock Non-markovian processes.
\newblock In \emph{Stochastic Process: Problems and Solutions}, pages 46--62. Springer, 1966.

\bibitem[Tiapkin et~al.(2023)Tiapkin, Belomestny, Calandriello, Moulines, Munos, Naumov, Perrault, Tang, Valko, and Menard]{pmlr-v202-tiapkin23a}
Daniil Tiapkin, Denis Belomestny, Daniele Calandriello, Eric Moulines, Remi Munos, Alexey Naumov, Pierre Perrault, Yunhao Tang, Michal Valko, and Pierre Menard.
\newblock Fast rates for maximum entropy exploration.
\newblock In \emph{International Conference on Machine Learning}, 2023.

\bibitem[Toquebiau et~al.(2024)Toquebiau, Bredeche, Benamar, and Jun]{toquebiau2024joint}
Maxime Toquebiau, Nicolas Bredeche, Fa{\"\i}z Benamar, and Jae-Yun Jun.
\newblock Joint intrinsic motivation for coordinated exploration in multi-agent deep {R}einforcement {L}earning.
\newblock \emph{arXiv preprint arXiv:2402.03972}, 2024.

\bibitem[Weissman et~al.(2003)Weissman, Ordentlich, Seroussi, Verd{\'u}, and Weinberger]{Weissman2003InequalitiesFT}
Tsachy Weissman, Erik Ordentlich, Gadiel Seroussi, Sergio Verd{\'u}, and Marcelo~J. Weinberger.
\newblock Inequalities for the l1 deviation of the empirical distribution.
\newblock 2003.

\bibitem[Whitehead and Lin(1995)]{whitehead1995reinforcement}
Steven~D Whitehead and Long-Ji Lin.
\newblock Reinforcement learning of non-markov decision processes.
\newblock \emph{Artificial Intelligence}, 73\penalty0 (1-2):\penalty0 271--306, 1995.

\bibitem[Xu et~al.(2024)Xu, Zhang, and Huang]{xu2024population}
Pei Xu, Junge Zhang, and Kaiqi Huang.
\newblock Population-based diverse exploration for sparse-reward multi-agent tasks.
\newblock In \emph{International Joint Conference on Artificial Intelligence}, 2024.

\bibitem[Yang et~al.(2021)Yang, Shi, Zhao, Xie, and Yang]{yang2021ciexplore}
Huanhuan Yang, Dianxi Shi, Chenran Zhao, Guojun Xie, and Shaowu Yang.
\newblock Ciexplore: Curiosity and influence-based exploration in multi-agent cooperative scenarios with sparse rewards.
\newblock In \emph{ACM International Conference on Information \& Knowledge Management}, 2021.

\bibitem[Yang and Spaan(2023)]{yang2023cem}
Qisong Yang and Matthijs~TJ Spaan.
\newblock {CEM}: Constrained entropy maximization for task-agnostic safe exploration.
\newblock In \emph{AAAI Conference on Artificial Intelligence}, 2023.

\bibitem[Yarats et~al.(2021)Yarats, Fergus, Lazaric, and Pinto]{yarats2021reinforcement}
Denis Yarats, Rob Fergus, Alessandro Lazaric, and Lerrel Pinto.
\newblock Reinforcement learning with prototypical representations.
\newblock In \emph{International Conference on Machine Learning}, 2021.

\bibitem[Yarats et~al.(2022)Yarats, Brandfonbrener, Liu, Laskin, Abbeel, Lazaric, and Pinto]{yarats2022don}
Denis Yarats, David Brandfonbrener, Hao Liu, Michael Laskin, Pieter Abbeel, Alessandro Lazaric, and Lerrel Pinto.
\newblock Don't change the algorithm, change the data: Exploratory data for offline {R}einforcement {L}earning.
\newblock \emph{arXiv preprint arXiv:2201.13425}, 2022.

\bibitem[Yu et~al.(2022)Yu, Velu, Vinitsky, Gao, Wang, Bayen, and Wu]{yu2022surprisingeffectivenessppocooperative}
Chao Yu, Akash Velu, Eugene Vinitsky, Jiaxuan Gao, Yu~Wang, Alexandre Bayen, and Yi~Wu.
\newblock The surprising effectiveness of ppo in cooperative multi-agent games.
\newblock \emph{Advances in Neural Information Processing Systems}, 2022.

\bibitem[Yuan et~al.(2022)Yuan, Pun, and Wang]{yuan2022renyi}
Mingqi Yuan, Man-On Pun, and Dong Wang.
\newblock R{\'e}nyi state entropy maximization for exploration acceleration in reinforcement learning.
\newblock \emph{IEEE Transactions on Artificial Intelligence}, 4\penalty0 (5):\penalty0 1154--1164, 2022.

\bibitem[Zahavy et~al.(2021)Zahavy, O'Donoghue, Desjardins, and Singh]{zahavy2023rewardconvexmdps}
Tom Zahavy, Brendan O'Donoghue, Guillaume Desjardins, and Satinder Singh.
\newblock Reward is enough for convex {MDP}s.
\newblock \emph{Advances in Neural Information Processing Systems}, 2021.

\bibitem[Zamboni et~al.(2024{\natexlab{a}})Zamboni, Cirino, Restelli, and Mutti]{zamboni2024explore}
Riccardo Zamboni, Duilio Cirino, Marcello Restelli, and Mirco Mutti.
\newblock How to explore with belief: State entropy maximization in {POMDP}s.
\newblock In \emph{International Conference on Machine Learning}, 2024{\natexlab{a}}.

\bibitem[Zamboni et~al.(2024{\natexlab{b}})Zamboni, Cirino, Restelli, and Mutti]{zamboni2024limits}
Riccardo Zamboni, Duilio Cirino, Marcello Restelli, and Mirco Mutti.
\newblock The limits of pure exploration in {POMDP}s: When the observation entropy is enough.
\newblock \emph{{RLJ}}, 2:\penalty0 676--692, 2024{\natexlab{b}}.

\bibitem[Zhang et~al.(2021{\natexlab{a}})Zhang, Cai, Huang, and Li]{zhang2020exploration}
Chuheng Zhang, Yuanying Cai, Longbo Huang, and Jian Li.
\newblock Exploration by maximizing {R}{\'e}nyi entropy for reward-free rl framework.
\newblock In \emph{AAAI Conference on Artificial Intelligence}, 2021{\natexlab{a}}.

\bibitem[Zhang et~al.(2020)Zhang, Koppel, Bedi, Szepesvari, and Wang]{zhang2020variationalpolicygradientmethod}
Junyu Zhang, Alec Koppel, Amrit~Singh Bedi, Csaba Szepesvari, and Mengdi Wang.
\newblock Variational policy gradient method for {R}einforcement {L}earning with general utilities.
\newblock \emph{Advances in Neural Information Processing Systems}, 2020.

\bibitem[Zhang et~al.(2023)Zhang, Cao, Yuan, Yu, and Zhan]{zhang2023self}
Shaowei Zhang, Jiahan Cao, Lei Yuan, Yang Yu, and De-Chuan Zhan.
\newblock Self-motivated multi-agent exploration.
\newblock \emph{arXiv preprint arXiv:2301.02083}, 2023.

\bibitem[Zhang et~al.(2021{\natexlab{b}})Zhang, Rashidinejad, Jiao, Tian, Gonzalez, and Russell]{zhang2021made}
Tianjun Zhang, Paria Rashidinejad, Jiantao Jiao, Yuandong Tian, Joseph~E Gonzalez, and Stuart Russell.
\newblock {MADE}: Exploration via maximizing deviation from explored regions.
\newblock \emph{Advances in Neural Information Processing Systems}, 2021{\natexlab{b}}.

\bibitem[Zisselman et~al.(2023)Zisselman, Lavie, Soudry, and Tamar]{zisselman2023explore}
Ev~Zisselman, Itai Lavie, Daniel Soudry, and Aviv Tamar.
\newblock Explore to generalize in zero-shot {RL}.
\newblock \emph{Advances in Neural Information Processing Systems}, 2023.

\end{thebibliography}

%%%%%%%%%%%%%%%%%%%%%%%%%%%%%%%%%%%%%%%%%%%%%%%%%%%%%%%%%%%%

\clearpage
\appendix

%%%%%%%%%%%%%%%%%%%%%%%%%%%%%%%%%%%%%%%%%%%%%%%%%%%%%%%%%%%%

\section{Further Insights on the Unsupervised Objectives.}
\label{apx:comparison}

\paragraph*{Motivating Example.}~Let us envision a team of agents in a "search and rescue" task. In a specific building (environment) the target may be found in different place (different rewards) and the unsupervised pre-training phase aims to prepare for all of them. Mixture entropy is a good surrogate objective in this case, as the agents will split up into different portions of the buildings to traverse in order to find the target quickly.

\textbf{Clarification on the Ideal Objective: Joint and Mixture Objectives Comparison}

As in single-agent settings, the goal of unsupervised (MA)RL via state entropy pre-training is to learn exploration for any possible task while interacting with a reward-free environment. If the tasks is assumed to be represented through state-based reward functions, the latter translates into state coverage: The state entropy is a proxy for state coverage~\citep{hazan2019provably,mutti2021taskagnosticexplorationpolicygradient, liu2021behavior}.

As a consequence, the most natural state entropy formulation in Markov games is the \textbf{joint state entropy}. However, it comes with some important drawbacks:
\begin{itemize}
    \item \textbf{Estimation.} The support of the entropy grows exponentially with the number of agents $|S|^{|\mathcal{N}|}$, so does the complexity of the entropy estimation problem~\citep{beirlant1997nonparametric};
    \item \textbf{Concentration.} The empirical entropy concentrates as $\sqrt{K^{-1}}$ for $K$ trajectories (see Thm.~\ref{thr:objectivemismatch});
    \item \textbf{Redundancy.} When Asm. 3.1 holds and the state space $|\mathcal{S}|$ is the same for every agent, the joint entropy may inflate state coverage as $(s,s')$ and $(s',s)$ are different joint states. 
    %Indeed, this \textbf{redundancy} calls for the fast optimization of joint entropy in Figure 2. [MIRCO: I would stress this point later, I'm worried it is not clear here]
\end{itemize}

In other words, the problem of optimizing the joint entropy suffers from the \emph{curse of multiagency}, which is particularly relevant in practice (while their difference might not be so relevant in ideal settings, see Fact~\ref{fact:sufficiencypga} and Thm.~\ref{theorem:iteration complexity-gen}). 

Another potential formulation is the \textbf{mixture state entropy}, which has the following properties:
\begin{itemize}
    \item \textbf{Estimation.} The support of the entropy and therefore the estimation complexity do not grow with the number of agents;
    \item \textbf{Concentration.} The empirical entropy concentrates as $\sqrt{(K|\mathcal{N}|)^{-1}}$ for $K$ trajectories (see Thm.~\ref{thr:objectivemismatch});
    \item \textbf{Redundancy.} For the mixture entropy objective, the joint states $(s,s')$ and $(s',s)$ are contributing in the same way; therefore, there is no difference in visiting one or the other. 
\end{itemize}
The latter can be a limitation when we aim to explore all the possible \emph{joint states}, e.g., when the reward functions of the agents will be different in the eventual tasks.
Yet, at least the mixture entropy is also a lower bound to the joint entropy objective with a $\log (|\mathcal{N}|)$ approximation (see Lem.~\ref{lem:entropymismatch}) and thus a valid proxy also in the latter case, given the favorable estimation and concentration properties.

\section{Proofs of the Main Theoretical Results}
\label{apx:proof}

In this Section, we report the full proofing steps of the Theorems and Lemmas in the main paper.

\entropymismatch*

\begin{proof}
    The bounds follow directly from simple yet fundamental relationships between entropies of joint, marginal and mixture distributions which can be found in~\citet{ paninski2003, Kolchinsky_2017}, in particular:
    \begin{align*}
        \frac{1}{|\Ns|}H(d^\pi) \leq \frac{1}{|\Ns|}\sum_{i \in [\Ns]}H(d_i^\pi)  \overset{\text{(a)}}{\leq} H(\tilde d^\pi)  &\overset{\text{(b)}}{\leq} \frac{1}{|\Ns|}\sum_{i \in [\Ns]}H(d_i^\pi)+ \log(|\Ns|) \\& \overset{\text{(c)}}{\leq} \sup_{i \in [\Ns]}H(d_i^\pi)+ \log(|\Ns|) \leq H(d^\pi) + \log(|\Ns|)
    \end{align*}
    where step (a) and (b) use the fact that $\tilde d^\pi(s) := \frac{1}{|\Ns|}\sum_{i \in [\Ns]} d^\pi_i(s)$ is a uniform mixture over the agents, whose distribution over the weights has entropy $\log(|\Ns|)$, so as we can apply the bounds from~\citet{Kolchinsky_2017}. Step (c) uses the fact that $H(d^\pi) = \sum_{i \in [\Ns]}H(d^\pi_i|d^\pi_{<i})$, then taking the supremum as first $i$ it follows that $ \sup_{i \in [\Ns]}H(d_i^\pi) = H(d^\pi) - \sum_{j \in [\Ns]> i} H(d^\pi_j|d^\pi_{<j}, d^\pi_{i}) \leq H(d^\pi)$ due to non-negativity of entropy.
\end{proof}

\objectivemismatch*

\begin{proof}
    For the general proof structure, we adapt the steps of~\citet{mutti2023challengingcommonassumptionsconvex} for cMDPs to the different objectives possible in cMGs.
    Let us start by considering joint objectives, then:
    \begin{align*}
        \big| \zeta_K(\pi) - \zeta_\infty (\pi) \big|
        &= \Big| \E_{d_K\sim p^{\pi}_K} \left[ \mathcal F (d_K) \right] - \mathcal F (d^{\pi}) \Big| \leq \E_{d_K\sim p^{\pi}_K} \left[ \left| \mathcal F (d_K) - \mathcal F (d^{\pi}) \right| \right] \\
        &\overset{\text{(a)}}{\leq} \E_{d_K\sim p^{\pi}_K} \left[ L \left\| d_K- d^{\pi} \right\|_1 \right] \leq L \E_{d_K\sim p^{\pi}_K} \left[ \left\| d_K- d^{\pi} \right\|_1 \right] \\
        &\overset{\text{(b)}}{\leq}   L \E_{d_K\sim p^{\pi}_K} \left[ \max_{t \in [T]}  \left\| d_{K, t} - d^{\pi}_t \right\|_1 \right],
    \end{align*}
    where in step (a) we apply the Lipschitz assumption on $\mathcal F$ to write and in step (b) we apply a maximization over the episode's step by noting that $d_{K} = \frac{1}{T} \sum_{t \in [T]} d_{K, t}$ and $d^\pi = \frac{1}{T} \sum_{t \in [T]} d^\pi_t$.
    We then apply bounds in high probability
    \begin{align*}
        Pr \Big(  \max_{t \in [T]}  \left\| d_{K, t} - d^{\pi}_t \right\|_1 \geq \epsilon \Big)
        &\leq Pr \Big( \bigcup_{ t} \left\| d_{K, t} - d^{\pi}_t \right\|_1 \geq \epsilon \Big)  \\
        &\overset{\text{(c)}}{\leq}  \sum_{ t} Pr \Big( \left\| d_{K, t} - d^{\pi}_t \right\|_1 \geq \epsilon \Big) \\
        &\leq  T \ Pr \Big( \left\| d_{K, t} - d^{\pi}_t \right\|_1 \geq \epsilon \Big), 
    \end{align*}
    with $\epsilon > 0$ and in step (c) we applied a union bound. We then consider standard concentration inequalities for empirical distributions~\citep{Weissman2003InequalitiesFT} so to obtain the final bound
    \begin{equation}
        Pr \Bigg( \left\| d_{K, t} - d^{\pi}_t \right\|_1 \geq \sqrt{\frac{2|\Ss| \log(2 / \delta')}{K}} \ \Bigg) \leq \delta'. \label{eq:empirical_dist_concentration}
    \end{equation}
    By setting $\delta' = \delta / T$, and then plugging the empirical concentration inequality, we have that with probability at least $1 - \delta$
    \begin{equation*}
        \big| \zeta_K (\pi) - \zeta_\infty (\pi) \big| \leq  L T \sqrt{\frac{2|\Ss| \log(2 T / \delta )}{K}},
    \end{equation*}
    which concludes the proof for joint objectives.

    The proof for disjoint objectives follows the same rational by bounding each per-agent term separately and after noticing that due to Assumption~\ref{ass:mixture}, the resulting bounds get simplified in the overall averaging. As for mixture objectives, the only core difference is after step (b), where $\tilde{d}_K$ takes the place of $d_K$ and $\tilde{d^\pi}$ of $d^\pi$. The remaining steps follow the same logic, out of noticing that the empirical distribution with respect to $\tilde{d^\pi}$ is taken with respect $|\Ns|K$ samples in total. Both the two bounds then take into account that the support of the empirical distributions have size $|\tilde \Ss|$ and not $|\Ss|$.
\end{proof}

\subsection{Policy Gradient in cMGs with Infinite-Trials Formulations.}
\label{apx:theory}
In this Section, we analyze policy search for the infinite-trials joint problem $\zeta_\infty$ of Eq.~\eqref{eq:mse}, via projected gradient ascent over parametrized policies, providing in Th.~\ref{theorem:iteration complexity-gen} the formal counterpart of Fact~\ref{fact:sufficiencypga} in the Main paper. As a side note, all of the following results hold for the (infinite-trials) mixture objective $\tilde \zeta_\infty$ of Eq.~\eqref{eq:mse_mixture}. We will consider the class of parametrized policies with parameters $\theta_i \in \Theta_i \subset \mathbb R^d$, with the joint policy then defined as $\pi_\theta, \theta \in \Theta = \times_{i \in [\Ns]} \Theta_i$. Additionally, we will focus on the computational complexity only, by assuming access to the exact gradient. The study of statistical complexity surpasses the scope of the current work. We define the \textbf{(independent) Policy Gradient Ascent} (PGA) update as:
\begin{eqnarray}
\label{defn:grad-proj}
\theta^{k+1}_i  =  \argmax_{\theta_i\in\Theta_i} \zeta_\infty(\pi_{\theta^k}) \!+\! \left\langle \nabla_{\theta_i} \zeta_\infty(\pi_{\theta^k}),\theta_i\!-\!\theta^k_i\right\rangle \!-\! \frac{1}{2\eta}\|\theta_i\!-\!\theta_i^k\|^2 =  \Pi_{\Theta_i}\big\{\theta^k_i + \eta\nabla_{\theta_i} \zeta_\infty(\pi_{\theta^k})\big\}
\end{eqnarray}
where $\Pi_{\Theta_i}\{\cdot\}$ denotes Euclidean projection onto $\Theta_i$, and equivalence holds by the convexity of $\Theta_i$. The classes of policies that allow for this condition to be true will be discussed shortly.

In general the overall proof is built of three main steps, shared with the theory of Potential Markov Games~\citep{leonardos2021globalconvergencemultiagentpolicy}: (i) prove the existence of well behaved stationary points; (ii) prove that performing independent policy gradient is equivalent to perform joint policy gradient; (iii) prove that the (joint) PGA update converges to the stationary points via single-agent like analysis.
In order to derive the subsequent convergence proof, we will make the following assumptions:

\begin{assumption}
	\label{assumption:gen-para} Define the quantity $\lambda(\theta) := d^{\pi_\theta}$, then:\\
	\textbf{(i).} $\lambda(\cdot)$ forms a bijection between $\Theta$ and $\lambda(\Theta)$, where $\Theta$ and $\lambda(\Theta)$ are closed and convex. \\
	\textbf{(ii).} The Jacobian matrix $\nabla_{\theta}\lambda(\theta)$ is Lipschitz continuous in $\Theta$.  \\
	\textbf{(iii).} Denote $g(\cdot) := \lambda^{-1}(\cdot)$ as the inverse mapping of $\lambda(\cdot)$. Then there exists $ \ell_{\theta}>0$ s.t. $\|g(\lambda)-g(\lambda')|\leq \ell_\theta\|\lambda-\lambda'\|$ for some norm $\|\cdot\|$ and for all $\lambda,\lambda'\in\lambda(\Theta)$. 
\end{assumption}

\begin{assumption}
	\label{assumption:ncvx-Lip}
	There exists $ L>0$ such that the gradient $\nabla_{\theta }\zeta_\infty(\pi_{\theta})$ is $L$-Lipschitz. 
\end{assumption}

\begin{assumption}
	\label{assumption:gradient}
	The agents have access to a gradient oracle $\mathcal O(\cdot)$ that returns $\nabla_{\theta_i}\zeta_\infty(\pi_\theta)$ for any deployed joint policy $\pi_\theta$.
\end{assumption}

\paragraph*{On the Validity of Assumption~\ref{assumption:gen-para}.}~This set of assumptions enforces the objective $\zeta_\infty(\pi_\theta)$ to be well-behaved with respect to $\theta$ even if non-convex in general, and will allow for a rather strong result. Yet, the assumptions are known to be true for directly parametrized policies over the whole support of the distribution $d^\pi$~\citep{zhang2020variationalpolicygradientmethod}, and as a result they implicitly require agents to employ policies conditioned over the full state-space $\Ss$. Fortunately enough, they also guarantee $\Theta$ to be convex.

\begin{lemma}[\textbf{(i)} Global optimality of stationary policies~\citep{zhang2020variationalpolicygradientmethod}]
	\label{lemma:global-opt}
	Suppose Assumption \ref{assumption:gen-para} holds, and $\mathcal F$ is a concave, and continuous function defined in an open neighborhood containing $\lambda(\Theta)$. 
	Let $\theta^*$ be a first-order stationary point of problem \eqref{eq:mse}, i.e.,\vspace{-2mm}
	\begin{equation}
	\label{defn:1st-order-condition}
	\exists u^*\in\hat{\partial}(\mathcal F\circ\lambda)(\theta^*),\quad\text{s.t.}\quad \langle u^*, \theta-\theta^*\rangle\leq0 \qquad\mbox{for}\qquad\forall \theta\in\Theta.
	\end{equation}
	Then $\theta^*$ is a globally optimal solution of problem \eqref{eq:mse}.
\end{lemma}
This result characterizes the optimality of stationary points for Eq.~\eqref{eq:mse}. Furthermore, we know from~\citet{leonardos2021globalconvergencemultiagentpolicy} that stationary points of the objective are Nash Equilibria.

\begin{lemma}[\textbf{(ii)} Projection Operator~\citep{leonardos2021globalconvergencemultiagentpolicy}]\label{claim:projection} 
    Let $\theta := (\theta_1,...,\theta_\Ns)$ be the parameter profile for all agents and use the update of Eq.~\eqref{defn:grad-proj} over a non-disjoint infinite-trials objective. Then, it holds that
    \begin{equation*}
        \Pi_{\Theta}\big\{\theta^k + \eta\nabla_\theta \zeta_\infty(\pi_{\theta^k})\big\} = \Big( \Pi_{\Theta_i}\big\{\theta^k_i + \eta\nabla_{\theta_i} \zeta_\infty(\pi_{\theta^k})\big\}\Big)_{i \in [\Ns]}
    \end{equation*}
\end{lemma}

This result will only be used for the sake of the convergence analysis, since it allows to analyze independent updates as joint updates over a single objective. The following Theorem is the formal counterpart of Fact~\ref{fact:sufficiencypga} and it is a direct adaptation to the multi-agent case of the single-agent proof by~\citet{zhang2020variationalpolicygradientmethod}, by exploiting the previous result.

\begin{theorem}[\textbf{(iii)} Convergence rate of independent PGA to stationary points (Formal Fact~\ref{fact:sufficiencypga})]
	\label{theorem:iteration complexity-gen}
	Let Assumptions \ref{assumption:gen-para} and \ref{assumption:ncvx-Lip} hold. Denote $D_\lambda \!:=\! \max_{\lambda,\lambda'\in\lambda(\Theta)} \|\lambda-\lambda'\|$ as defined in Assumption~\ref{assumption:gen-para}(iii). Then the independent policy gradient update \eqref{defn:grad-proj} with $\eta = 1/L$ satisfies for all $k$ with respect to a stationary (joint) policy $\pi_{\theta^*}$ the following\vspace{-2mm}
	\begin{equation*}
	%\label{thm:cvg-gen-para-1}
	\zeta_\infty(\pi_{\theta^*}) \!-\! \zeta_\infty(\pi_{\theta^k})\leq \frac{4L\ell_{\theta}^2D_\lambda^2}{k+1}.
	\end{equation*}
\end{theorem}

\begin{proof}
    First, the Lipschitz continuity in Assumption \ref{assumption:ncvx-Lip} indicates that 
        $$\left|\zeta_\infty(\lambda(\theta)) - \zeta_\infty(\lambda(\theta^k)) - \langle \nabla_\theta\zeta_\infty(\lambda(\theta^k)),\theta-\theta^k\rangle\right|\leq \frac{L}{2}\|\theta-\theta^k\|^2.$$
        Consequently, for any $\theta\in\Theta$ we have the ascent property:
        \begin{equation}\label{eq:taylor_proof}
        \zeta_\infty(\lambda(\theta)) \geq \zeta_\infty(\lambda(\theta^k)) + \langle \nabla_\theta\zeta_\infty(\lambda(\theta^k)),\theta-\theta^k\rangle - \frac{L}{2}\|\theta-\theta^k\|^2 \geq \zeta_\infty(\lambda(\theta)) - L\|\theta-\theta^k\|^2.
        \end{equation}
        The optimality condition in the policy update rule \eqref{defn:grad-proj} coupled with the result of Lemma~\ref{claim:projection} allows us to follow the same rational as~\citet{zhang2020variationalpolicygradientmethod}. We will report their proof structure after this step for completeness.
    \begin{align}
        \label{thm:ItrCmp-1}
        \MoveEqLeft
        \zeta_\infty(\lambda(\theta^{k+1}))  \geq  \zeta_\infty(\lambda(\theta^k)) + \langle \nabla_\theta\zeta_\infty(\lambda(\theta^k)),\theta^{k+1}-\theta^k\rangle - \frac{L}{2}\|\theta^{k+1}-\theta^k\|^2 \nonumber \\
        & =  \max_{\theta\in\Theta} \zeta_\infty(\lambda(\theta^k)) + \langle \nabla_\theta\zeta_\infty(\lambda(\theta^k)),\theta-\theta^k\rangle - \frac{L}{2}\|\theta-\theta^k\|^2\nonumber\\
        & \overset{\text{(a)}}{\geq}  \max_{\theta\in\Theta} \zeta_\infty(\lambda(\theta)) - L\|\theta-\theta^k\|^2\nonumber\\
        & \overset{\text{(b)}}{\geq}   \max_{\alpha\in[0,1]}\left\{\zeta_\infty(\lambda(\theta_{\alpha})) - L\|\theta_{\alpha}-\theta^k\|^2: \theta_{\alpha} = g(\alpha\lambda(\theta^*) + (1-\alpha)\lambda(\theta^k)) \right\}.
    \end{align}
        where step (a) follows from \eqref{eq:taylor_proof} and step (b) uses the convexity of $\lambda(\Theta)$. Then, by the concavity of $\zeta_\infty$ and the fact that the composition $\lambda\circ g = id$ due to Assumption~\ref{assumption:gen-para}(i), we have that:
        $$\zeta_\infty(\lambda(\theta_{\alpha})) = \zeta_\infty(\alpha\lambda(\theta^*) + (1-\alpha)\lambda(\theta^k))\geq\alpha\zeta_\infty(\lambda(\theta^*)) + (1-\alpha)\zeta_\infty(\lambda(\theta^k)).$$
        Moreover, due to Assumption~\ref{assumption:gen-para}(iii) we have that:
        \begin{eqnarray}
        \label{eqn:important-gen}
            \|\theta_{\alpha} - \theta^k\|^2 & = & \|g(\alpha\lambda(\theta^*) + (1-\alpha)\lambda(\theta^k))- g(\lambda(\theta^k))\|^2\\
            & \leq & \alpha^2\ell_{\theta}^2\|\lambda(\theta^*) - \lambda(\theta^k)\|^2\nonumber\\
            & \leq & \alpha^2\ell_{\theta}^2D_\lambda^2.\nonumber
        \end{eqnarray}
        From which we get 
        \begin{align}
        \MoveEqLeft 
        \zeta_\infty(\lambda(\theta^*)) - \zeta_\infty(\lambda(\theta^{k+1})) \nonumber \\
        & \leq  \min_{\alpha\in[0,1]}\left\{\zeta_\infty(\lambda(\theta^*))-\zeta_\infty(\lambda(\theta_{\alpha})) + L\|\theta_{\alpha}-\theta^k\|^2: \theta_{\alpha} = g(\alpha\lambda(\theta^*) + (1-\alpha)\lambda(\theta^k)) \right\}\nonumber\\
        & \leq  \min_{\alpha\in[0,1]}(1-\alpha)\big(\zeta_\infty(\lambda(\theta^*))-\zeta_\infty(\lambda(\theta^k))\big) + \alpha^2L\ell_{\theta}^2D_\lambda^2 \,.
        \label{thm:ItrCmp-2-gen}
        \end{align}
        We define $\Lambda(\pi_\theta) := \lambda(\theta)$, then $\alpha_k = \frac{\zeta_\infty(\Lambda(\pi^*)) - \zeta_\infty(\Lambda(\pi^k))}{2L\ell_{\theta}^2D_\lambda^2}\geq0$, which is the minimizer of the RHS of  \eqref{thm:ItrCmp-2-gen} as long as it satisfies $\alpha_k\leq 1$. Now, we claim the following: If $\alpha_k\ge 1$ then $\alpha_{k+1}<1$. Further, if $\alpha_k<1$ then $\alpha_{k+1}\le \alpha_k$. The two claims together mean that $(\alpha_k)_k$ is decreasing and all $\alpha_k$ are in $[0,1)$ except perhaps $\alpha_0$.
    
        To prove the first of the two claims, assume $\alpha_k\ge 1$.
        This implies that $\zeta_\infty(\Lambda(\pi^*)) - \zeta_\infty(\Lambda(\pi^k))\geq 2L\ell_{\theta}^2D_\lambda^2$. Hence, choosing $\alpha=1$ in \eqref{thm:ItrCmp-2-gen}, we get
        \[\zeta_\infty(\lambda(\theta^*)) - \zeta_\infty(\lambda(\theta^{k}))\leq L\ell_{\theta}^2D_\lambda^2\]
        which implies that $\alpha_{k+1}\le 1/2<1$. To prove the second claim, we plug  $\alpha_k$ into \eqref{thm:ItrCmp-2-gen} to get
        \[
        \zeta_\infty(\lambda(\theta^*)) - \zeta_\infty(\lambda(\theta^{k+1})) \leq  \left(1-\frac{\zeta_\infty(\lambda(\theta^*)) - \zeta_\infty(\lambda(\theta^{k}))}{4L\ell_{\theta}^2D_\lambda^2}\right)(\zeta_\infty(\lambda(\theta^*)) - \zeta_\infty(\lambda(\theta^{k}))),
        \]
        which shows that $\alpha_{k+1}\le \alpha_k$ as required.
        
        Now, by our preceding discussion, for $k=1,2,\dots$ the previous recursion holds.
        Using the definition of $\alpha_k$, we rewrite this in the equivalent form  
        \[
        \frac{\alpha_{k+1}}{2}\leq \left(1-\frac{\alpha_{k}}{2}\right)\cdot\frac{\alpha_{k}}{2}.
        \] 
    By rearranging the preceding expressions and algebraic manipulations, we obtain
        $$\frac{2}{\alpha_{k+1}} \geq \frac{1}{\left(1-\frac{\alpha_{k}}{2}\right)\cdot\frac{\alpha_{k}}{2}} = \frac{2}{\alpha_{k}} + \frac{1}{1-\frac{\alpha_{k}}{2}}\geq\frac{2}{\alpha_k} + 1.$$
        For simplicity assume that $\alpha_0<1$ also holds. Then,
        $\frac{2}{\alpha_{k}}\geq \frac{2}{\alpha_0} + k$, and consequenlty
        $$\zeta_\infty(\lambda(\theta^*)) - \zeta_\infty(\lambda(\theta^{k}))\leq \frac{\zeta_\infty(\lambda(\theta^*)) - \zeta_\infty(\lambda(\theta^0))}{1+ \frac{\zeta_\infty(\lambda(\theta^*)) - \zeta_\infty(\lambda(\theta^0))}{4L\ell_{\theta}^2D_\lambda^2}\cdot k} \leq \frac{4L\ell_{\theta}^2D_\lambda^2}{k}.$$
        A similar analysis holds when $\alpha_0>1$. Combining these two gives that 
        $\zeta_\infty(\lambda(\pi^*)) - \zeta_\infty(\lambda(\pi^{k}))\leq \frac{4L\ell_{\theta}^2D_\lambda^2}{k+1}$ no matter the value of $\alpha_0$, which proves the result. 
    \end{proof}

\subsection{The Use of Markovian and Non-Markovian Policies in cMGs with Finite-Trials Formulations.}
\label{apx:policies}
The following result describes how in cMGs, as for cMDPs, Non-Markovian policies are the right policy class to employ to guarantee well-behaved results.

\begin{restatable}[Sufficiency of Disjoint Non-Markvoian Policies]{lem}{sufficiency}
    \label{lem:sufficiency} 
    For every cMG $\mathcal M$ there exist a joint policy $\pi^\star = (\pi^{\star, i})_{i \in \Ns}$, with $\pi^{\star, i}\in\Delta_{\Ss^T}^{\Acal^i}$ being a deterministic Non-Markovian policy, that is a Nash Equilibrium for non-Disjoint single-trial objectives, for $K=1$.
\end{restatable}

\begin{proof}
    The proof builds over a straight reduction. We build from the original MG $\MDP$ a temporally extended Markov Game $\tilde \MDP= ( \Ns, \tilde \Ss,   \Acal, \Pb, r, \mu, T)$. A state $\tilde s$ is defined for each history that can be induced, i.e., $\tilde s \in \tilde \Ss \iff \sbf \in \Ss^T $. We keep the other objects equivalent, where for the extended transition model we solely consider the last state in the history to define the conditional probability to the next history. We introduce a common reward function across all the agents $ r: \tilde \Ss \rightarrow \mathbb R$ such that $r(\tilde s) = H(d(\tilde s))$ for joint objectives and $r(\tilde s) = (1/N)\sum_{i \in [\Ns]} H(d_i(\tilde s_i))$ for mixture objectives, for all the histories of length T and $0$ otherwise. We now know that according to~\citet[Theorem 3.1,][]{leonardos2021globalconvergencemultiagentpolicy} there exists a deterministic Markovian policy $\tilde \pi^\star = (\tilde \pi^i)_{i \in \Ns}, \tilde \pi^i \in \Delta^{\Acal_i}_{\tilde \Ss}$ that is a Nash Equilibrium for $\tilde \MDP$. Since $\tilde s$ corresponds to the set of histories of the original game, $\tilde \pi^\star$ maps to a non-Markovian policy in it. Finally, it is straightforward to notice that the NE of  $\tilde \pi^\star$ for $\tilde \MDP$ implies the NE of $\tilde \pi^\star$ for the original cMG $\MDP$.
\end{proof}

The previous result implicitly asks for policies conditioned over the joint state space, as happened for infinite-trials objectives as well. Interestingly, finite-trials objectives allow for a further characterization of how an optimal Markovian policy would behave when conditioned on the per-agent states only:

\begin{restatable}[Behavior of Optimal Markovian Decentralized Policies]{lemma}{behaviormarkovian}
	Let $\pi_{\text{NM}} = (\pi^i_{\text{NM}}\in\Delta_{\Ss^T}^{\Acal^i})_{i \in [\Ns]}$ an optimal deterministic non-Markovian centralized policy and $\bar \pi_{\text{M}} = (\bar \pi^i_{\text{M}}\in\Delta_{\Ss}^{\Acal^i})_{i \in [\Ns]}$ the optimal Markovian centralized policy, namely $\bar \pi_{\text{M}}  = \argmax_{\pi = (\pi^i\in\Delta_{\Ss}^{\Acal^i})_{i \in [\Ns]} }\zeta_1(\pi)$. For a fixed sequence $\sbf_t \in \Ss^t$ ending in state $s = (s_i,s_{-i})$, the variance of the event of the optimal Markovian decentralized policy $ \pi_{\text{M}} = ( \pi^i_{\text{M}}\in\Delta_{\Ss_i}^{\Acal^i})_{i \in [\Ns]}$ taking $a^* = \pi_{\text{NM}} (\cdot|\sbf_t) = \bar \pi_{\text{M}}(\cdot|s,t)$ in $s_i$ at step $t$ is given by
	\begin{align*}
		\Var \big[ &\mathcal{B} ( \pi_{\text{M}} (a^*|s_i, t) ) \big] = \Var_{\sbf  \oplus s \sim p^{\pi_{\text{NM}}}_t} \big[ \E \big[ \mathcal{B} ( \pi_{\text{NM}} (a^* | \sbf  \oplus s) ) \big] \big]\\ &+ \Var_{\sbf \oplus(\cdot, s_{-i}) \sim p^{\bar \pi_{\text{M}}}_t} \big[ \E \big[ \mathcal{B} ( \bar \pi_{\text{M}} (a^* |  s_i, s_{-i}, t) ) \big] \big].
	\end{align*}
	where $\sbf \oplus s \in \Ss^t$ is any sequence of length $t$ such that the final state is $s$, \ie $\sbf \oplus s := (\sbf_{t - 1} \in \Ss^{t - 1}) \oplus s$, and $\mathcal{B} (x)$ is a Bernoulli with parameter $x$.
    \label{thr:behaviormarkovian}
\end{restatable}

Unsurprisingly, this Lemma shows that whenever the optimal Non-Markovian strategy for requires to adapt its decision in a joint state $s$ according to the history that led to it, an optimal Markovian policy for the same objective  must necessarily be a stochastic policy, additionally, whenever the optimal Markovian policy conditioned over per-agent states only will need to be stochastic whenever the optimal Markovian strategy conditioned on the full states randomizes its decision based on the joint state $s$.

\begin{proof}
	Let us consider the random variable $A_i \sim \mathcal{P}_i$ denoting the event ``the agent $i$ takes action $a^*_i \in \Acal_i$''. Through the law of total variance~\cite{bertsekas2002introduction}, we can write the variance of $A$ given $s \in \Ss$ and $t \geq 0$ as
	\begin{align}
		\Var \big[ A | s, t \big]
		&= \E \big[ A^2 | s, t \big] - \E \big[ A | s, t \big]^2 \nonumber
		= \E_{\sbf} \Big[ \E \big[ A^2 | s, t,  \sbf  \big] \Big] - \E_{\sbf} \Big[ \E \big[ A | s, t, \sbf \big] \Big]^2 \nonumber \\
		&= \E_{\sbf} \Big[ \Var \big[ A | s, t,  \sbf  \big] + \E \big[ A | s, t,  \sbf  \big]^2 \Big]
		- \E_{\sbf} \Big[ \E_{\pi} \big[ A | s, t,  \sbf  \big] \Big]^2 \nonumber \\
		&= \E_{\sbf} \Big[ \Var \big[ A | s, t,  \sbf  \big] \Big] +\E_{\sbf} \Big[  \E \big[ A | s, t,  \sbf  \big]^2 \Big] - \E_{\sbf} \Big[ \E \big[ A | s, t,  \sbf  \big] \Big]^2 \nonumber \\
		&= \E_{\sbf} \Big[ \Var \big[ A | s, t,  \sbf  \big] \Big] + \Var_{ \sbf } \Big[ \E\big[ A | s, t,  \sbf  \big] \Big]. \label{eq:lotv1}
	\end{align}
	Now let the conditioning event $ \sbf $ be distributed as $ \sbf  \sim p^{\pi_{\text{NM}}}_{t - 1}$, so that the condition $s, t,  \sbf $ becomes $ \sbf  \oplus s$ where $\sbf  \oplus s = (s_{0}, a_{0}, s_{1}, \ldots, s_{t} = s) \in \Ss^t$, and let the variable $A$ be distributed according to $\mathcal{P}$ that maximizes the objective given the conditioning. Hence, we have that the variable $A$ on the left hand side of~\eqref{eq:lotv1} is distributed as a Bernoulli $\mathcal{B} (\bar \pi_{\text{M}} (a^* | s, t))$, and the variable $A$ on the right hand side of~\eqref{eq:lotv2} is distributed as a Bernoulli $\mathcal{B} (\pi_{\text{NM}} (a^* | \sbf  \oplus s))$. Thus, we obtain
	\begin{equation}
		\Var \big[ \mathcal{B} ( \bar \pi_{\text{M}} (a^*|s, t) ) \big] = \E_{\sbf  \oplus s \sim p^{\pi_{\text{NM}}}_t} \big[ \Var \big[ \mathcal{B} ( \pi_{\text{NM}} (a^* | \sbf  \oplus s) ) \big] \big] + \Var_{\sbf  \oplus s \sim p^{\pi_{\text{NM}}}_t} \big[ \E \big[ \mathcal{B} ( \pi_{\text{NM}} (a^* | \sbf  \oplus s) ) \big] \big].
	\label{eq:lotv2}
	\end{equation}
	We know from Lemma~\ref{lem:sufficiency} that the policy $\pi_{\text{NM}}$ is deterministic, so that $\Var \big[ \mathcal{B} ( \pi_{\text{NM}} (a^* | \sbf  \oplus s) ) \big] = 0$ for every $\sbf  \oplus s$. We then repeat the same steps in order to compare the two different Markovian policies:
    \begin{align}
		\Var \big[ A | s_i, t \big]
		&= \E_{s_{-i}} \Big[ \Var \big[ A | s_i, s_{-i}, t \big] \Big] + \Var_{s_{-i}} \Big[ \E\big[ A | s_i, s_{-i}, t \big] \Big].  \nonumber 
	\end{align}
    Repeating the same considerations as before we get that we can use~\eqref{eq:lotv2} to get:
    \begin{align*}
		\Var \big[ \mathcal{B} ( \pi_{\text{M}} (a^*|s_i, t) ) \big] &= \E_{\sbf \oplus(\cdot, s_{-i}) \sim p^{ \bar \pi_{\text{M}}}_t} \Big[ \Var \big[ \mathcal{B} ( \bar \pi_{\text{M}} (a^* | s_i, s_{-i}, t) ) \big] + \E \big[ \mathcal{B} ( \bar \pi_{\text{M}} (a^* |  s_i, s_{-i}, t) ) \big] \Big]\\
        &=\Var_{\sbf  \oplus s \sim p^{\pi_{\text{NM}}}_t} \big[ \E \big[ \mathcal{B} ( \pi_{\text{NM}} (a^* | \sbf  \oplus s) ) \big] \big] + \Var_{\sbf \oplus(\cdot, s_{-i}) \sim p^{\bar \pi_{\text{M}}}_t} \big[ \E \big[ \mathcal{B} ( \bar \pi_{\text{M}} (a^* |  s_i, s_{-i}, t) ) \big] \big].
	\end{align*}
\end{proof}

%\paragraph*{Quasi-Potentialness of Disjoint Objectives.}~All results so far have focused on non-disjoint objectives, as they allow for simpler analysis. As a side comment, we can at least conjecture that decentralized learning algorithms will enjoy good performance guarantees, as the infinite-trajectory formulation of the disjoint objective of Eq.~\eqref{eq:mse_decentralized} defines an almost Potential Game in the sense of~\citet{guo2024markovalphapotentialgames} with respect to the mixture objective of Eq.~\eqref{eq:mse_mixture}. Unfortunately, this is a conjecture since convergence properties of decentralized algorithms in \emph{Potential} CMGs are yet unknown, and it merely suggests that there are good reasons to believe that decentralized learning over Eq.~\eqref{eq:mse_decentralized} does not lead to learning instabilities.
%\begin{restatable}[$\alpha$-Potentialness of Disjoint Objectives]{lem}{potentialness} \label{lem:potentialness} For every Convex Markov Game $\mathcal M$ equipped with a $L$-Lipschitz function $\mathcal F$, Eq.~\eqref{eq:mse_decentralized} define an $\alpha$-potential game, namely that for any $\pi = (\pi^i,\pi^{-i}), \tilde \pi = (\tilde \pi^i, \pi^i)$, $i \in [\Ns]$ it holds that \begin{equation*}\| \zeta^i_\infty(\pi) -  \zeta^i_\infty(\tilde \pi) - (\zeta_\infty(\pi) -  \zeta_\infty(\tilde \pi))\|_1 \leq \alpha.\end{equation*}\end{restatable}

\section{Details on the Empirical Corroboration.}
\label{apx:exp}

All the experiments were performed over an Apple M2 chip (8-core CPU, 8-core GPU, 16-core Neural Engine) with 8 GB unified memory with a maximum time of execution of 24 hours.

\paragraph*{Environments.}~The main empirical proof of concept was based on two environments. First, Env. (\textbf{i}), the so called \emph{secret room} environment by~\citet{liu2021cooperative}. In this environment, two agents operate within two rooms of a $10\times10$ discrete grid. There is one switch in each room, one in position $(1,9)$ (corner of first room), another in position $(9,1)$ (corner of second room). The rooms are separated by a door and agents start in the same room deterministically at positions $(1,1)$ and $(2,2)$ respectively. The door will open only when one of the switches is occupied, which means that the (Manhattan) distance between one of the agents and the switch is less than $1.5$. The full state vector contains $x, y$ locations of the two agents and binary variables to indicate if doors are open \emph{but} per-agent policies are conditioned on their respective states only and the state of the door. For Sparse-Rewards Tasks, the goal was set to be deterministically at the worst case, namely $(9,9)$ and to provide a positive reward to both the agents of $100$ when reached, which means again that the (Manhattan) distance between one of the agents and the switch is less than $1.5$, a reward of $0$ otherwise. The second environment, Env. (\textbf{ii}), was the MaMuJoCo \emph{reacher} environment~\cite{peng2021facmac}. In this environment, two agents operate the two linked joints and each space dimension is discretized over $10$ bins. Per-agent policies were conditioned on their respective joint angles only.  For Sparse-Rewards Tasks, the goal was set to be randomly at the worst case, namely on position $(\pm0.21, \pm0.21)$ on the boundary of the reachable area. Reaching the goal mean to have a tip position (not observable by the agents and not discretized) at a distance less that $0.05$ and provides a positive reward to both the agents of $1$ when reached, a reward of $0$ otherwise. 

\paragraph*{Class of Policies.}~In Env. (\textbf{i}), the policy was parametrized by a dense $(64,64)$ Neural Network that takes as input the per-agent state features and outputs an action vector probabilities through a last soft-max layer. In Env. (\textbf{ii}), the policy was represented by a Gaussian distribution with diagonal covariance matrix. It takes as input the environment state features and outputs an action vector. The mean is state-dependent and is the downstream output of a a dense $(64,64)$ Neural Network. The standard deviation is state-independent, represented by a separated trainable vector and initialized to $-0.5$. The weights are initialized via Xavier Initialization.

\paragraph*{Trust Region Pure Exploration (TRPE).}~As outlined in the pseudocode of Algorithm~\ref{alg:trpe}, in each epoch a dataset of $N$ trajectories is gathered for a given exploration horizon $T$, leading to the reported number of samples. Throughout the experiment the number of epochs $e$ were set equal to $e=10k$, the number of trajectories $N=10$, the KL threshold $\delta = 6$, the maximum number of off-policy iterations set to $n_{\text{off,iter}} = 20$, the learning rate was set to $\eta = 10^{-5}$ and the number of seeds set equal to $4$ due to the inherent low stochasticity of the environment.

\paragraph*{Multi-Agent TRPO (MA-TRPO).}~We follow the same notation in~\citet{duan2016benchmarking}. Agents have independent critics $(64,64)$ Dense networks and in each epoch a dataset of $N$ trajectories is gathered for a given exploration horizon $T$ for each agent, leading to the reported number of samples. Throughout the experiment the number of epochs $e$ were set equal to $e=100$, the number of trajectories building the batch size $N=20$, the KL threshold $\delta = 10^{-4}$, the maximum number of off-policy iterations set to $n_{\text{off,iter}} = 20$, the discount was set to $\gamma = 0.99$.

The Repository is made available at the following \href{https://anonymous.4open.science/r/trpe-DB16/README.md}{Repository.}

\newpage

\begin{figure*}[]
    \centering
    \begin{tikzpicture}
    % Draw rounded box for the legend
    \node[draw=black, rounded corners, inner sep=2pt, fill=white] (legend) at (0,0) {
        \begin{tikzpicture}[scale=0.8]
            % Mixture
            \draw[thick, color={rgb,255:red,230; green,159; blue,0}, opacity=0.8] (0,0) -- (1,0);
            \fill[color={rgb,255:red,230; green,159; blue,0}, opacity=0.2] (0,-0.1) rectangle (1,0.1);
            \node[anchor=west, font=\scriptsize] at (1.2,0) {Mixture};
            
            % Joint
            \draw[thick, dashed, color={rgb,255:red,86; green,180; blue,233}, opacity=0.8] (2.5,0) -- (3.5,0);
            \fill[color={rgb,255:red,86; green,180; blue,233}, opacity=0.2] (2.5,-0.1) rectangle (3.5,0.1);
            \node[anchor=west, font=\scriptsize] at (3.7,0) {Joint};

            % Disjoint
            \draw[thick, dotted, color={rgb,255:red,204; green,121; blue,167}, opacity=0.8] (4.7,0) -- (5.7,0);
            \fill[color={rgb,255:red,204; green,121; blue,167}, opacity=0.2] (4.7,-0.1) rectangle (5.7,0.1);
            \node[anchor=west, font=\scriptsize] at (5.9,0) {Disjoint};
            
            % Uniform
            \draw[thick, color={rgb,255:red,153; green,153; blue,153}, opacity=0.8] (7.2,0) -- (8.2,0);
            \fill[color={rgb,255:red,153; green,153; blue,153}, opacity=0.2] (7.2,-0.1) rectangle (8.2,0.1);
            \node[anchor=west, font=\scriptsize] at (8.4,0) {Random Initialization};
        \end{tikzpicture}
    };
\end{tikzpicture}
    %\hfill
    \vfill
    %vspace{-0.2cm}
    \begin{subfigure}[b]{0.245\textwidth}
        \includegraphics[width=\textwidth]{figures/roomS10N10/Exploration/T50/joint_entropy_nokl.pdf}
        %\vspace{-0.8cm}
        \caption{\centering TRPE Joint Entropy (Env.~(\textbf{i}), $T=50$).}
        \label{subfig:imagea1}
    \end{subfigure}
    \hfill
    \begin{subfigure}[b]{0.245\textwidth}
        \includegraphics[width=\textwidth]{figures/roomS10N10/Exploration/T50/mixture_entropy_nokl.pdf}
        %\vspace{-0.8cm}
        \caption{\centering TRPE Mixture Entropy (Env.~(\textbf{i}), $T=50$).}
        \label{subfig:imagea2}
    \end{subfigure}
    \hfill
    \begin{subfigure}[b]{0.245\textwidth}
        \centering
        \includegraphics[width=\textwidth]{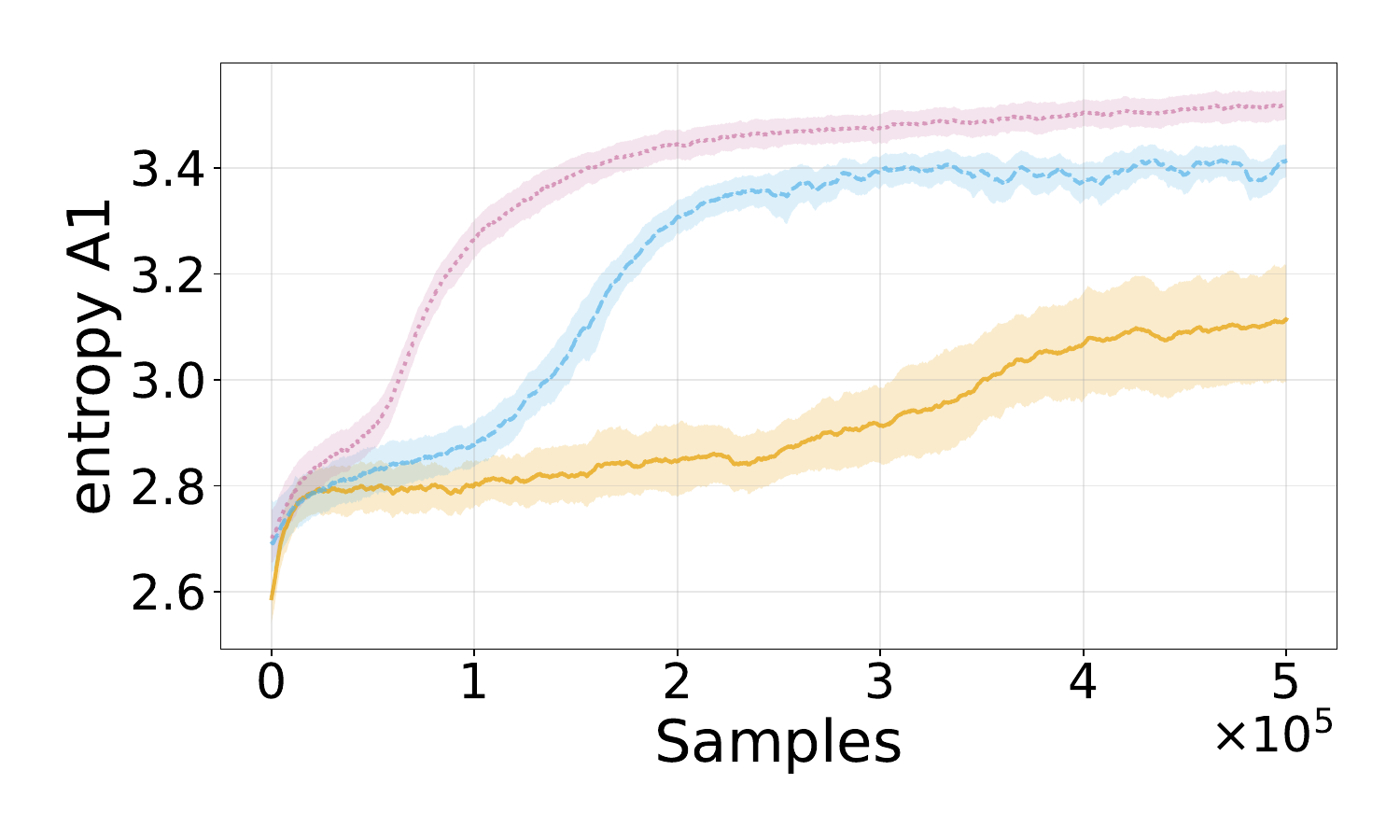}
        %\vspace{-0.8cm}
        \caption{\centering TRPE Entropy Agent 1 (Env.~(\textbf{i}), $T=50$).}
        \label{subfig:image11}
    \end{subfigure}
    \hfill
    \begin{subfigure}[b]{0.245\textwidth}
        \centering
        \includegraphics[width=\textwidth]{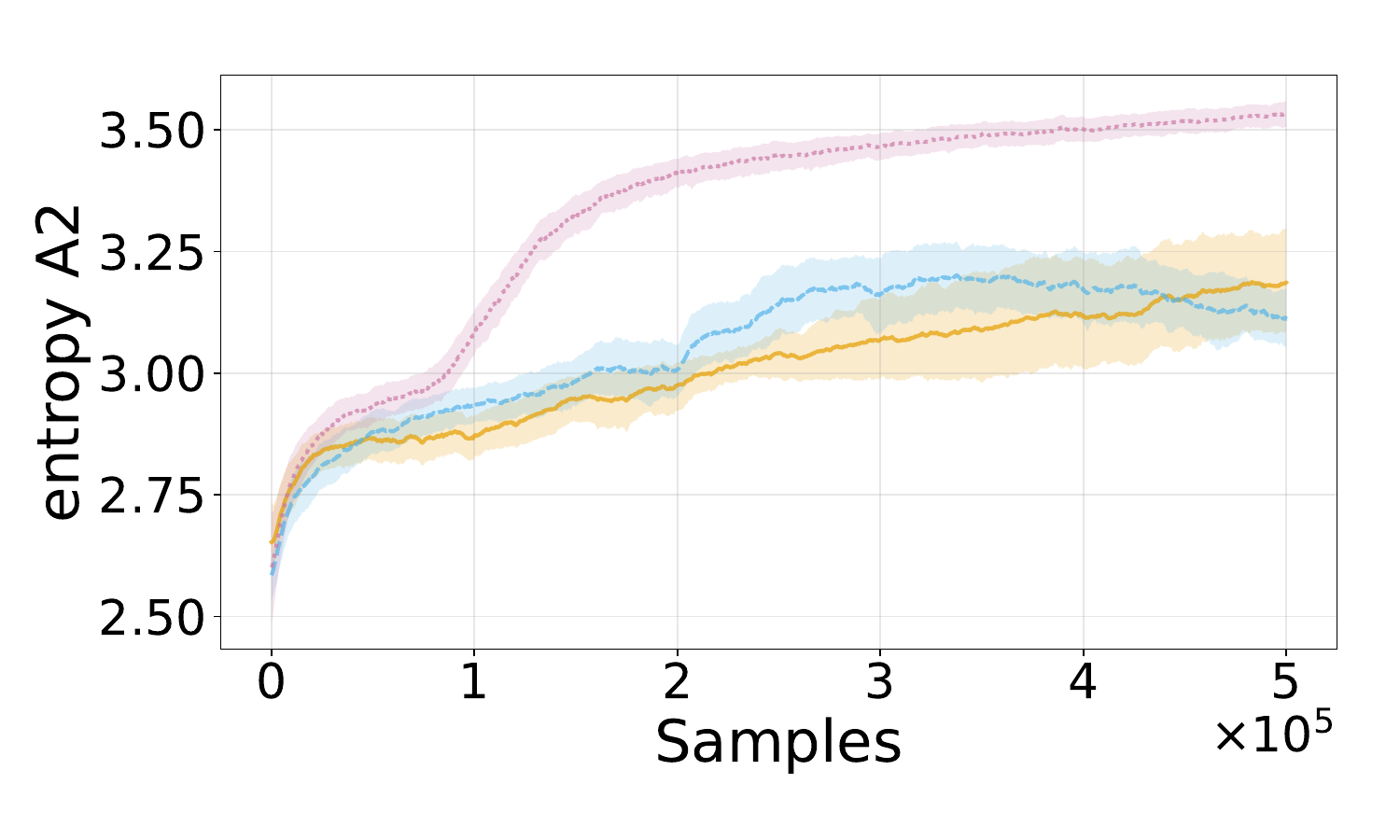}
        %\vspace{-0.8cm}
        \caption{\centering TRPE Entropy Agent 2 (Env.~(\textbf{i}), $T=50$).}
        \label{subfig:image11}
    \end{subfigure}
    \vfill
    \begin{subfigure}[b]{0.245\textwidth}
        \includegraphics[width=\textwidth]{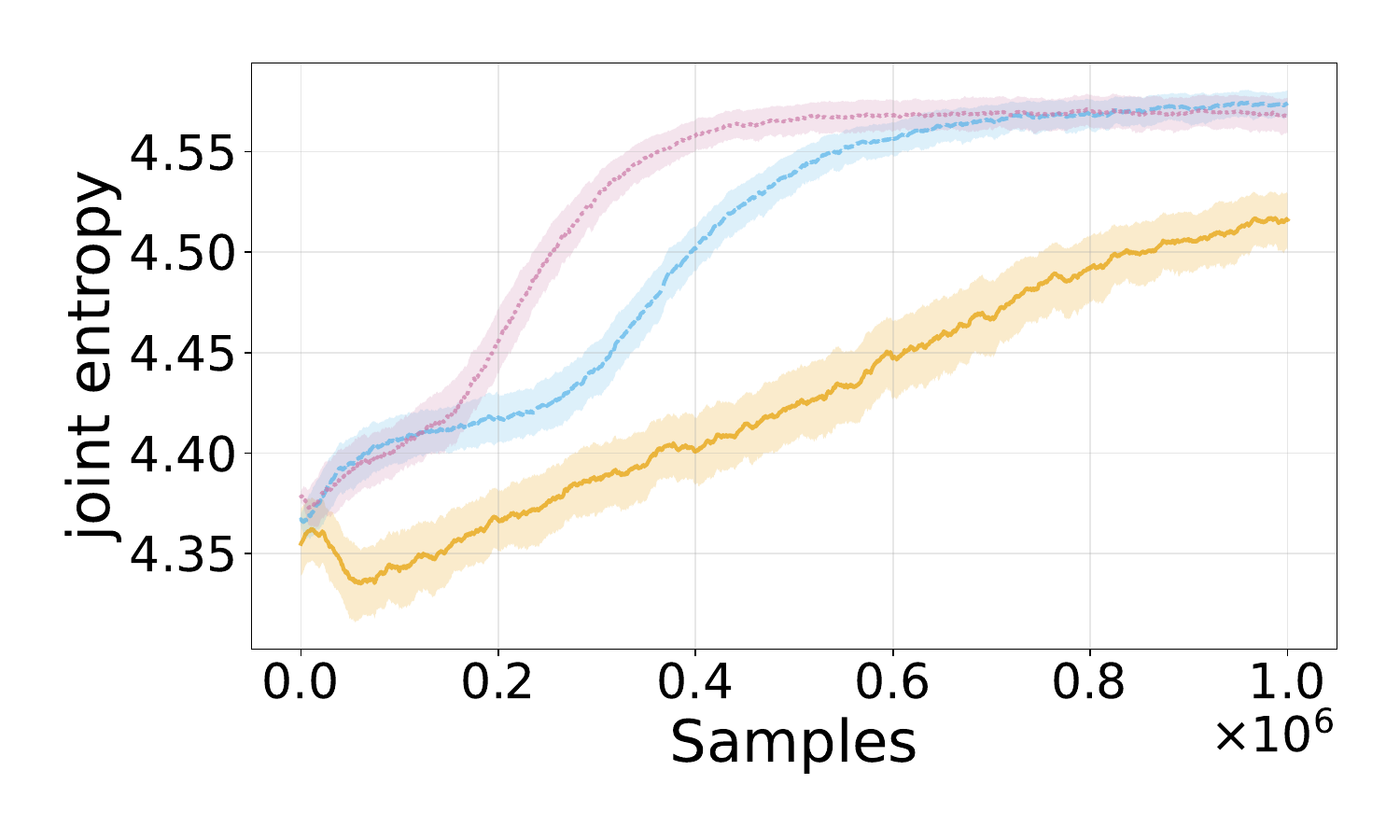}
        %\vspace{-0.8cm}
        \caption{\centering TRPE Joint Entropy (Env.~(\textbf{i}), $T=100$).}
        \label{subfig:imagea1}
    \end{subfigure}
    \hfill
    \begin{subfigure}[b]{0.245\textwidth}
        \includegraphics[width=\textwidth]{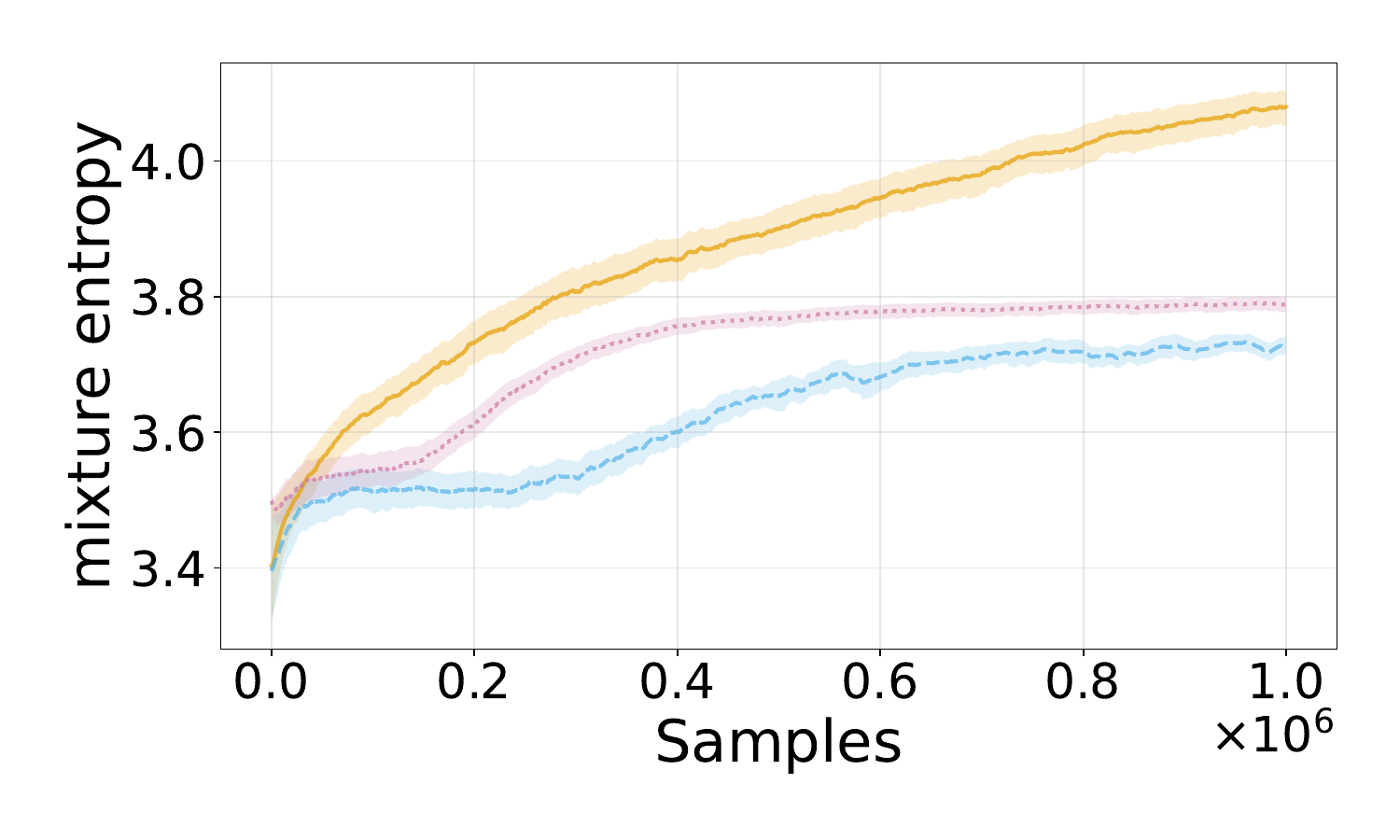}
        %\vspace{-0.8cm}
        \caption{\centering TRPE Mixture Entropy (Env.~(\textbf{i}), $T=100$).}
        \label{subfig:imagea2}
    \end{subfigure}
    \hfill
    \begin{subfigure}[b]{0.245\textwidth}
        \centering
        \includegraphics[width=\textwidth]{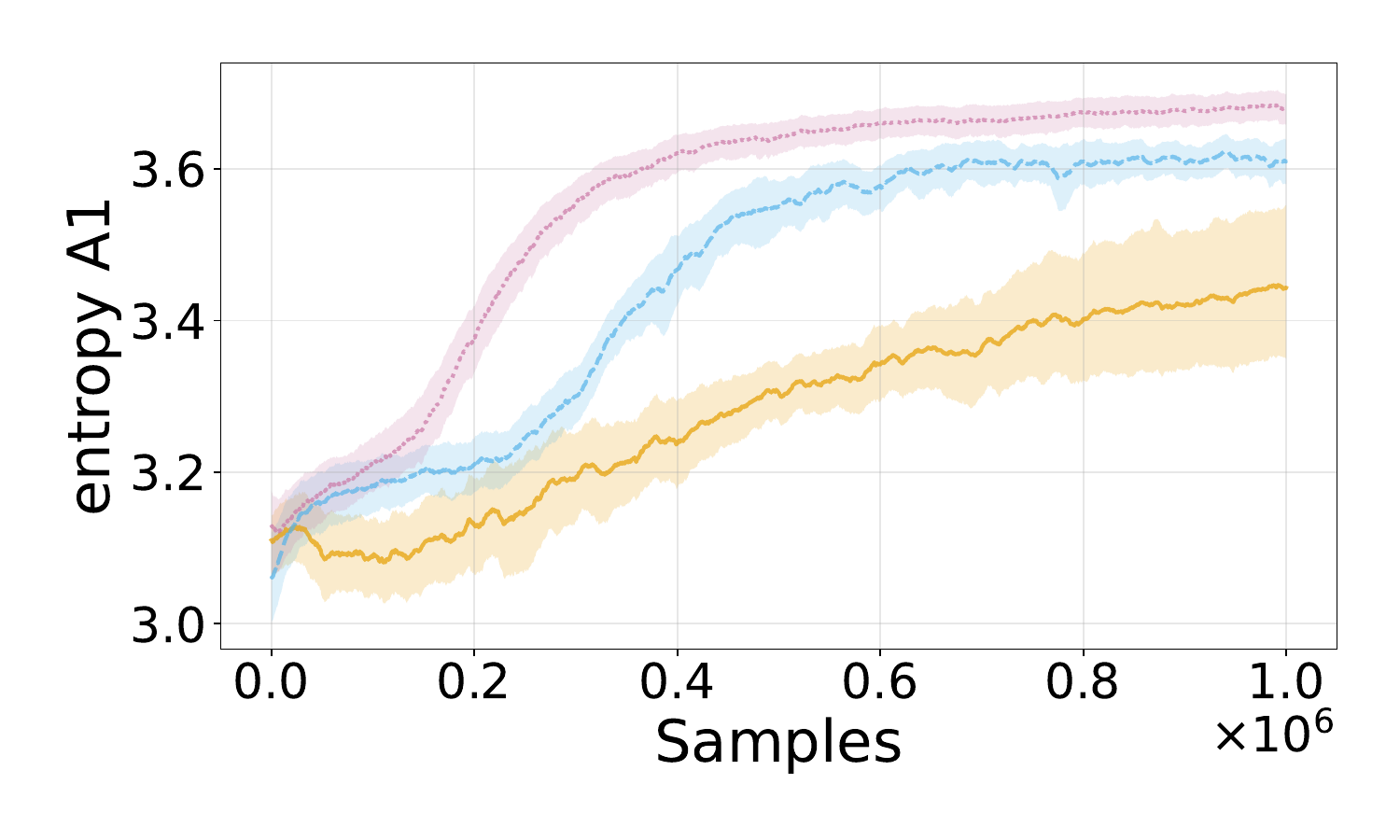}
        %\vspace{-0.8cm}
        \caption{\centering TRPE Entropy Agent 1 (Env.~(\textbf{i}), $T=100$).}
        \label{subfig:image11}
    \end{subfigure}
    \hfill
    \begin{subfigure}[b]{0.245\textwidth}
        \centering
        \includegraphics[width=\textwidth]{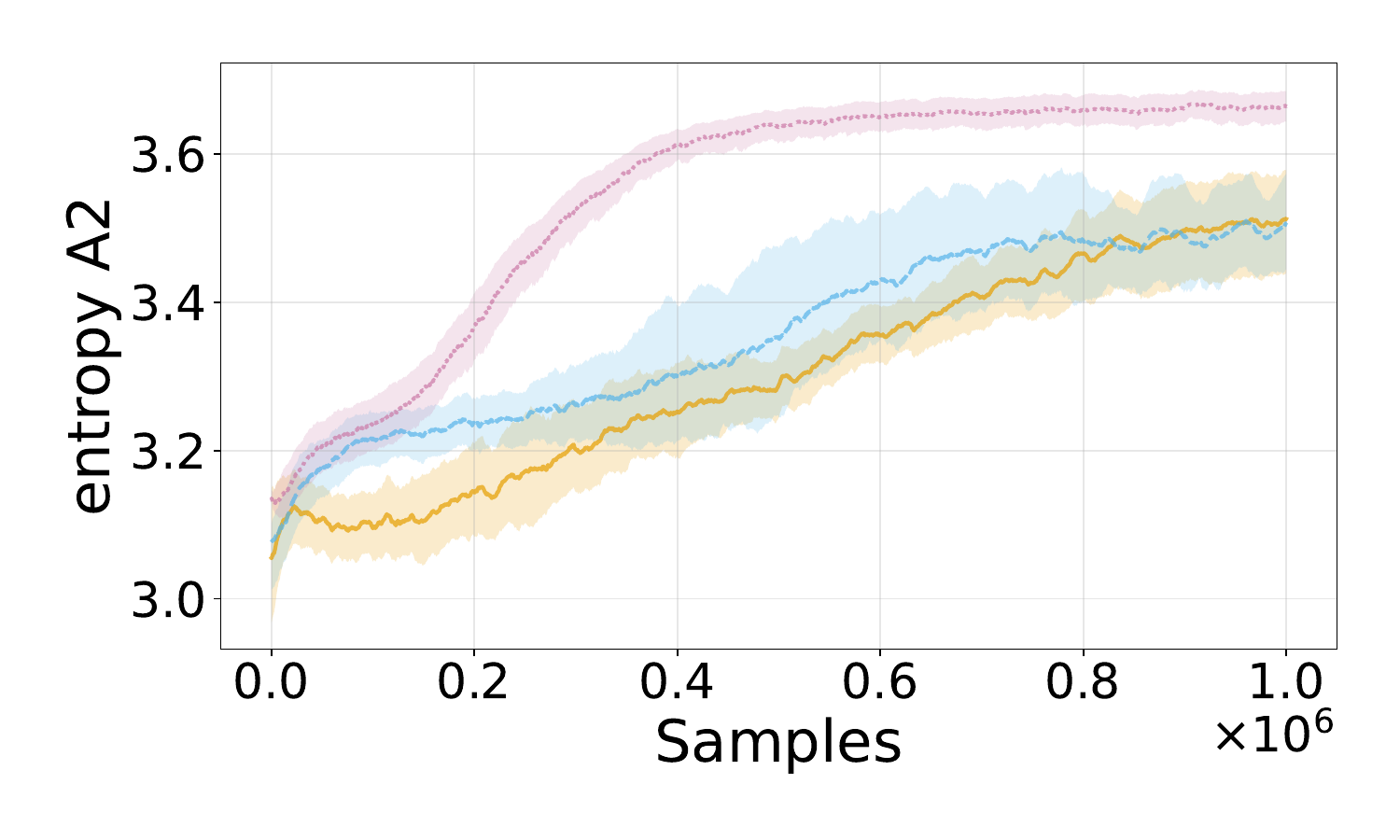}
        %\vspace{-0.8cm}
        \caption{\centering TRPE Entropy Agent 2 (Env.~(\textbf{i}), $T=100$).}
        \label{subfig:image11}
    \end{subfigure}
    \vfill 
    \begin{subfigure}[b]{0.245\textwidth}
        \includegraphics[width=\textwidth]{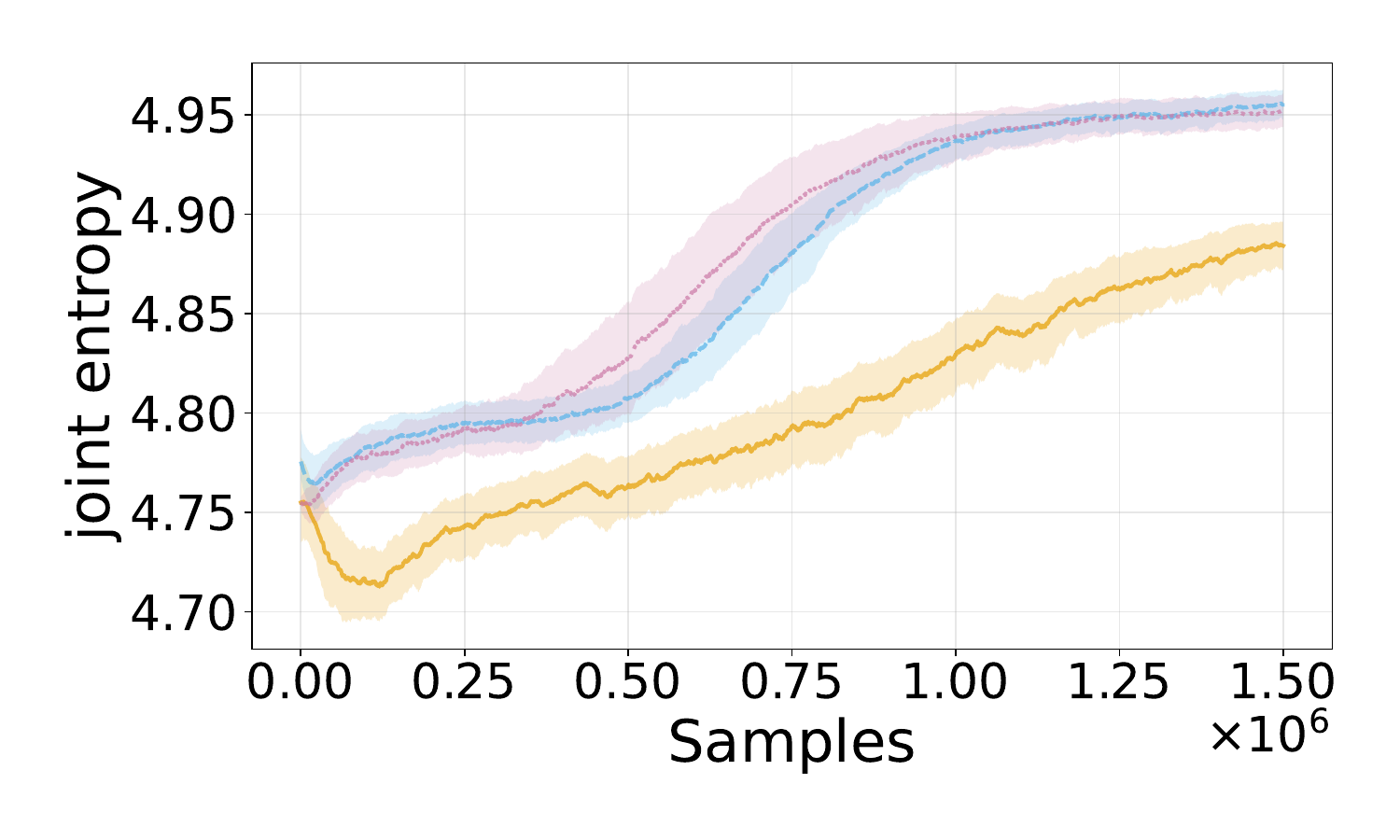}
        %\vspace{-0.8cm}
        \caption{\centering TRPE Joint Entropy (Env.~(\textbf{i}), $T=150$).}
        \label{subfig:imagea1}
    \end{subfigure}
    \hfill
    \begin{subfigure}[b]{0.245\textwidth}
        \includegraphics[width=\textwidth]{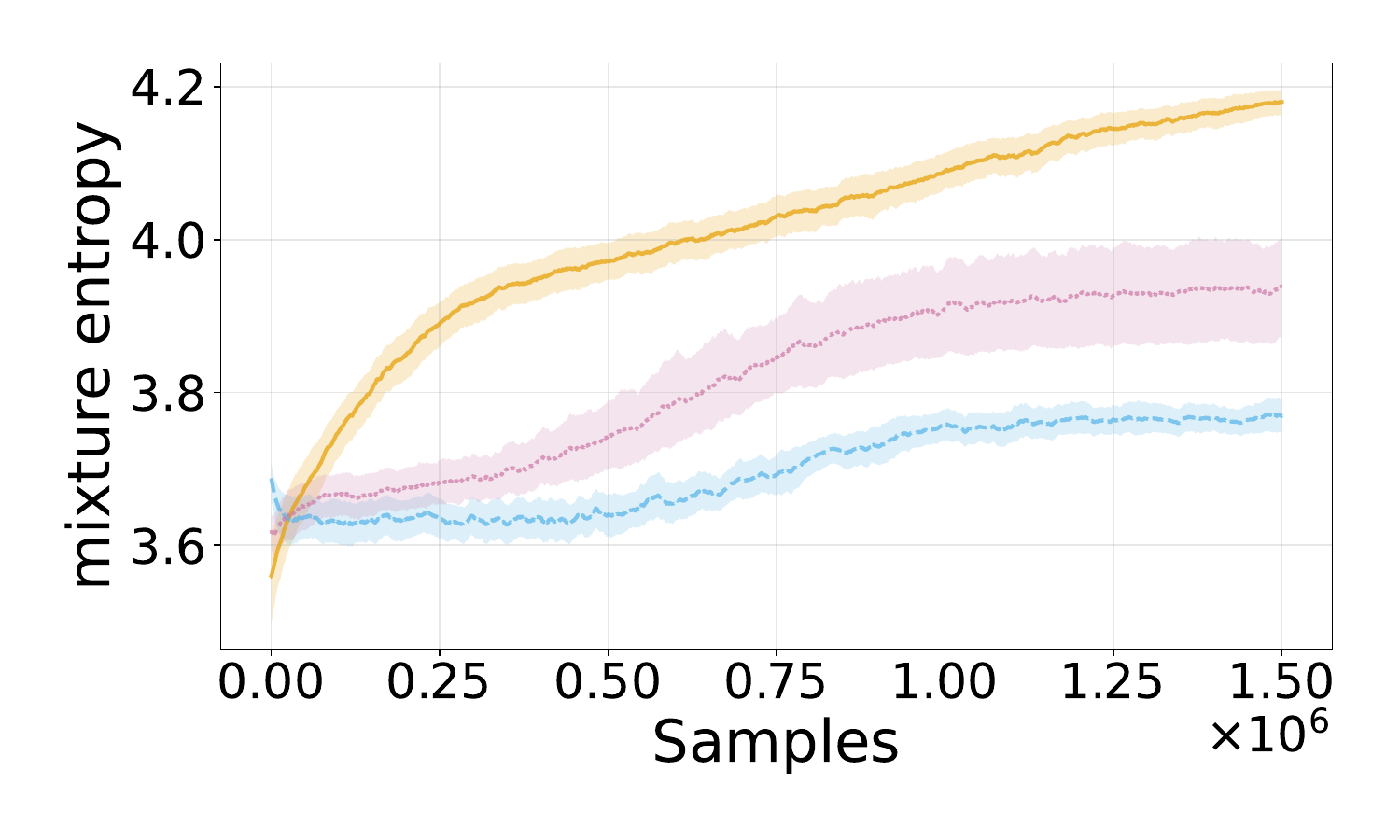}
        %\vspace{-0.8cm}
        \caption{\centering TRPE Mixture Entropy (Env.~(\textbf{i}), $T=150$).}
        \label{subfig:imagea2}
    \end{subfigure}
    \hfill
    \begin{subfigure}[b]{0.245\textwidth}
        \centering
        \includegraphics[width=\textwidth]{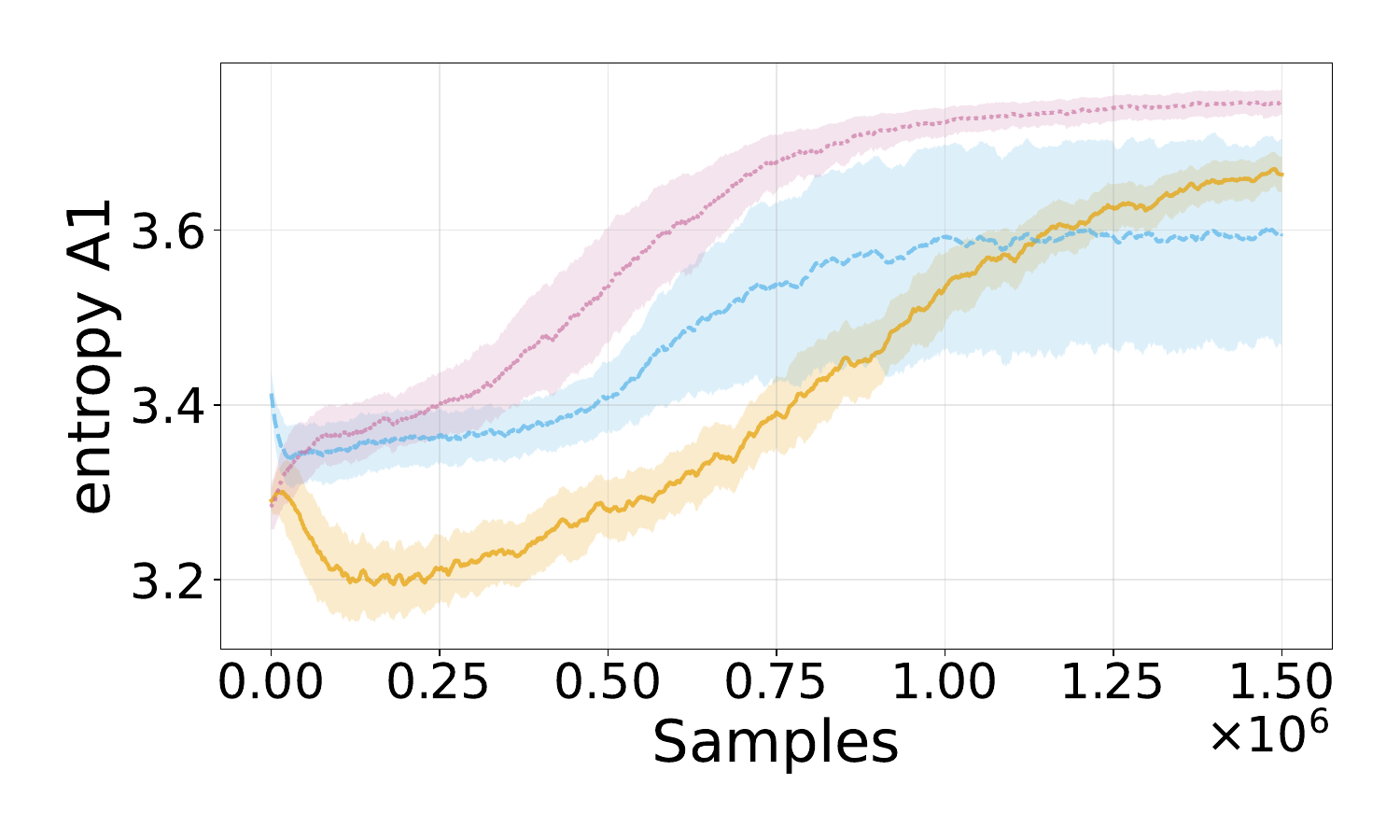}
        %\vspace{-0.8cm}
        \caption{\centering TRPE Entropy Agent 1 (Env.~(\textbf{i}), $T=150$).}
        \label{subfig:image11}
    \end{subfigure}
    \hfill
    \begin{subfigure}[b]{0.245\textwidth}
        \centering
        \includegraphics[width=\textwidth]{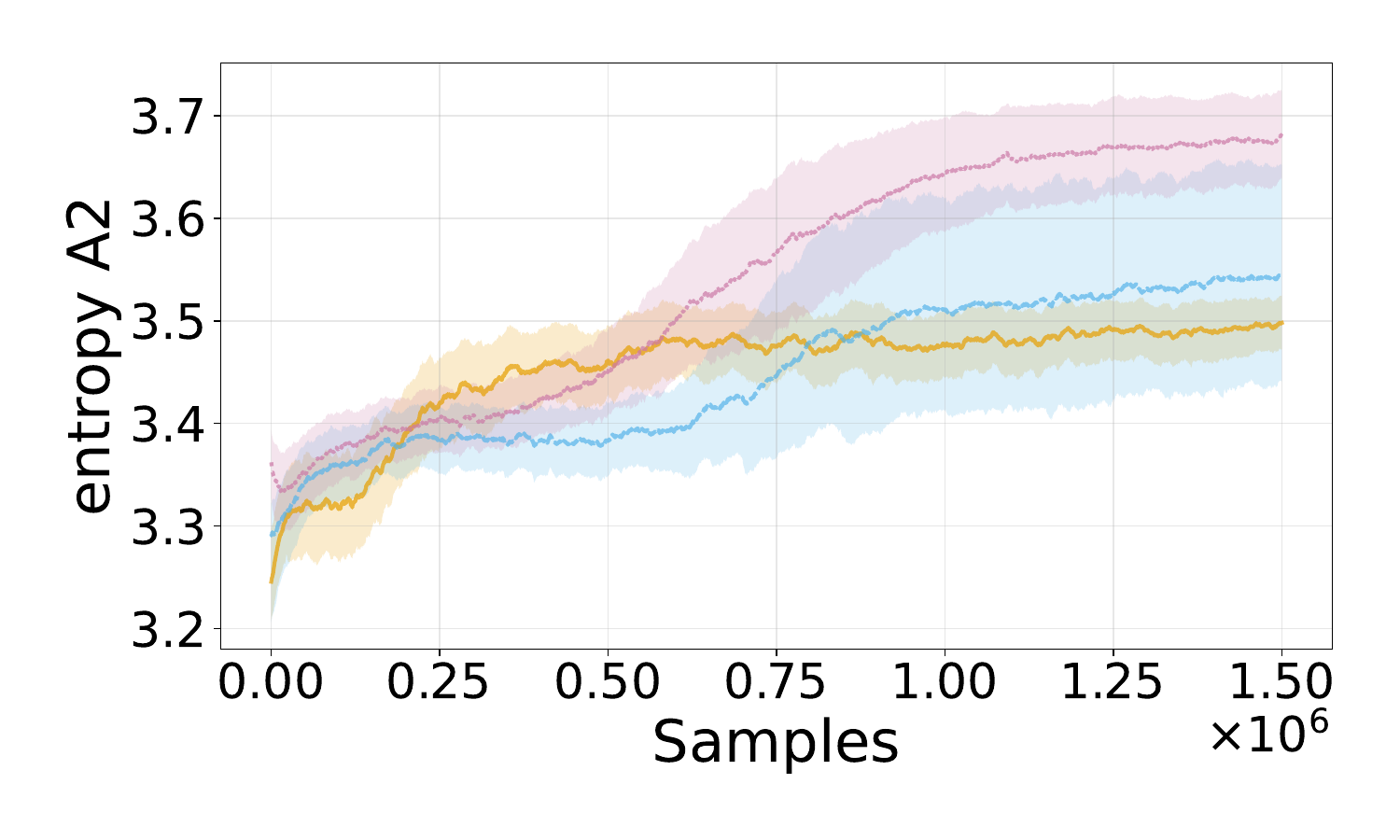}
        %\vspace{-0.8cm}
        \caption{\centering TRPE Entropy Agent 2 (Env.~(\textbf{i}), $T=150$).}
        \label{subfig:image11}
    \end{subfigure}
    \vfill
    \begin{subfigure}[b]{0.245\textwidth}
        \includegraphics[width=\textwidth]{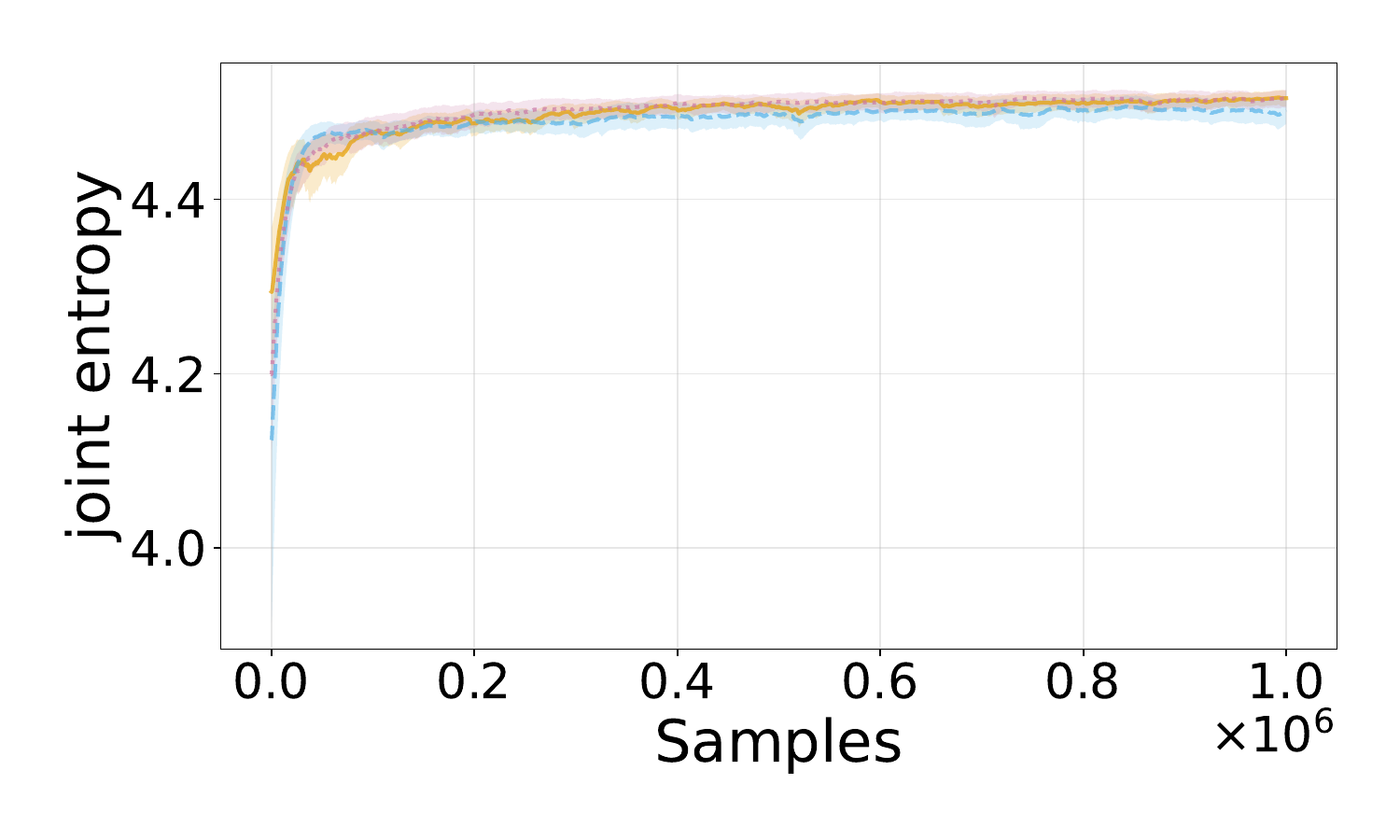}
        %\vspace{-0.8cm}
        \caption{\centering TRPE Joint Entropy (Env.~(\textbf{ii}), $T=100$).}
        \label{subfig:imagea1}
    \end{subfigure}
    \hfill
    \begin{subfigure}[b]{0.245\textwidth}
        \includegraphics[width=\textwidth]{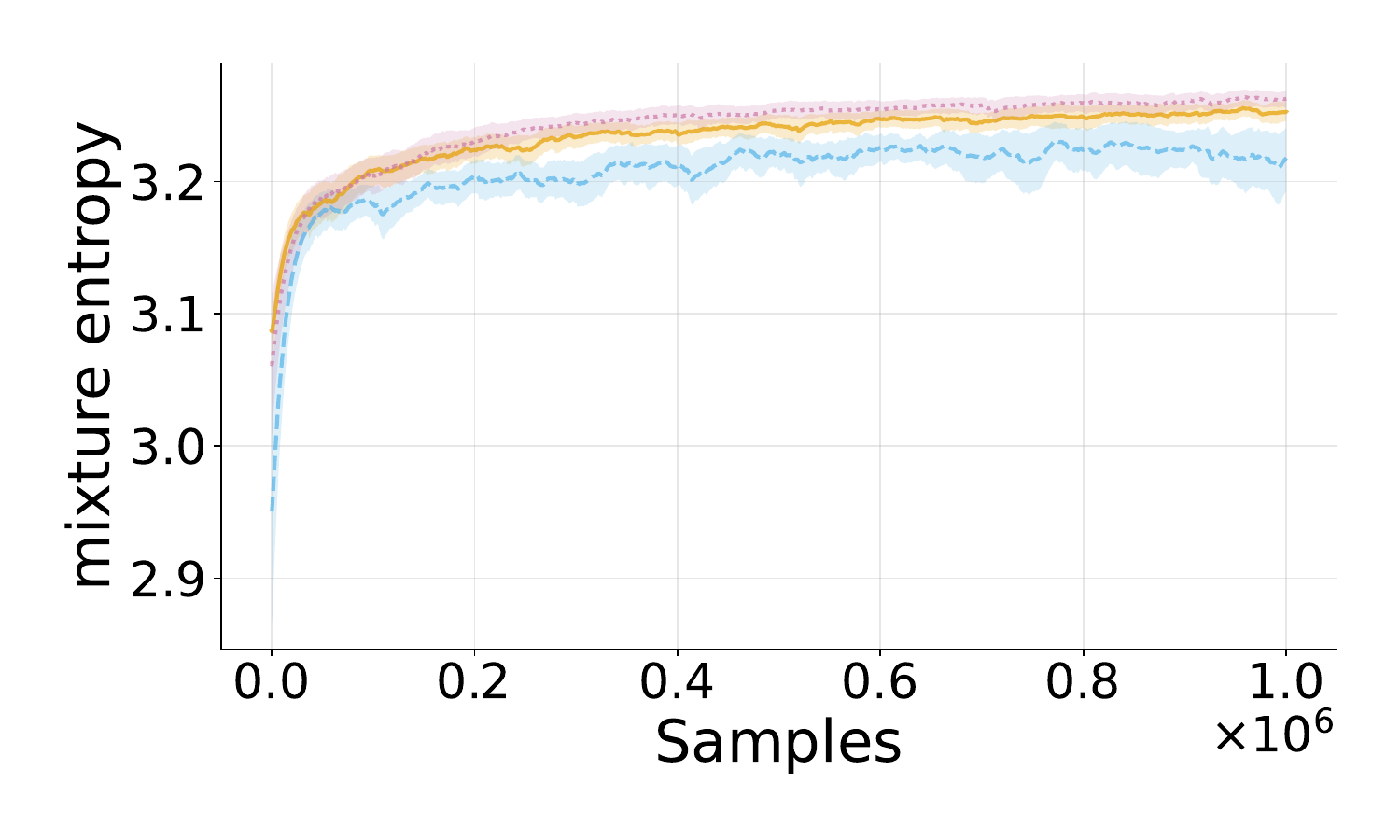}
        %\vspace{-0.8cm}
        \caption{\centering TRPE Mixture Entropy (Env.~(\textbf{ii}), $T=100$).}
        \label{subfig:imagea2}
    \end{subfigure}
    \hfill
    \begin{subfigure}[b]{0.245\textwidth}
        \centering
        \includegraphics[width=\textwidth]{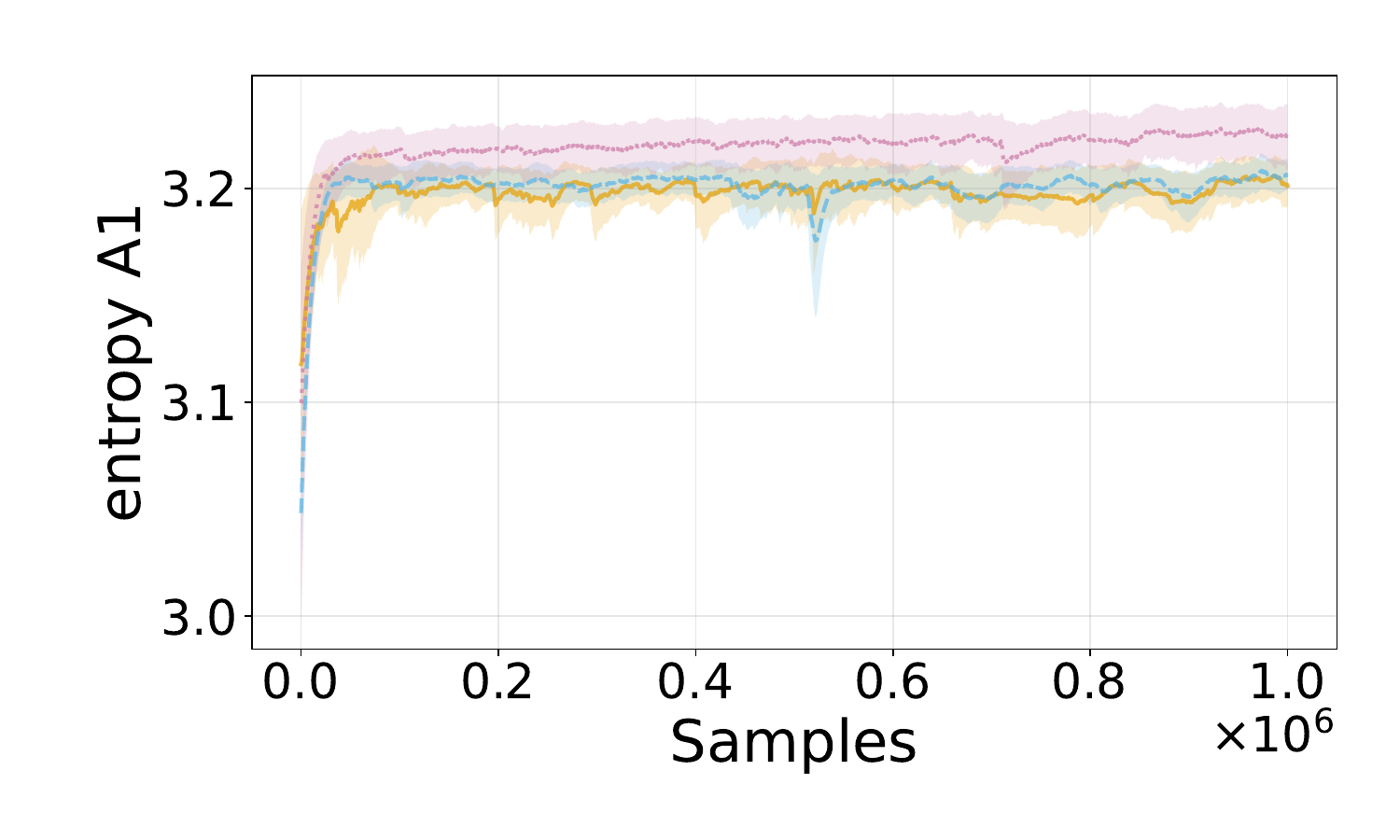}
        %\vspace{-0.8cm}
        \caption{\centering TRPE Entropy Agent 1 (Env.~(\textbf{ii}), $T=100$).}
        \label{subfig:image11}
    \end{subfigure}
    \hfill
    \begin{subfigure}[b]{0.245\textwidth}
        \centering
        \includegraphics[width=\textwidth]{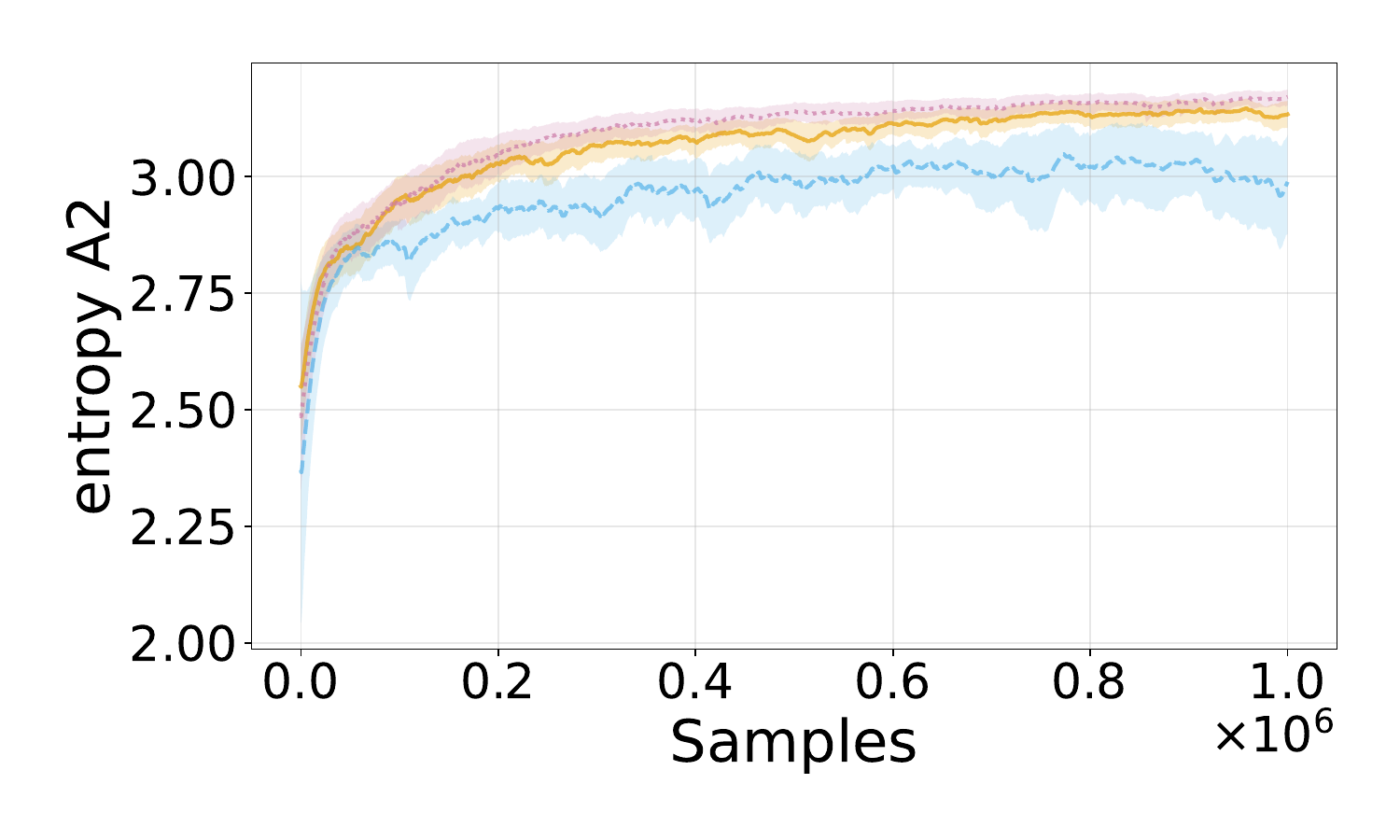}
        %\vspace{-0.8cm}
        \caption{\centering TRPE Entropy Agent 2 (Env.~(\textbf{ii}), $T=100$).}
        \label{subfig:image11}
    \end{subfigure}
\caption{\centering Full Visualization of Reported Experiments. Experiments with longer horizons highlight how \textbf{the easier the task, the less crucial the distinction between the objectives is.}}
\label{fig:pretraining}
\end{figure*}

\begin{figure*}[!]
    \centering
    %vspace{-0.2cm}
    \begin{tikzpicture}
    % Draw rounded box for the legend
    \node[draw=black, rounded corners, inner sep=2pt, fill=white] (legend) at (0,0) {
        \begin{tikzpicture}[scale=0.8]
            % Mixture
            \draw[thick, color={rgb,255:red,230; green,159; blue,0}, opacity=0.8] (0,0) -- (1,0);
            \fill[color={rgb,255:red,230; green,159; blue,0}, opacity=0.2] (0,-0.1) rectangle (1,0.1);
            \node[anchor=west, font=\scriptsize] at (1.2,0) {Mixture};
            
            % Joint
            \draw[thick, dashed, color={rgb,255:red,86; green,180; blue,233}, opacity=0.8] (2.5,0) -- (3.5,0);
            \fill[color={rgb,255:red,86; green,180; blue,233}, opacity=0.2] (2.5,-0.1) rectangle (3.5,0.1);
            \node[anchor=west, font=\scriptsize] at (3.7,0) {Joint};

            % Disjoint
            \draw[thick, dotted, color={rgb,255:red,204; green,121; blue,167}, opacity=0.8] (4.7,0) -- (5.7,0);
            \fill[color={rgb,255:red,204; green,121; blue,167}, opacity=0.2] (4.7,-0.1) rectangle (5.7,0.1);
            \node[anchor=west, font=\scriptsize] at (5.9,0) {Disjoint};
            
            % Uniform
            \draw[thick, color={rgb,255:red,153; green,153; blue,153}, opacity=0.8] (7.2,0) -- (8.2,0);
            \fill[color={rgb,255:red,153; green,153; blue,153}, opacity=0.2] (7.2,-0.1) rectangle (8.2,0.1);
            \node[anchor=west, font=\scriptsize] at (8.4,0) {Random Initialization};
        \end{tikzpicture}
    };
\end{tikzpicture}
    %\hfill
    \vfill
    \begin{subfigure}[b]{0.245\textwidth}
        \includegraphics[width=\textwidth]{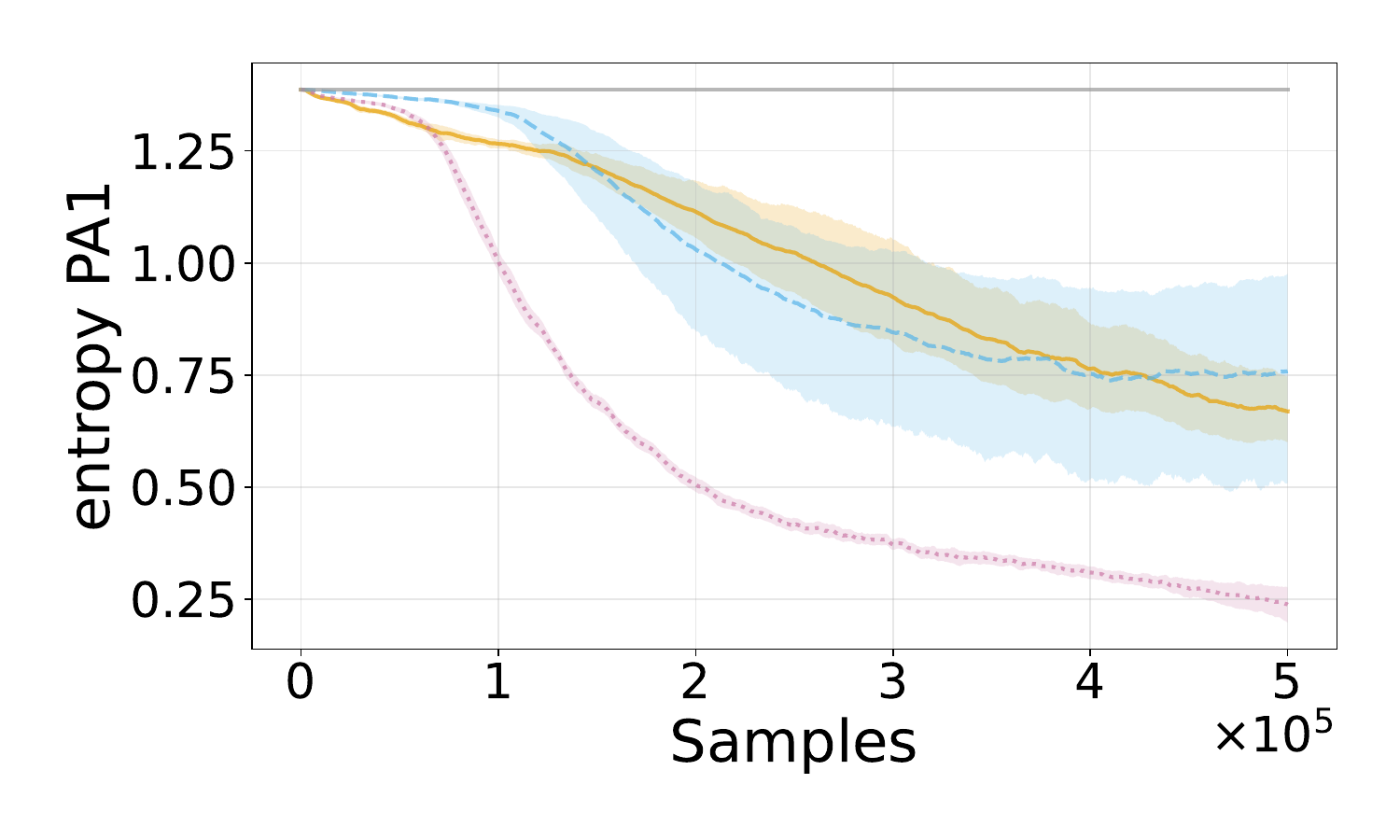}
        %\vspace{-0.8cm}
        \caption{\centering Entropy of Agent 1 Policy in TRPE Training (\textbf{i}), $T=50$).}
        \label{subfig:image93}
    \end{subfigure}
    \hfill
    \begin{subfigure}[b]{0.245\textwidth}
        \includegraphics[width=\textwidth]{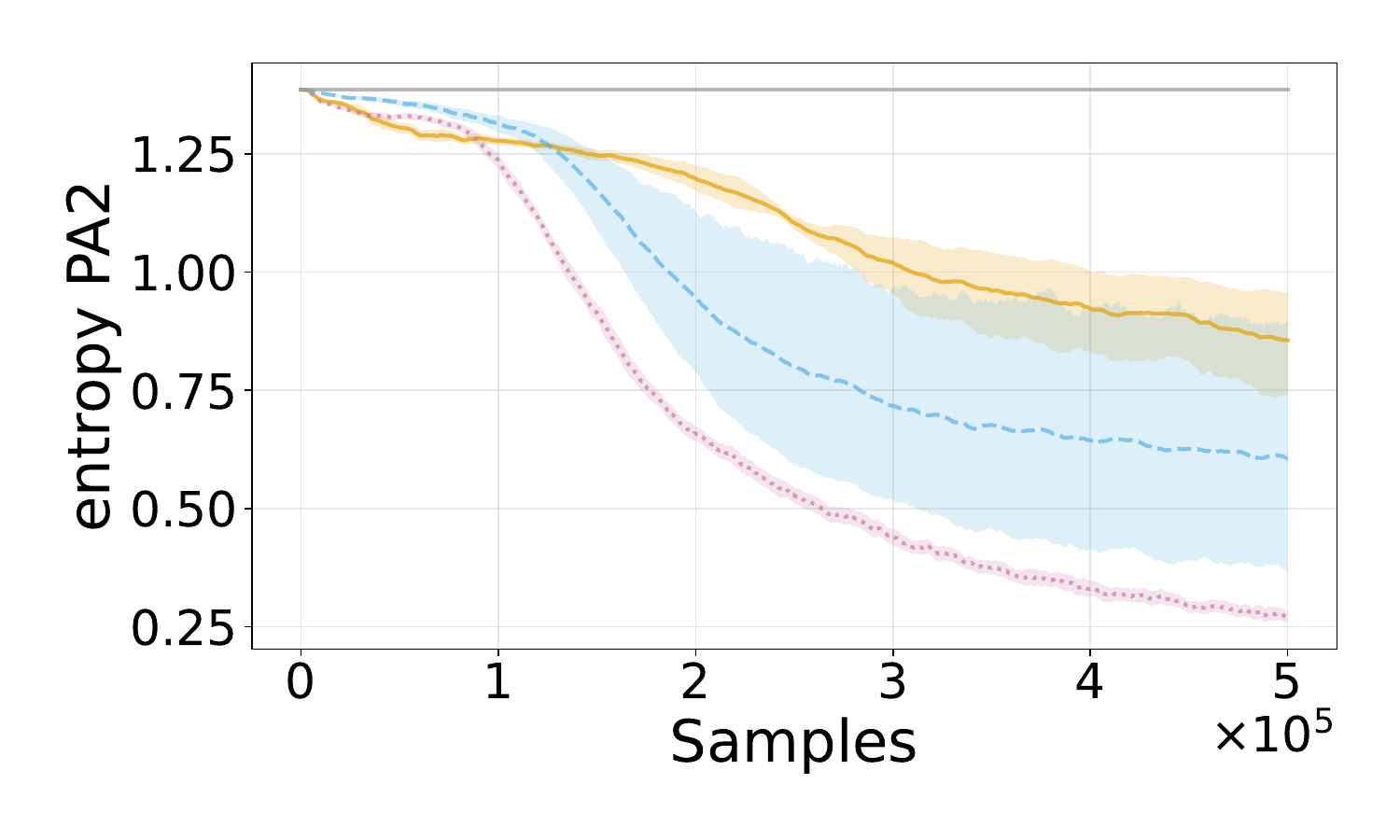}
        %\vspace{-0.8cm}
        \caption{\centering Entropy of Agent 2 Policy in TRPE Training (\textbf{i}), $T=50$).}
        \label{subfig:image130}
    \end{subfigure}
    \hfill
    \begin{subfigure}[b]{0.245\textwidth}
        \includegraphics[width=\textwidth]{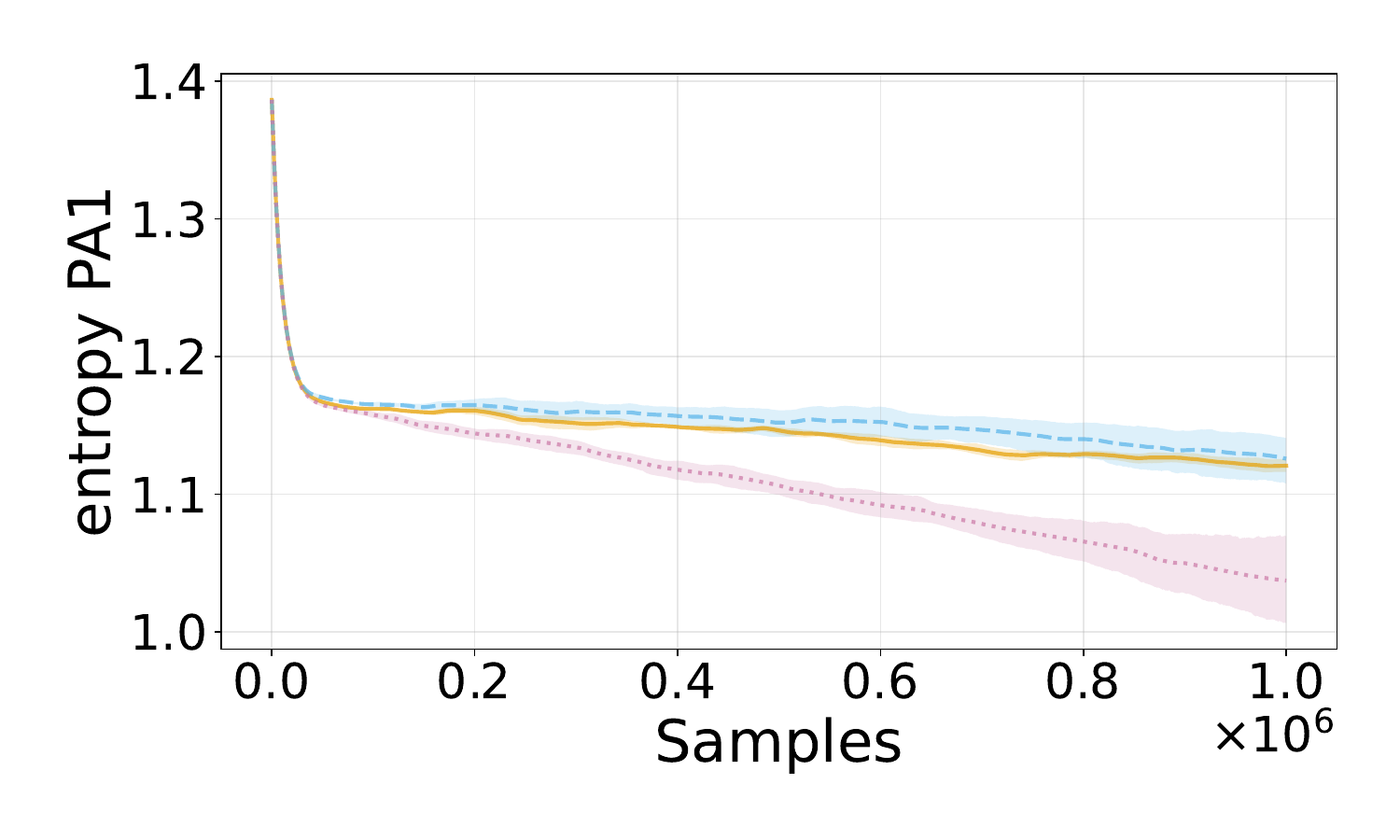}
        %\vspace{-0.8cm}
        \caption{\centering Entropy of Agent 1 Policy in TRPE Training (\textbf{ii}), $T=100$).}
        \label{subfig:image130}
    \end{subfigure}
    \hfill
    \begin{subfigure}[b]{0.245\textwidth}
        \includegraphics[width=\textwidth]{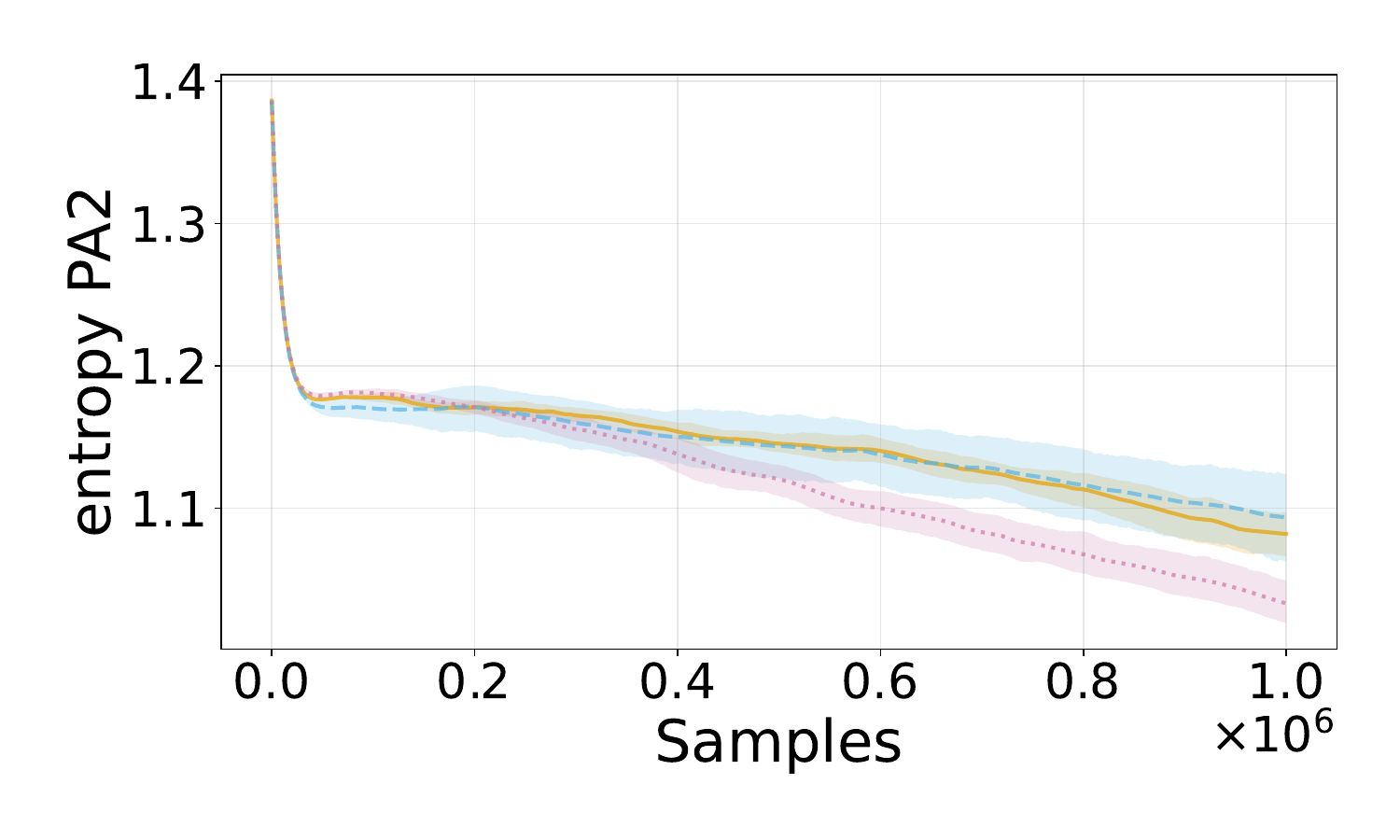}
        %\vspace{-0.8cm}
        \caption{\centering Entropy of Agent 2 Policy in TRPE Training (\textbf{ii}), $T=100$).}
        \label{subfig:image130}
    \end{subfigure}
\caption{\centering Policiy Entropy Insights for TRPO Pretraining in Env (\textbf{i}) and Env (\textbf{ii}). \textbf{Lower Entropic Policies with Disjoint Objectives might justify the difference in pre-training performance even if the performances in training are similar}.}
\label{fig:333}
\end{figure*}
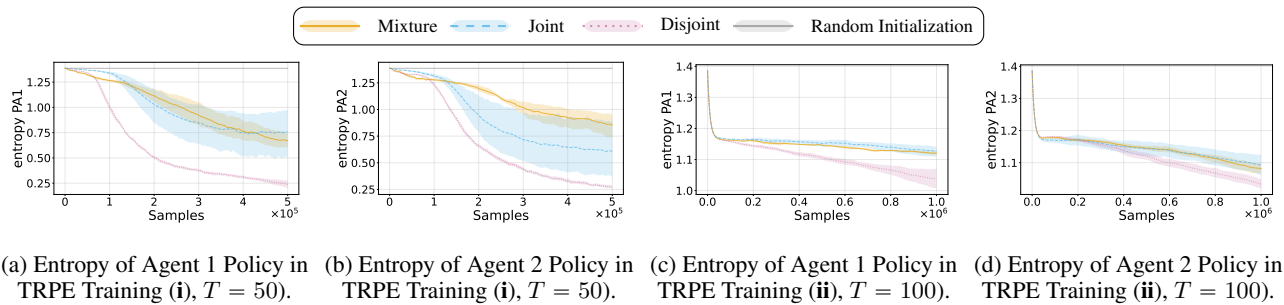

%%%%%%%%%%%%%%%%%%%%%%%%%%%%%%%%%%%%%%%%%%%%%%%%%%%%%%%%%%%%%
\clearpage

\newpage
\section*{NeurIPS Paper Checklist}

\begin{enumerate}

\item {\bf Claims}
    \item[] Question: Do the main claims made in the abstract and introduction accurately reflect the paper's contributions and scope?
    \item[] Answer:  \answerYes{} % Replace by \answerYes{}, \answerNo{}, or \answerNA{}.
    \item[] Justification: Both the theoretical and the empirical claims are explicitly covered throughout the paper: 
    \item[] Guidelines:
    \begin{itemize}
        \item The answer NA means that the abstract and introduction do not include the claims made in the paper.
        \item The abstract and/or introduction should clearly state the claims made, including the contributions made in the paper and important assumptions and limitations. A No or NA answer to this question will not be perceived well by the reviewers. 
        \item The claims made should match theoretical and experimental results, and reflect how much the results can be expected to generalize to other settings. 
        \item It is fine to include aspirational goals as motivation as long as it is clear that these goals are not attained by the paper. 
    \end{itemize}

\item {\bf Limitations}
    \item[] Question: Does the paper discuss the limitations of the work performed by the authors?
    \item[] Answer: \answerYes{} % Replace by \answerYes{}, \answerNo{}, or \answerNA{}.
    \item[] Justification: The authors included an explicit section covering the limitations of the proposed approach, made the assumptions underlying the models explicit and clearly stated the aim of the empirical corroboration in providing evidences of the nature of the new problem rather than confirming SOTA performances of the proposed algorithm.
    \item[] Guidelines:
    \begin{itemize}
        \item The answer NA means that the paper has no limitation while the answer No means that the paper has limitations, but those are not discussed in the paper. 
        \item The authors are encouraged to create a separate "Limitations" section in their paper.
        \item The paper should point out any strong assumptions and how robust the results are to violations of these assumptions (e.g., independence assumptions, noiseless settings, model well-specification, asymptotic approximations only holding locally). The authors should reflect on how these assumptions might be violated in practice and what the implications would be.
        \item The authors should reflect on the scope of the claims made, e.g., if the approach was only tested on a few datasets or with a few runs. In general, empirical results often depend on implicit assumptions, which should be articulated.
        \item The authors should reflect on the factors that influence the performance of the approach. For example, a facial recognition algorithm may perform poorly when image resolution is low or images are taken in low lighting. Or a speech-to-text system might not be used reliably to provide closed captions for online lectures because it fails to handle technical jargon.
        \item The authors should discuss the computational efficiency of the proposed algorithms and how they scale with dataset size.
        \item If applicable, the authors should discuss possible limitations of their approach to address problems of privacy and fairness.
        \item While the authors might fear that complete honesty about limitations might be used by reviewers as grounds for rejection, a worse outcome might be that reviewers discover limitations that aren't acknowledged in the paper. The authors should use their best judgment and recognize that individual actions in favor of transparency play an important role in developing norms that preserve the integrity of the community. Reviewers will be specifically instructed to not penalize honesty concerning limitations.
    \end{itemize}

\item {\bf Theory assumptions and proofs}
    \item[] Question: For each theoretical result, does the paper provide the full set of assumptions and a complete (and correct) proof?
    \item[] Answer: \answerYes{} % Replace by \answerYes{}, \answerNo{}, or \answerNA{}.
    \item[] Justification: All the assumptions are clearly stated, and the proofs are exaustively reported in the Appendix, with references when needed.
    \item[] Guidelines:
    \begin{itemize}
        \item The answer NA means that the paper does not include theoretical results. 
        \item All the theorems, formulas, and proofs in the paper should be numbered and cross-referenced.
        \item All assumptions should be clearly stated or referenced in the statement of any theorems.
        \item The proofs can either appear in the main paper or the supplemental material, but if they appear in the supplemental material, the authors are encouraged to provide a short proof sketch to provide intuition. 
        \item Inversely, any informal proof provided in the core of the paper should be complemented by formal proofs provided in appendix or supplemental material.
        \item Theorems and Lemmas that the proof relies upon should be properly referenced. 
    \end{itemize}

    \item {\bf Experimental result reproducibility}
    \item[] Question: Does the paper fully disclose all the information needed to reproduce the main experimental results of the paper to the extent that it affects the main claims and/or conclusions of the paper (regardless of whether the code and data are provided or not)?
    \item[] Answer: \answerYes{} % Replace by \answerYes{}, \answerNo{}, or \answerNA{}.
    \item[] Justification: All the information needed for reproducibility has been provided in the Appendix and the repository to the code has been provided as well.
    \item[] Guidelines:
    \begin{itemize}
        \item The answer NA means that the paper does not include experiments.
        \item If the paper includes experiments, a No answer to this question will not be perceived well by the reviewers: Making the paper reproducible is important, regardless of whether the code and data are provided or not.
        \item If the contribution is a dataset and/or model, the authors should describe the steps taken to make their results reproducible or verifiable. 
        \item Depending on the contribution, reproducibility can be accomplished in various ways. For example, if the contribution is a novel architecture, describing the architecture fully might suffice, or if the contribution is a specific model and empirical evaluation, it may be necessary to either make it possible for others to replicate the model with the same dataset, or provide access to the model. In general. releasing code and data is often one good way to accomplish this, but reproducibility can also be provided via detailed instructions for how to replicate the results, access to a hosted model (e.g., in the case of a large language model), releasing of a model checkpoint, or other means that are appropriate to the research performed.
        \item While NeurIPS does not require releasing code, the conference does require all submissions to provide some reasonable avenue for reproducibility, which may depend on the nature of the contribution. For example
        \begin{enumerate}
            \item If the contribution is primarily a new algorithm, the paper should make it clear how to reproduce that algorithm.
            \item If the contribution is primarily a new model architecture, the paper should describe the architecture clearly and fully.
            \item If the contribution is a new model (e.g., a large language model), then there should either be a way to access this model for reproducing the results or a way to reproduce the model (e.g., with an open-source dataset or instructions for how to construct the dataset).
            \item We recognize that reproducibility may be tricky in some cases, in which case authors are welcome to describe the particular way they provide for reproducibility. In the case of closed-source models, it may be that access to the model is limited in some way (e.g., to registered users), but it should be possible for other researchers to have some path to reproducing or verifying the results.
        \end{enumerate}
    \end{itemize}

\item {\bf Open access to data and code}
    \item[] Question: Does the paper provide open access to the data and code, with sufficient instructions to faithfully reproduce the main experimental results, as described in supplemental material?
    \item[] Answer: \answerYes{} % Replace by \answerYes{}, \answerNo{}, or \answerNA{}.
    \item[] Justification: The link can be found in the appendix.
    \item[] Guidelines:
    \begin{itemize}
        \item The answer NA means that paper does not include experiments requiring code.
        \item Please see the NeurIPS code and data submission guidelines (\url{https://nips.cc/public/guides/CodeSubmissionPolicy}) for more details.
        \item While we encourage the release of code and data, we understand that this might not be possible, so “No” is an acceptable answer. Papers cannot be rejected simply for not including code, unless this is central to the contribution (e.g., for a new open-source benchmark).
        \item The instructions should contain the exact command and environment needed to run to reproduce the results. See the NeurIPS code and data submission guidelines (\url{https://nips.cc/public/guides/CodeSubmissionPolicy}) for more details.
        \item The authors should provide instructions on data access and preparation, including how to access the raw data, preprocessed data, intermediate data, and generated data, etc.
        \item The authors should provide scripts to reproduce all experimental results for the new proposed method and baselines. If only a subset of experiments are reproducible, they should state which ones are omitted from the script and why.
        \item At submission time, to preserve anonymity, the authors should release anonymized versions (if applicable).
        \item Providing as much information as possible in supplemental material (appended to the paper) is recommended, but including URLs to data and code is permitted.
    \end{itemize}

\item {\bf Experimental setting/details}
    \item[] Question: Does the paper specify all the training and test details (e.g., data splits, hyperparameters, how they were chosen, type of optimizer, etc.) necessary to understand the results?
    \item[] Answer: \answerYes{} % Replace by \answerYes{}, \answerNo{}, or \answerNA{}.
    \item[] Justification: The information can be found in the Appendix.
    \item[] Guidelines:
    \begin{itemize}
        \item The answer NA means that the paper does not include experiments.
        \item The experimental setting should be presented in the core of the paper to a level of detail that is necessary to appreciate the results and make sense of them.
        \item The full details can be provided either with the code, in appendix, or as supplemental material.
    \end{itemize}

\item {\bf Experiment statistical significance}
    \item[] Question: Does the paper report error bars suitably and correctly defined or other appropriate information about the statistical significance of the experiments?
    \item[] Answer: \answerYes{} % Replace by \answerYes{}, \answerNo{}, or \answerNA{}.
    \item[] Justification: the results are accompanied by confidence intervals.
    \item[] Guidelines:
    \begin{itemize}
        \item The answer NA means that the paper does not include experiments.
        \item The authors should answer "Yes" if the results are accompanied by error bars, confidence intervals, or statistical significance tests, at least for the experiments that support the main claims of the paper.
        \item The factors of variability that the error bars are capturing should be clearly stated (for example, train/test split, initialization, random drawing of some parameter, or overall run with given experimental conditions).
        \item The method for calculating the error bars should be explained (closed form formula, call to a library function, bootstrap, etc.)
        \item The assumptions made should be given (e.g., Normally distributed errors).
        \item It should be clear whether the error bar is the standard deviation or the standard error of the mean.
        \item It is OK to report 1-sigma error bars, but one should state it. The authors should preferably report a 2-sigma error bar than state that they have a 96\% CI, if the hypothesis of Normality of errors is not verified.
        \item For asymmetric distributions, the authors should be careful not to show in tables or figures symmetric error bars that would yield results that are out of range (e.g. negative error rates).
        \item If error bars are reported in tables or plots, The authors should explain in the text how they were calculated and reference the corresponding figures or tables in the text.
    \end{itemize}

\item {\bf Experiments compute resources}
    \item[] Question: For each experiment, does the paper provide sufficient information on the computer resources (type of compute workers, memory, time of execution) needed to reproduce the experiments?
    \item[] Answer: \answerYes{} % Replace by \answerYes{}, \answerNo{}, or \answerNA{}.
    \item[] Justification: The Appendix contains all the required information.
    \item[] Guidelines:
    \begin{itemize}
        \item The answer NA means that the paper does not include experiments.
        \item The paper should indicate the type of compute workers CPU or GPU, internal cluster, or cloud provider, including relevant memory and storage.
        \item The paper should provide the amount of compute required for each of the individual experimental runs as well as estimate the total compute. 
        \item The paper should disclose whether the full research project required more compute than the experiments reported in the paper (e.g., preliminary or failed experiments that didn't make it into the paper). 
    \end{itemize}
    
\item {\bf Code of ethics}
    \item[] Question: Does the research conducted in the paper conform, in every respect, with the NeurIPS Code of Ethics \url{https://neurips.cc/public/EthicsGuidelines}?
    \item[] Answer: \answerYes{} % Replace by \answerYes{}, \answerNo{}, or \answerNA{}.
    \item[] Justification: The authors have reviewed the NeurIPS Code of Ethics and confirm the paper conform with it.
    \item[] Guidelines:
    \begin{itemize}
        \item The answer NA means that the authors have not reviewed the NeurIPS Code of Ethics.
        \item If the authors answer No, they should explain the special circumstances that require a deviation from the Code of Ethics.
        \item The authors should make sure to preserve anonymity (e.g., if there is a special consideration due to laws or regulations in their jurisdiction).
    \end{itemize}

\item {\bf Broader impacts}
    \item[] Question: Does the paper discuss both potential positive societal impacts and negative societal impacts of the work performed?
    \item[] Answer: \answerNA{} % Replace by \answerYes{}, \answerNo{}, or \answerNA{}.
    \item[] Justification:
    \item[] Guidelines:
    \begin{itemize}
        \item The answer NA means that there is no societal impact of the work performed.
        \item If the authors answer NA or No, they should explain why their work has no societal impact or why the paper does not address societal impact.
        \item Examples of negative societal impacts include potential malicious or unintended uses (e.g., disinformation, generating fake profiles, surveillance), fairness considerations (e.g., deployment of technologies that could make decisions that unfairly impact specific groups), privacy considerations, and security considerations.
        \item The conference expects that many papers will be foundational research and not tied to particular applications, let alone deployments. However, if there is a direct path to any negative applications, the authors should point it out. For example, it is legitimate to point out that an improvement in the quality of generative models could be used to generate deepfakes for disinformation. On the other hand, it is not needed to point out that a generic algorithm for optimizing neural networks could enable people to train models that generate Deepfakes faster.
        \item The authors should consider possible harms that could arise when the technology is being used as intended and functioning correctly, harms that could arise when the technology is being used as intended but gives incorrect results, and harms following from (intentional or unintentional) misuse of the technology.
        \item If there are negative societal impacts, the authors could also discuss possible mitigation strategies (e.g., gated release of models, providing defenses in addition to attacks, mechanisms for monitoring misuse, mechanisms to monitor how a system learns from feedback over time, improving the efficiency and accessibility of ML).
    \end{itemize}
    
\item {\bf Safeguards}
    \item[] Question: Does the paper describe safeguards that have been put in place for responsible release of data or models that have a high risk for misuse (e.g., pretrained language models, image generators, or scraped datasets)?
    \item[] Answer: \answerNA{} % Replace by \answerYes{}, \answerNo{}, or \answerNA{}.
    \item[] Justification:
    \item[] Guidelines:
    \begin{itemize}
        \item The answer NA means that the paper poses no such risks.
        \item Released models that have a high risk for misuse or dual-use should be released with necessary safeguards to allow for controlled use of the model, for example by requiring that users adhere to usage guidelines or restrictions to access the model or implementing safety filters. 
        \item Datasets that have been scraped from the Internet could pose safety risks. The authors should describe how they avoided releasing unsafe images.
        \item We recognize that providing effective safeguards is challenging, and many papers do not require this, but we encourage authors to take this into account and make a best faith effort.
    \end{itemize}

\item {\bf Licenses for existing assets}
    \item[] Question: Are the creators or original owners of assets (e.g., code, data, models), used in the paper, properly credited and are the license and terms of use explicitly mentioned and properly respected?
    \item[] Answer:\answerNA{} % Replace by \answerYes{}, \answerNo{}, or \answerNA{}.
    \item[] Justification:
    \item[] Guidelines:
    \begin{itemize}
        \item The answer NA means that the paper does not use existing assets.
        \item The authors should cite the original paper that produced the code package or dataset.
        \item The authors should state which version of the asset is used and, if possible, include a URL.
        \item The name of the license (e.g., CC-BY 4.0) should be included for each asset.
        \item For scraped data from a particular source (e.g., website), the copyright and terms of service of that source should be provided.
        \item If assets are released, the license, copyright information, and terms of use in the package should be provided. For popular datasets, \url{paperswithcode.com/datasets} has curated licenses for some datasets. Their licensing guide can help determine the license of a dataset.
        \item For existing datasets that are re-packaged, both the original license and the license of the derived asset (if it has changed) should be provided.
        \item If this information is not available online, the authors are encouraged to reach out to the asset's creators.
    \end{itemize}

\item {\bf New assets}
    \item[] Question: Are new assets introduced in the paper well documented and is the documentation provided alongside the assets?
    \item[] Answer: \answerNA{} % Replace by \answerYes{}, \answerNo{}, or \answerNA{}.
    \item[] Justification:
    \item[] Guidelines:
    \begin{itemize}
        \item The answer NA means that the paper does not release new assets.
        \item Researchers should communicate the details of the dataset/code/model as part of their submissions via structured templates. This includes details about training, license, limitations, etc. 
        \item The paper should discuss whether and how consent was obtained from people whose asset is used.
        \item At submission time, remember to anonymize your assets (if applicable). You can either create an anonymized URL or include an anonymized zip file.
    \end{itemize}

\item {\bf Crowdsourcing and research with human subjects}
    \item[] Question: For crowdsourcing experiments and research with human subjects, does the paper include the full text of instructions given to participants and screenshots, if applicable, as well as details about compensation (if any)? 
    \item[] Answer: \answerNA{}% Replace by \answerYes{}, \answerNo{}, or \answerNA{}.
    \item[] Justification: 
    \item[] Guidelines:
    \begin{itemize}
        \item The answer NA means that the paper does not involve crowdsourcing nor research with human subjects.
        \item Including this information in the supplemental material is fine, but if the main contribution of the paper involves human subjects, then as much detail as possible should be included in the main paper. 
        \item According to the NeurIPS Code of Ethics, workers involved in data collection, curation, or other labor should be paid at least the minimum wage in the country of the data collector. 
    \end{itemize}

\item {\bf Institutional review board (IRB) approvals or equivalent for research with human subjects}
    \item[] Question: Does the paper describe potential risks incurred by study participants, whether such risks were disclosed to the subjects, and whether Institutional Review Board (IRB) approvals (or an equivalent approval/review based on the requirements of your country or institution) were obtained?
    \item[] Answer: \answerNA{} % Replace by \answerYes{}, \answerNo{}, or \answerNA{}.
    \item[] Justification: 
    \item[] Guidelines:
    \begin{itemize}
        \item The answer NA means that the paper does not involve crowdsourcing nor research with human subjects.
        \item Depending on the country in which research is conducted, IRB approval (or equivalent) may be required for any human subjects research. If you obtained IRB approval, you should clearly state this in the paper. 
        \item We recognize that the procedures for this may vary significantly between institutions and locations, and we expect authors to adhere to the NeurIPS Code of Ethics and the guidelines for their institution. 
        \item For initial submissions, do not include any information that would break anonymity (if applicable), such as the institution conducting the review.
    \end{itemize}

\item {\bf Declaration of LLM usage}
    \item[] Question: Does the paper describe the usage of LLMs if it is an important, original, or non-standard component of the core methods in this research? Note that if the LLM is used only for writing, editing, or formatting purposes and does not impact the core methodology, scientific rigorousness, or originality of the research, declaration is not required.
    %this research? 
    \item[] Answer: \answerNA{}% Replace by \answerYes{}, \answerNo{}, or \answerNA{}.
    \item[] Justification:
    \item[] Guidelines:
    \begin{itemize}
        \item The answer NA means that the core method development in this research does not involve LLMs as any important, original, or non-standard components.
        \item Please refer to our LLM policy (\url{https://neurips.cc/Conferences/2025/LLM}) for what should or should not be described.
    \end{itemize}

\end{enumerate}

\end{document}